\documentclass[twoside,11pt]{article}

\usepackage[abbrvbib, preprint]{jmlr2e}
\usepackage{ISG_style_base}
\usepackage{CRL_specific}

\usepackage[utf8]{inputenc} 
\usepackage[T1]{fontenc}    
\usepackage{hyperref}
\usepackage{booktabs}       
\usepackage{amsfonts}       
\usepackage{nicefrac}       
\usepackage{microtype}     
\usepackage{amssymb}
\usepackage{physics}        
\usepackage{graphicx}
\usepackage{subcaption}
\usepackage{bbm}
\usepackage{dsfont}
\usepackage{cite}
\usepackage{flushend}
\usepackage{float}
\usepackage[dvipsnames]{xcolor}
\usepackage{algorithm}
\usepackage[noend]{algpseudocode}
\usepackage{mathtools}
\usepackage{aligned-overset}
\usepackage{comment}
\usepackage{tablefootnote}
\usepackage{enumitem}
\usepackage{centernot}
\usepackage{thm-restate}
\usepackage{pifont}
\usepackage{soul}
\setstcolor{red}
\usepackage{cleveref}

\crefname{assumption}{Assumption}{Assumptions}
\Crefname{assumption}{Assumption}{Assumptions}

\usepackage[toc,page,header]{appendix}
\usepackage{minitoc}

\usepackage{lastpage}
\jmlrheading{26}{2025}{1-\pageref{LastPage}}{2/24; Revised
3/25}{5/25}{24-0194}{Burak Var\i c\i, Emre Acartürk, Karthikeyan Shanmugam, Abhishek Kumar, and Ali Tajer}
\ShortHeadings{Score-based Causal Representation Learning: Linear and General Transformations}{Var\i c\i, Acartürk, Shanmugam, Kumar, and Tajer}

\begin{document}

\doparttoc 
\faketableofcontents 

\title{Score-based Causal Representation Learning: \\ Linear and General Transformations}

\author{\name Burak Var{\i}c{\i}\thanks{Equal contribution} \thanks{The work was completed when B.V. was at Rensselaer Polytechnic Institute.}  \email bvarici@andrew.cmu.edu \\
      \addr Machine Learning Department
      \\ Carnegie Mellon University,  Pittsburgh, PA 15213, USA
      \AND
      \name Emre Acart\"{u}rk\footnotemark[1] \email acarte@rpi.edu \\
      \addr Electrical, Computer, and Systems Engineering
      \\ Rensselaer Polytechnic Institute,  Troy, NY 12180, USA
      \AND
      \name Karthikeyan Shanmugam \email karthikeyanvs@google.com \\
      \addr Google DeepMind India \\
        Bengaluru 560043, India
      \AND
      \name Abhishek Kumar\thanks{Contributions of A.K. were done when the author was at Google Brain.} \email abhishek.mlwork@gmail.com \\
      \addr Amazon AGI, USA
      \AND
      \name Ali Tajer \email tajer@ecse.rpi.edu \\
      \addr Electrical, Computer, and Systems Engineering
      \\ Rensselaer Polytechnic Institute, Troy, NY 12180, USA
      }

\editor{Yoshua Bengio}
 
\maketitle

\begin{abstract}
  
    This paper addresses intervention-based causal representation learning (CRL) under a general nonparametric latent causal model and an unknown transformation that maps the latent variables to the observed variables. Linear and general transformations are investigated. The paper addresses both the \emph{identifiability} and \emph{achievability} aspects. Identifiability refers to determining algorithm-agnostic conditions that ensure the recovery of the true latent causal variables and the underlying latent causal graph. Achievability refers to the algorithmic aspects and addresses designing algorithms that achieve identifiability guarantees. By drawing novel connections between \emph{score functions} (i.e., the gradients of the logarithm of density functions) and CRL, this paper designs a \emph{score-based class of algorithms} that ensures both identifiability and achievability. First, the paper focuses on \emph{linear} transformations and shows that one stochastic hard intervention per node suffices to guarantee identifiability. It also provides partial identifiability guarantees for soft interventions, including identifiability up to mixing with parents for general causal models and perfect recovery of the latent graph for sufficiently nonlinear causal models. Secondly, it focuses on \emph{general} transformations and demonstrates that two stochastic hard interventions per node are sufficient for identifiability. This is achieved by defining a differentiable loss function whose global optima ensure identifiability for general CRL.
    Notably, one does \emph{not} need to know which pair of interventional environments has the same node intervened. Finally, the theoretical results are empirically validated via experiments on structured synthetic data and image data. \looseness=-1
\end{abstract}
    

\section{Overview}\label{sec:introduction} 

Causal representation learning (CRL) aims to form a causal understanding of the world by learning appropriate representations that support causal interventions, reasoning, and planning \citep{scholkopf2021toward}. Specifically, CRL considers a data-generating process in which high-level latent causally-related variables are mapped to low-level, generally high-dimensional observed data through an \emph{unknown} transformation. Formally, consider a causal Bayesian network~\citep{pearl2009causality} encoded by a directed acyclic graph (DAG) $\mcG$ with $n$ nodes and generating \emph{causal} random variables $\bZ \triangleq  [Z_1,\dots,Z_n]^{\top}$. These random variables are transformed by an \emph{unknown} function $g: \R^n \to \R^d$ to generate the $d$-dimensional \emph{observed} random variables $\bX \triangleq [X_1,\dots,X_d]^{\top}$ according to:
\begin{equation}
    \bX = g(\bZ) \  . \label{eq:data-generation-process-intro}
\end{equation}
CRL is the process of using the observed data $\bX$ and recovering (i)~\textbf{the causal structure $\mcG$} and (ii)~\textbf{the latent causal variables $\bZ$}. Achieving these implicitly involves another objective of recovering the unknown transformation $g$ as well. Addressing CRL consists of two central questions:
\begin{itemize}[leftmargin=1em]
    \item {\bf Identifiability,} which refers to determining the necessary and sufficient conditions under which $\mcG$ and $\bZ$ can be recovered. The scope of identifiability (e.g., perfect or partial) critically depends on the extent of information available about the data, the underlying causal structure, and the transformation. The nature of the identifiability results can be algorithm-agnostic and non-constructive without specifying how to recover $\mcG$ and $\bZ$. In particular, this is the case when considering CRL under a general transform $g$ without parametric assumptions. Furthermore, the study of identifiability also investigates necessary conditions, e.g., which type of data is required for identifiability, regardless of the algorithmic approach.
    
    \item {\bf Achievability}, which complements identifiability and pertains to designing algorithms that can recover $\mcG$ and $\bZ$ while maintaining identifiability guarantees. Achievability hinges on forming reliable estimates for the transformation $g$.
\end{itemize}

\paragraph{CRL from interventions.} Identifiability is known to be impossible without additional supervision or sufficient statistical diversity among the samples of the observed data $\bX$. As shown in~\citep{hyvarinen1999nonlinear,locatello2019challenging}, this is the case even for the simpler settings in which the latent variables are statistically independent (i.e., graph $\mcG$ has no edges). On the other hand, using data generated under interventions in addition to the observational data generated by the underlying latent model creates useful statistical diversity. Specifically, an intervention on a set of causal variables alters the causal mechanisms that generate those variables. Note that these causal mechanisms capture the effect of parents on the child variable. Such interventions, even when imposing sparse changes in the statistical models, create variations in the observed data sufficient for learning latent causal representations. This has led to CRL via intervention as an important class of CRL problems, which in its general form has remained an open problem~\citep{scholkopf2021toward}. This paper addresses this open problem by drawing a novel connection to score functions, i.e., the gradients of the logarithm of density functions. 

We note that we use the interventions as a weak form of supervision via having access to \emph{only} the pair of distributions before and after an intervention. Such supervision can be quite flexible and bodes well for practical applications in genomics~\citep{tejada2025causal} and robotics~\citep{lee2021causal}. For instance, genomics experiments often involve interventions, which can be modeled as deterministic or stochastic interventions, depending on the experimental mechanism. In robotics, let us consider the causal variables to be the joint angles of a robotic arm. A stochastic intervention corresponds to setting the joint angle to take values within a suitable support. We note that the interventions here can be soft since the feasible support for the joint may still depend on the other joint angles after the intervention. In this setting, the observations are simply images of the entire arm captured by a camera positioned at a specific location.

\paragraph{Contributions.} This paper provides both identifiability and achievability results for CRL from interventions under general latent causal models and general transformations. We establish these results by uncovering hitherto unknown connections between \emph{score functions} (i.e., the gradients of the logarithm of density functions) and CRL. We leverage these connections to design CRL algorithms that serve as constructive proofs for the identifiability and achievability results. We do not make any parametric assumption on the latent causal model, i.e., the relationships among elements of $\bZ$ take any arbitrary form. For the transformation $g$, we consider a diffeomorphism (i.e., bijective such that both $g$ and $g^{-1}$ are continuously differentiable) onto its image. Our results are categorized into two main groups based on the form of $g$: (i)~\textbf{linear} transformation, and (ii)~\textbf{general} (nonparametric) transformation. We consider both stochastic hard and soft interventions, and our contributions are summarized below. \looseness=-1

\paragraph{Linear transformations.} We first focus on linear transformations and investigate various extents of results (perfect and partial) given \emph{one} intervention per node. In this setting, we consider both hard and soft interventions.
    \begin{itemize}
        \item[\yes] On \textbf{identifiability from hard interventions}, we show that \emph{one} hard intervention per node suffices to guarantee perfect identifiability (\Cref{th:linear-hard}). 

        \item[\yes] On \textbf{identifiability from soft interventions}, we show that \emph{one} soft intervention per node suffices to guarantee identifiability up to ancestors -- transitive closure of the latent DAG is recovered and latent variables are recovered up to a linear function of their ancestors (\Cref{th:linear-soft}). 

        \item[\yes] While establishing these results, we also establish partial identifiability of a latent variable given an intervention on \emph{single} node (\Cref{thm:linear-partial-identifiability}).

        \item[\yes] We further tighten the previous results and show that when the latent causal model is sufficiently nonlinear, perfect DAG recovery becomes possible using soft interventions. Furthermore, we recover a latent representation that is Markov with respect to the latent DAG, preserving the true conditional independence relationships (\Cref{th:linear-soft-full-rank}). 
        
        \item[\yes] On \textbf{achievability}, we design an algorithm referred to as {\bf L}inear {\bf S}core-based {\bf Ca}usal {\bf L}atent {\bf E}stimation via {\bf I}nterventions (LSCALE-I), which achieves the identifiability guarantees under both soft and hard interventions. LSCALE-I first forms an encoder estimate using correlation matrices of the observed score differences and then uses the learned encoder to estimate the transitive closure of the latent DAG. Then, in the case of hard interventions, it refines these outputs without requiring any conditional independence test.   
    \end{itemize}   

\paragraph{General transformations.} In this setting, we do not have any restriction on the transformation from the latent space to the observed space. In this general setting, our contributions are as follows.
    \begin{itemize}
        \item[\yes] On \textbf{identifiability}, we show that observational data and \emph{two} distinct hard interventions per node suffice to guarantee perfect identifiability (\Cref{th:general-uncoupled}). This result generalizes the recent results in the literature in two ways. First, we do not require the commonly adopted \emph{faithfulness} assumption on latent causal models. Secondly, we assume the learner does not know which pair of environments intervene on the same node. While proving these results, we also establish identifiability of a single variable given two hard interventions \emph{only} for the said variable (\Cref{thm:general-partial-identifiability}).
    
        \item[\yes] More importantly, on \textbf{achievability}, we design the first provably correct algorithm that recovers $\mcG$ and $\bZ$ perfectly. This algorithm is referred to as {\bf G}eneralized {\bf S}core-based {\bf Ca}usal {\bf L}atent {\bf E}stimation via {\bf I}nterventions (GSCALE-I). We note that GSCALE-I requires only the score functions of observed variables as its inputs and computes those of the latent variables by leveraging the Jacobian of the decoders. 
        
        \item[\yes] We also establish new results that shed light on the role of observational data. Specifically, when two interventional environments per node are given in pairs, observational data is only needed for recovering the latent DAG (\Cref{th:general-coupled}). Furthermore, we show that observational data can be dispensed with under a weak faithfulness condition on the latent causal model (\Cref{th:general-coupled-no-observational}). Remarkably, our results are tight in terms of the identifiability objectives, which we will discuss in \Cref{sec:objectives}.

        \item[\yes] Finally, we show that extrapolation to unseen combinations of single-node interventions can be achieved \emph{without performing CRL}, just by using score functions of the \emph{observed} space (\Cref{sec:intervention-extrapolation}).
    \end{itemize}
Before describing our novel methodology, we recall the general approach to solving CRL. Recovering the latent causal variables hinges on finding the inverse of $g$ based on the observed data~$\bX$, which facilitates recovering $\bZ$ via $\bZ = g^{-1}(\bX)$. In other words, referring to $(g^{-1},g)$ as the \emph{true} encoder-decoder pair, we search for an encoder that takes observed variables back to the true latent space. A valid encoder should necessarily have an associated decoder to ensure a perfect reconstruction of observed variables from estimated latent variables. However, infinitely many valid encoder-decoder pairs satisfy the reconstruction property. The purpose of using interventions is to inject variations into the observed data, which can help us distinguish among these solutions.

\paragraph{Score-based methodology.} We start by showing that an intervention on a latent node induces changes only in the score function's coordinates corresponding to the intervened node and its parents. This is because the intervention changes only the causal mechanism (i.e., conditional distribution) of the intervened node, which is a function of the latent node and its parents. Furthermore, when we consider two hard interventions on the same node, such changes will be limited only to the intervened node (parents intact). This implies that score changes of the latent variables $\bZ$ are generally sparse across different environments. Furthermore, these score changes contain all the information about the latent causal structure. Motivated by these key properties, we formalize a score-based CRL framework, based on which we design provably correct distinct algorithms, LSCALE-I and GSCALE-I, for linear and general transformations, respectively, presented in \Cref{sec:linear} and \Cref{sec:general}. We briefly describe the key technique in GSCALE-I for general transformation via two interventions. LSCALE-I involves exploiting the linearity of the transformation to achieve identifiability using only one intervention per node. \looseness=-1

\paragraph{Algorithm sketch for general transformations.}  
Consider two interventional environments in which the same node is intervened on. As described in the preceding paragraph, we show that the score functions of the latent variables under these two environments differ only at the coordinate of the intervened node. 
Subsequently, the key idea of the score-based framework is that tracing these sparse changes in the score functions of the latent variables can guide finding reliable estimates for the inverse of the transformation $g$, which in turn facilitates estimating $\bZ$. In particular, we will seek encoders that result in sparse score variations among the estimated latent variables, thereby matching the properties of the true latent variables. To this end, we consider a pair of interventional environments for each node and find an encoder that minimizes the variations in the score function across all pairs. We show that the encoder obtained via this procedure perfectly recovers $\mcG$ and $\bZ$. An important process in this methodology is projecting the score changes in observed data to the latent space so that we do not need to estimate latent scores for each possible encoder. We show that this can be done by multiplying the observed score difference by the Jacobian of the decoder associated with the encoder. Therefore, recovering $Z_i$ is facilitated by solving the following problem:
\begin{equation}\label{eq:intro-OPT-X-verbal}
\mbox{For $i$-th int. env. pair :} \;\; 
    \min
    \Big\| \E\Big[\big | \text{Jac. of decoder} \times \big(\text{score difference of $\bX$} \big) \big\|\Big] - \be_i \Big \|^2 \ ,
\end{equation}
in which $\be_i$ denotes $i$-th standard basis vector, reflecting the sparsity property. Score differences of observed variables are computed across two environments with the same intervened node.  When we have two interventions for \emph{each} latent variable, we solve this problem for $i\in[n]$ simultaneously to recover complete $\bZ$, and subsequently, the graph $\mcG$. Finally, the minimization is performed over the set of valid encoder-decoder pairs that ensure perfect reconstruction of $\bX$.

\begin{figure}[t]
    \centering
    \includegraphics[width=\linewidth]{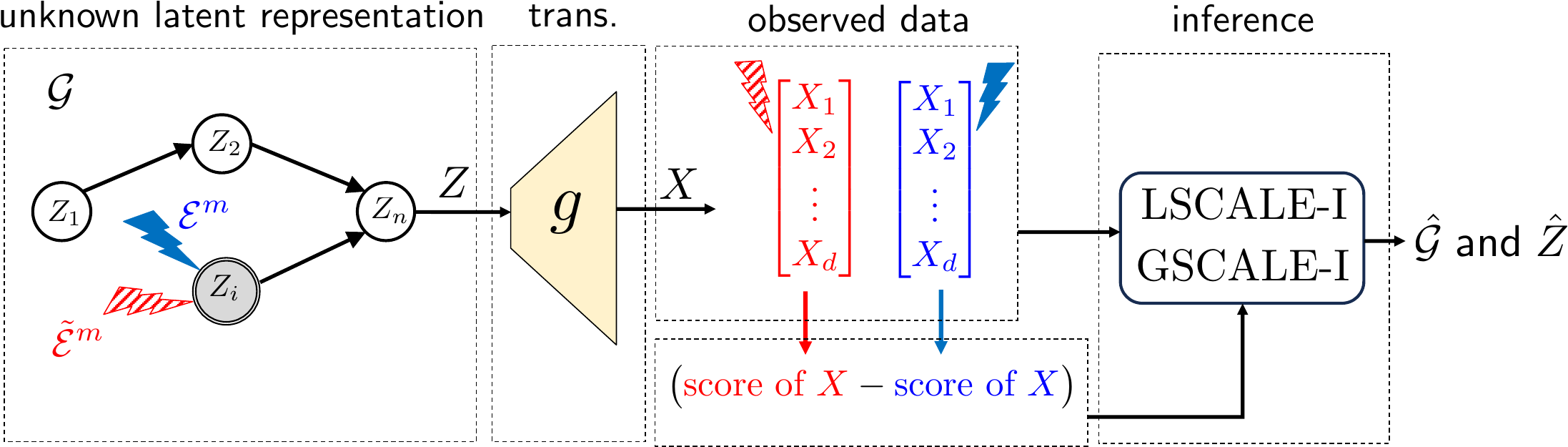}
    \caption{\textbf{An overview of LSCALE-I and GSCALE-I algorithms}. Corresponding to each latent variable $Z_i$, there are two interventional mechanisms, denoted by red and blue. For each pair of environments, the score functions of observed variables are fed into the GSCALE-I algorithm. Then, GSCALE-I uses this input to compute \emph{latent score differences} for a given encoder and returns the encoder that minimizes these latent score differences. This encoder is used to compute estimates $\hat \mcG$ and $\hat \bZ$.}
    \label{fig:algorithm-sketch}
\end{figure}

\paragraph{Organization.} The rest of the paper is organized as follows. \Cref{sec:related-work} provides an overview of the literature with the main focus on CRL via interventions. \Cref{sec:problem} provides the preliminaries for formulating the problem and specifies the notations and definitions used throughout the paper. \Cref{sec:analysis-score} establishes the properties of score functions under interventions and presents the key lemmas and their implications, which will be used in our CRL framework. We present our CRL algorithms and results for linear transformations in \Cref{sec:linear}, and general transformations in \Cref{sec:general}. Proofs of the building blocks of the results are presented in the main body of the paper, and the rest of the proofs are provided in the appendices. In \Cref{sec:experiments}, we empirically assess the performance of the proposed CRL algorithms for recovering the latent causal variables and the latent causal graph on both structured synthetic data, biological data, and image data. \Cref{sec:conclusion} concludes the paper with a discussion of the results and the future directions.

\section{Related Work}\label{sec:related-work}

In this paper, we study CRL in different interventional environments where the interventions act on the latent space. We first provide an overview of the literature that investigates CRL from interventional data, with the main results of the most closely related work summarized in \Cref{tab:related-comp-parametric} and \Cref{tab:related-comp-general}. Then, we discuss the other relevant lines of work.

\paragraph{Interventional causal representation learning.}
The majority of the rapidly growing literature on CRL from interventions focuses on \emph{parametric} settings, i.e., a parametric form is assumed for the latent model, the transformation, or both. Among the related studies, that in~\citep{ahuja2023interventional} mainly considers polynomial transformations without restrictions on the latent causal model and shows identifiability under deterministic \emph{do} interventions. It also establishes identifiability under general transformations, albeit requiring a combinatorial number of \emph{do} interventions. Finally, it shows identifiability under soft interventions with independent support assumptions on latent variables.
The study in~\citep{squires2023linear} considers a linear latent model with a linear mapping to observations and proves identifiability under hard interventions. It also shows the impossibility of perfect identifiability under soft interventions and proves identifiability up to ancestors. The study in~\citep{buchholz2023learning} focuses on linear Gaussian latent models and extends the results in \citep{squires2023linear} to prove identifiability for general transformations. \citet{zhang2023identifiability} consider polynomial transformations under nonlinearity assumptions on latent models and prove identifiability up to ancestors under soft interventions. If the latent graph is restricted to polytrees, they further prove perfect identifiability. Identifying the non-intervened variables from the intervened variables by using single-node and multi-node soft interventions is studied in \citep{ahuja2023multi}. It also considers a new setting in which the latent DAG can change across data points and relies on the support invariance of non-intervened variables to identify them from the rest. The study in \citep{bing2024identifying} considers a nonlinear latent model under linear transformation and uses multi-target \emph{do} interventions to prove identifiability under certain sparsity assumptions. Linear transformation and linear non-Gaussian latent models are studied in \citep{jin2024learning}, establishing identifiability up to surrounding parents using soft interventions. It also establishes sufficient conditions for multi-target soft interventions to ensure identifiability. The study in \citep{saengkyongam2023identifying} takes a different approach and considers the task of intervention extrapolation. In this formulation, interventions are applied to exogenous action variables (e.g., instrumental variables) that linearly affect the latent variables.

\begin{table*}[t]
    \centering
    \caption{\small {\bf Parametric Settings.} Comparison of the results to prior studies in \emph{parametric} settings. Only the main results of the most closely related studies are listed, and standard shared assumptions are omitted. Formal definitions of identifiability measures are provided in \Cref{sec:problem}.}
    \scalebox{0.8}{
    \begin{tabular}{ccccccc}
        \toprule
        Work & Transform & Latent & Int. Data  & Identifiability \\
        & & Model & (one env. per node) & result  \\
        \hline 
        \citet{squires2023linear} & Linear & Linear & Hard & perfect  \\
         & Linear & Linear & Soft & up to ancestors \\
            \midrule 
        \citet{ahuja2023interventional} & Polynomial  & General & \emph{do} & perfect \\
        & Polynomial & Bounded RV & Soft & perfect \\
        \hline
        \citet{buchholz2023learning} & General & Lin. Gaussian & Hard & perfect \\ 
                                     & General & Lin. Gaussian & Soft & up to ancestors \\ \bottomrule
        \citet{zhang2023identifiability} & Polynomial & nonlinear & Soft & up to ancestors  \\ 
        & Polynomial & nonlinear (polytree) & Soft & perfect  \\ \bottomrule
        \textbf{\Cref{th:linear-hard}} & Linear & General & Hard & perfect \\    
        \textbf{\Cref{th:linear-soft}} & Linear & General & Soft  & up to ancestors \\    
        \textbf{\Cref{th:linear-soft-full-rank}} & Linear & nonlinear & Soft & perfect DAG and \\
        & & & & mixing w. surrounding  \\ \bottomrule
    \end{tabular}}
    \label{tab:related-comp-parametric}
\end{table*}

On the fully \emph{nonparametric} setting, \citet{vonkugelgen2023twohard} provide the most closely related identifiability results to ours. Specifically, they show that two \emph{coupled} hard interventions per node suffice for identifiability under the faithfulness assumption on latent causal models. Our results have two significant differences: 1) We address achievability via a provably correct algorithm whereas \citet{vonkugelgen2023twohard} focus mainly on identifiability (e.g., no algorithm for recovery of the latent variables), 2) we dispense with the restrictive assumptions on identifiability results, namely, we do not require to know which two environments share the same intervention target (hence, \emph{uncoupled} interventions), and do not require faithfulness on the latent models. While the importance of this relaxation is not fully apparent for single-node interventions, being able to work without knowing intervention targets is critical for investigating CRL under more realistic cases of multi-node interventions. Among the other studies on the nonparametric setting, \citet{jin2024learning} provide analogous results to \citep{vonkugelgen2023twohard} by considering two coupled soft interventions and identifying latent variables up to mixing with surrounding variables. \citet{jiang2023learning} consider identifying the latent DAG without recovering latent variables, where it is shown that a restricted class of DAGs can be recovered. 
Finally, some studies employ stronger supervision signals, such as annotations of the ground-truth causal variables~\citep{shen2022weakly}, or knowledge of the causal graph to recover the latent variables under hard interventions~\citep{liang2023causal}. 

\paragraph{Multi-view causal representation learning.} A commonly used form of weak supervision in CRL involves multi-view data, non-i.i.d. samples (or views) are observed, typically generated by the same or closely related realizations of underlying latent variables. Multi-view scenarios can involve various data settings. Earlier studies have mostly considered paired data, in which pre- and post-intervention observations are generated from the same set of latent variables \citep{locatello2020weakly,vonkugelgen2021self,brehmer2022weakly}. The more generalized and relaxed formulation typically involves partial observability, where multiple views are concurrently generated by an overlapping subset of latent variables, such as observing different camera angles of the same scene~\citep{sturma2023unpaired}. \citet{yao2023multi} generalizes the multi-view approach via a unified framework, which also allows partial observability with nonlinear transforms. We also note that \citet{yao2025unifying} presents an even more general unifying view to interpret different CRL approaches, including interventional and multi-view, as special ways of aligning the representations to known data symmetries via leveraging the invariance principle.

\paragraph{Temporal causal representation learning.} Temporal CRL is particularly motivated by applications in domains where time-series data is readily available, such as robotics and control systems, which has led to a growing interest in this area. Some studies in this category assume that the latent causal variables are mutually independent, meaning there are no instantaneous causal effects between variables at the same time step \citep{lippe2022intervention,yao2022,lachapelle2024nonparametric}. However, more recent studies have relaxed this assumption, allowing instantaneous causal effects when sufficient diversity is present in the observed or interventional data \citep{lippe2022icitris,lippe2023biscuit,li2024identification}. \looseness=-1

\paragraph{Other approaches.} In some other approaches to CRL, \citet{zhang2024causal} consider general mixing functions and general SCMs under sparsity constraints; 
\citet{liu2024identifiable} leverage nonlinear ICA principles to work with nonlinear mixing functions and polynomial latent SCMs; \citet{morioka2024causal} consider disjoint groups of observational variables, with general mixing functions and pairwise latent causal relationships; \citet{welch2024identifiability} use score functions for linear mixing functions and nonlinear Gaussian SCMs to derive coarser identifiability results; and \citet{li2024disentangled} study ``domains'' and interventions in a combined way for non-Markovian causal systems when the graph is known.

\paragraph{Identifiable representation learning.}
As a special case of CRL, where the latent variables are independent, there is extensive literature on identifying latent representations. Some representative approaches include leveraging knowledge of the mechanisms that govern the system's evolution~\citep {ahuja2021properties} and using weak supervision with auxiliary information~\citep{shu2019weakly}. Nonlinear independent component analysis (ICA) also uses side information, in the form of structured time series, to exploit temporal information~\citep{hyvarinen2017nonlinear,halva2020hidden} or knowledge of auxiliary variables that make latent variables conditionally independent~\citep{pmlr-v108-khemakhem20a,khemakhem2020ice,hyvarinen2019nonlinear}. \citet{morioka2024causal} imposes additional constraints on observational mixing and the causal model to prove identifiability. \citet{kivva2022identifiability} studies the identifiability of deep generative models without auxiliary information. \looseness=-1

\begin{table*}[t]
    \centering
    \caption{\small {\bf General (Nonparametric) Settings.} Comparison of the results to prior studies in the \emph{general} setting. Formal definitions of identifiability measures are provided in \Cref{sec:problem}. The shared assumptions (interventional discrepancy) and additional assumptions (faithfulness) are discussed in \Cref{sec:general}.}
    \scalebox{0.75}{
    \begin{tabular}{cccccccc}
        \toprule
        Work & Transform and & Obs. & Int. Data  & Faithfulness & Identifiability & Provable \\
        & Latent Model  & Data & (env. per node) & & result & algorithm \\
        \hline 
        \citet{vonkugelgen2023twohard} & General  & No & 2 coupled hard & Yes & perfect & \no \\ \bottomrule
        \citet{jin2024learning} & General & No & 2 coupled soft & No & perfect DAG and & \no \\
        & & & & & mix w. surrounding \\ \bottomrule
     \textbf{\Cref{th:general-coupled}}    & General & Yes & 2 coupled hard & No & perfect  & \yes \\ 
     \textbf{\Cref{th:general-coupled-no-observational}}   & General  & No & 2 coupled hard & Yes & perfect & \yes \\ 
     \textbf{\Cref{th:general-uncoupled}} & General & Yes & 2 uncoupled hard & No & perfect & \yes \\
     \bottomrule
    \end{tabular}}
    \label{tab:related-comp-general}
\end{table*}

\paragraph{Score functions for causal discovery within observed variables.} Score matching has recently gained attention in the causal discovery of observed variables. \citet{rolland2022score} use score matching to recover nonlinear additive Gaussian noise models. \citet{montagna2023scalable} focus on the same setting, recover the full graph from Jacobian scores, and dispense with the computationally expensive pruning stage of the algorithm in \citep{rolland2022score}. \citet{montagna2023assumption} empirically demonstrate the robustness of score matching approaches against the assumption violations in causal discovery. \citet{zhu2023sample} establish bounds on the error rate of score matching-based causal discovery methods. All of these studies are limited to observed causal variables, whereas in our case, we have a causal model in the latent space.

\section{Preliminaries and Definitions}\label{sec:problem}

\paragraph{Notations.}
For a vector $\ba\in\R^{m}$, the $i$-the entry is denoted by $a_i$. Random vectors are denoted by bold upper-case letters and their realizations are denoted by bold lower-case letters, e.g., $\bX$ and $\bx$. Matrices are denoted by bold upper-case letters, e.g., $\bA$, where $\bA_i$ denotes the $i$-th row of $\bA$ and $\bA_{i,j}$ denotes the entry at row $i$ and column $j$. For matrices $\bA$ and $\bB$ with the same shapes, $\bA \preccurlyeq \bB$ denotes component-wise inequality. We denote the indicator function by $\mathds{1}$, and for a matrix $\bA\in\R^{m\times n}$, we use the convention that $\mathds{1}\{\bA\}\in\{0,1\}^{m\times n}$, where the entries are specified by $[\mathds{1}\{\bA\}]_{i,j}=\mathds{1}\{\bA_{i,j}\neq 0\}$. For a positive integer $n\in\N$, we define $[n] \triangleq \{1,\dots,n\}$.  The permutation matrix associated with any permutation $\pi$ of $[n]$ is denoted by $\bP_{\pi}$, i.e., $[\pi_1 \;\;  \pi_2 \; \dots   \; \pi_n]^{\top}=\bP_{\pi}\cdot[1 \;\; 2 \; \dots \;  n]^{\top}$. The $n$-dimensional identity matrix is denoted by $\bI_{n \times n}$, and the Hadamard product is denoted by $\odot$. We use $\image(f)$ to denote the image of the function $f$. We let $\be_i$ denote the $i$-th standard basis vector of $\R^n$.
Given a function $f: \R^s \to \R^r$ that has first-order partial derivatives on $\R^s$, we denote the Jacobian of $f$ at $\bz\in\R^s$ by $J_{f}(\bz) \in \R^{r \times s}$ with entries $[J_{f}(\bz)]_{i,j} = \partial f(\bz)_i / \partial \bz_j$.

\subsection{Latent Causal Structure}

Consider latent causal random variables $\bZ \triangleq  [Z_1,\dots,Z_n]^{\top}$. An \emph{unknown} transformation $g: \R^n \to \R^d$ generates the observable random variables $\bX \triangleq [X_1,\dots,X_d]^{\top}$ from the latent variables $\bZ$ according to:
\begin{equation}
    \bX = g(\bZ) \  . \label{eq:data-generation-process}
\end{equation}
We assume that $d \geq n$, and the transformation $g$ is continuously differentiable and a diffeomorphism onto its image (otherwise, identifiability is ill-posed). We denote the image of $g$ by $\mcX \triangleq  \image(g)\subseteq \R^d$. The probability density functions (pdfs) of $\bZ$ and $\bX$ are denoted by $p$ and $p_{\bX}$, respectively. We assume that $p$ is absolutely continuous with respect to the $n$-dimensional Lebesgue measure. Subsequently, $p_{\bX}$, which is defined on the image manifold $\image(g)$, is absolutely continuous with respect to the $n$-dimensional Hausdorff measure rather than $d$-dimensional Lebesgue measure\footnote{For details of where this has been used, see \Cref{sec:score-diff-lm-proof}.}. The distribution of latent variables $\bZ$ factorizes with respect to a DAG that consists of $n$ nodes and is denoted by $\mcG$. Node $i\in[n]$ of $\mcG$ represents $Z_i$ and $p$ factorizes according to:
\begin{equation}
    p(\bz) = \prod_{i=1}^{n} p_i(z_i \med \bz_{\Pa(i)}) \ , \label{eq:pz_factorized}
\end{equation}
where $\Pa(i)$ denotes the set of parents of node $i$ and $p_i(z_i \med \bz_{\Pa(i)})$ is the conditional pdf of $z_i$ given the variables of its parents. We use $\Ch(i)$, $\An(i)$, and $\De(i)$ to denote the children, ancestors, and descendants of node $i$, respectively. Accordingly, for each node $i\in[n]$ we also define
\begin{align}\label{eq:graph-notations}
    \overline{\Pa}(i) \triangleq \Pa(i) \cup \{i\} , \;\; \overline{\Ch}(i)\triangleq \Ch(i) \cup \{i\} , \;\; \overline{\An}(i)\triangleq \An(i) \cup \{i\} , \; \mbox{and} \; \overline{\De}(i)\triangleq \De(i) \cup \{i\}\ .
\end{align}
We denote the transitive closure and transitive reduction of $\mcG$ by $\mcG_{\rm tc}$ and $\mcG_{\rm tr}$, respectively\footnote{Transitive closure of a DAG $\mcG$, denoted by $\mcG_{\rm tr}$, is a DAG with parents denoted by $\Pa_{\rm tr}(i)=\An(i)$ for each node $i$. The transitive reduction of $\mcG$ is the DAG with the fewest edges that preserves the same reachability relation as~$\mcG$.}. The parental relationships in these graphs are denoted by $\Pa_{\rm tc}$ and $\Pa_{\rm tr}$, and other graphical relationships are denoted similarly. Based on the modularity property, a change in the causal mechanism of node~$i$ does not affect those of the other nodes. We also assume that all conditional pdfs $\{p_i(z_i \mid \bz_{\Pa(i)}) : i \in[n]\}$ are continuously differentiable with respect to all variables and $p(\bz) \neq 0$ for all $\bz \in \R^n$. 
We consider the general structural causal models (SCMs) based on which
for each $i\in[n]$,
\begin{align}\label{eq:general-SCM}
    Z_i = f_{i}(\bZ_{\Pa(i)},N_{i}) \ ,
\end{align}
where $\{f_{i} : i \in [n]\}$ are general functions that capture the dependence of node $i$ on its parents and $\{N_{i} : i \in [n]\}$ account for the exogenous noise terms that we assume to have pdfs with full support. We specialize some of the results to additive noise SCMs, in which \eqref{eq:general-SCM} becomes
\begin{align}\label{eq:additive-SCM}
    Z_i = f_{i}(\bZ_{\Pa(i)}) + N_{i}\ .
\end{align}
Next, we provide several definitions that we will use frequently throughout the paper to formalize the framework and analyze it.

\begin{definition}[Valid Causal Order]\label{def:valid-causal-order}
    We refer to a permutation $(\pi_1,\dots,\pi_n)$ of $[n]$ as a \emph{valid causal order}\footnote{It is also called \emph{topological ordering} or \emph{topological sort} in the literature.} if $\pi_i \in \Pa(\pi_j)$ indicates that $i < j$. 
\end{definition}
In this paper, without loss of generality, we assume that $(1,\dots, n)$ is a valid causal order. 
We also define a graphical notion that will be useful for presenting our results and analysis on CRL under a linear transformation.
\begin{definition}[Surrounded Node] \label{def:surrounded}
    Node $i\in[n]$ in DAG $\mcG$ is said to be \emph{surrounded} if there exists another node $j\in[n]$ such that $\overline{\Ch}(i) \subseteq \Ch(j)$. We denote the set of nodes that surround $i\in[n]$ by $\sur(i)$, and the set of all nodes that are surrounded by $\mcS$, i.e.,
    \begin{align}
        \sur(i) &\triangleq \{j\in[n] \, : \, j\neq i\;\;, \;\; \overline{\Ch}(i) \subseteq \Ch(j) \}\ , \quad \mbox{and} \quad 
            \mcS \triangleq \{i\in [n] \, : \; \sur(i)\neq \emptyset\}\ . \label{eq:def-surrounded-set}
    \end{align}
\end{definition}
The intuition behind surrounded nodes is that, the effect of $i$ on its children $\Ch(i)$ can be dominated by the effect of its surrounding node $j$.\footnote{Surrounded node concept is first defined by \citep{varici2023score}, and later adopted by \citet{jin2024learning} when considering soft interventions.} Specifically, any effect of $i$ on a node $k \in \Ch(i)$ can also be interpreted as the effect of node $j$ since $k \in \Ch(j)$ as well. This effect causes ambiguities in the recovery of $Z_i$ in the case of soft interventions.

\begin{figure}[t]
    \centering
    \begin{subfigure}[t]{0.48\textwidth}
        \centering
        \includegraphics[height=1in]{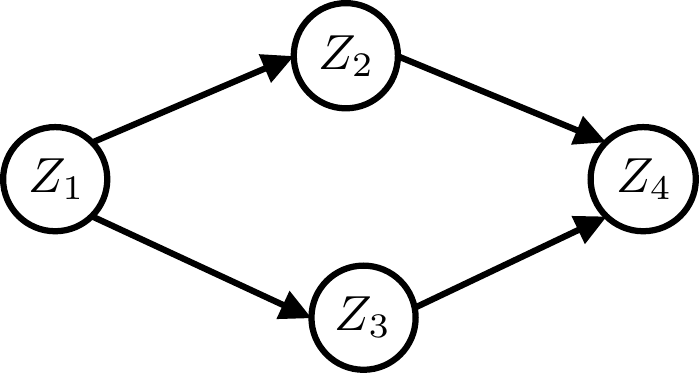}
        \label{fig:sample-graph-1}
    \end{subfigure}
    \begin{subfigure}[t]{0.48\textwidth}
        \centering
        \includegraphics[height=1in]{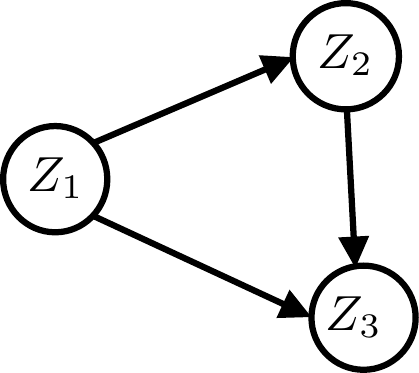}
        \label{fig:sample-graph-2}
    \end{subfigure}   
    \caption{Sample latent DAGs $\mcG$. \textbf{(Left)} The valid causal orders are $(1,2,3,4)$ and $(1,3,2,4)$, and only surrounded node is $4$ with $\sur(4)=\{2,3\}$. \textbf{(Right)} The only valid causal order is $(1,2,3)$, and the surrounded nodes are $\mcS=\{2,3\}$ with $\sur(2)=\{1\}, \sur(3)=\{1,2\}$.}
    \label{fig:sample}
\end{figure}

\subsection{Score Functions}\label{sec:score-definitions}
The \emph{score function} associated with a pdf is defined as the gradient of its logarithm. The score function associated with $p$ is denoted by 
\begin{equation}
    \bss(\bz) \triangleq \nabla_{\bz} \log p(\bz) \ . \label{eq:sz-def}
\end{equation}
Noting the connection $\bX = g(\bZ)$, the density of $\bX$ under $\mcE^0$, denoted by $p_{\bX}$, is supported on an $n$-dimensional manifold $\mcX$ \emph{embedded} in $\R^{d}$. Hence, specifying the score function of $\bX$ requires notions from differential geometry. For this purpose, we denote the tangent space of the manifold $\mcX$ at point $\bx \in \mcX$ by $T_{\bx} \mcX$. Tangent vectors $\bv \in T_x \mcX$ are equivalence classes of continuously differentiable curves $\gamma \colon (-1,1) \to \mcX \subseteq \R^d$ with $\gamma(0) = \bx$ and $\gamma'(0) = \bv$. Furthermore, given a function $f \colon \mcX \to \R$, denote its directional derivative at point $\bx \in \mcX$ along a tangent vector $\bv \in T_{\bx} \mcX$ by $D_{\bv} f(\bx)$, which is defined as
\begin{equation}
    D_{\bv} f(\bx) \triangleq \left. \frac{\dd}{\dd t} (f \circ \gamma)(t) \right|_{t = 0} \ ,
\end{equation}
for any curve $\gamma$ in equivalence class $v$. The differential of $f$ at point $\bx \in \mcX$, denoted by $\dd f_{\bx}$, is the linear operator mapping tangent vector $\bv \in T_{\bx} \mcX$ to $D_{\bv} f_{\bx}$ \citep[p.~57]{simon2014introduction}, i.e.,
\begin{equation}
    \dd f_{x} \colon T_{\bx} \mcX \ni v \mapsto D_{\bv} f(\bx) \in \R \ . \label{eq:differential-def}
\end{equation}
Let $\bB \in \R^{d \times n}$ be a matrix for which the columns of $\bB$ form an orthonormal basis for $T_{\bx} \mcX$. Denote the directional derivative of $f$ along the $i$-th column of $\bB$ by $D_i f$ for all $i \in [n]$. Then, the differential operator can be expressed by the vector
\begin{equation}
    D f_{\bx} \triangleq \bB \cdot \big[ D_1 f_{\bx} \, \ldots \, D_n f_{\bx} \big]^{\top} \in \R^{d} \ , \label{eq:Df-defn}
\end{equation}
such that
\begin{equation}
    \dd f_{\bx}(\bv) = \bv^{\top} \cdot D f_{\bx} \ , \qquad \forall \bx \in \mcX \ , \quad \forall \bv \in T_{\bx} \mcX \ .
\end{equation}
Note that the differential operator $\dd f$ is a generalization of the gradient. Hence, we can generalize the definition of the score function using the differential operator by setting $f$ to the logarithm of pdf. Therefore, the score function of $\bX$ under $\mcE^0$ is specified as follows:
\begin{equation}
    \bss_{\bX}(\bx) \triangleq D \log p_{\bX}(\bx) \ , \qquad \forall \bx \in \mcX \ . \label{eq:sx-def}
\end{equation}

\subsection{Intervention Mechanisms}\label{sec:intervention-mechanisms}

We consider two types of interventions. A \emph{soft} intervention on node $i$ (also referred to as \emph{imperfect} intervention in literature), changes the conditional distribution $p_i(z_i \med \bz_{\Pa(i)})$ to a distinct conditional distribution, which we denote by $q_i(z_i \med \bz_{\Pa(i)})$. A soft intervention does not necessarily remove the functional dependence of an intervened node on its parents but rather alters it to a different mechanism. A stochastic \emph{hard} intervention on node $i$ (also referred to as \emph{perfect} intervention) is stricter than a soft intervention and removes the edges incident on $i$. A hard intervention on node changes $p_i(z_i \med \bz_{\Pa(i)})$ to $q_i(z_i)$ that emphasizes the lack of dependence of $z_i$ on $\bz_{\Pa(i)}$. Finally, we note that in some settings, we assume two hard interventions per node, in which case the two hard interventional mechanisms for node $i$ are denoted by two distinct pdfs $q_i(z_i)$ and $\tilde q_i(z_i)$.

\paragraph{Interventional environments.}
We consider atomic \emph{interventional} environments in which each environment one node is intervened in, as is customary in the closely related CRL literature~\citep{squires2023linear,ahuja2023interventional,buchholz2023learning}. In some settings (linear transformation), we will have \emph{one} interventional environment per node and denote the interventional environments by $\mcE\triangleq \{\mcE^1,\dots,\mcE^n\}$, where we call $\mcE$ the \emph{atomic} environment set. We denote the node intervened in environment $\mcE^m$ by $I^m \in [n]$. For other settings (general transformation), we will have  \emph{two} interventional environments per node and denote the second atomic environment set by $\tilde\mcE=\{\tilde\mcE^1,\dots,\tilde\mcE^n\}$. Similarly, we denote the intervened node in $\tilde \mcE^m$ by $\tilde I^m$ for each $m\in[n]$. We assume that node-environment pairs are \emph{unspecified}, i.e.,  the ordered intervention sets  $\mcI\triangleq (I^1,\dots,I^n)$ and $\tilde \mcI \triangleq (\tilde I^1,\dots,\tilde I^n)$ are two \emph{unknown} permutations of $[n]$. We also adopt the convention that $\mcE^{0}$ is the \emph{observational} environment and $I^{0} \triangleq  \emptyset$. Next, we define the notion of coupling between the environment sets $\mcE$ and $\tilde \mcE$.

\begin{definition}[Coupled/Uncoupled Environments]
    The two environment sets $\mcE$ and $\tilde\mcE$ are said to be \textbf{coupled} if for the unknown permutations $\mcI$ and $\tilde \mcI$ we know that $\mcI=\tilde\mcI$, i.e., the same node is intervened in environments $\mcE^i$ and $\tilde \mcE^i$. The two environment sets are said to be \textbf{uncoupled} if $\tilde\mcI$ is an unknown permutation of $\mcI$.
\end{definition}
Next, we define $p^m$ as the pdf of $\bZ$ in environment~$\mcE^m$. Hence, under soft and hard intervention for each $m \in [n]$, $p^m$ can be factorized as follows. 
\begin{align}
    \mbox{soft int. in } \mcE^m: & \quad p^m(\bz) = q_\ell(z_{\ell} \med\bz_{\Pa(\ell)}) \prod_{i \neq \ell} p_i(z_i \med \bz_{\Pa(i)}) \ , && \;\; \mbox{where} \quad \ell=I^m \ ,\label{eq:pz_m_factorized_soft} \\
    \mbox{hard int. in } \mcE^m: & \quad p^m(\bz) = q_{\ell}(z_{\ell}) \prod_{i \neq \ell} p_i(z_i \med \bz_{\Pa(i)})\ , && \;\; \mbox{where} \quad  \ell=I^m \ . \label{eq:pz_m_factorized_hard} 
\end{align}
Similarly, we define $\tilde p^m$ as the pdf of $\bZ$ in $\tilde \mcE^m$,  which can be factorized similarly to \eqref{eq:pz_m_factorized_hard} with $\tilde q_{\ell}$ replaced with $q_{\ell}$. Hence, the score functions associated with $p^m$ and $\tilde p^m$ are specified as follows.
\begin{equation}
    \bss^{m}(\bz) \triangleq \nabla_{\bz} \log p^{m}(\bz) \ , \quad \mbox{and} \quad \tilde \bss^{m}(\bz) \triangleq \nabla_{\bz} \log \tilde p^{m}(\bz) \ . \label{eq:sz-defs}
\end{equation}
We denote the score functions of the observed variables $\bX$ under $\mcE^m$ and $\tilde \mcE^m$ by $\bss^{m}_{\bX}$ and $\tilde \bss^{m}_{\bX}$, respectively. Note that the score functions change across different environments, which is induced by the changes in the distribution of $\bZ$. 
Specifically, following \eqref{eq:pz_factorized} and \eqref{eq:pz_m_factorized_soft}, the latent scores $\bss(\bz)$ and $\bss^{m}(\bz)$ are decomposed as 
\begin{align}
    \bss(\bz) &= \nabla_{\bz} \log p_{\ell}(z_{\ell} \med\bz_{\Pa(\ell)}) + \sum_{i \neq \ell} \nabla_{\bz} \log p_i(z_i \med \bz_{\Pa(i)}) \ , \label{eq:s_z_decompose_obs} \\
    \mbox{and} \quad \bss^{m}(\bz) &= \nabla_{\bz} \log q_{\ell}(z_{\ell} \med\bz_{\Pa(\ell)}) + \sum_{i \neq \ell} \nabla_{\bz} \log p_i(z_i \med \bz_{\Pa(i)}) \ . \label{eq:s_z_decompose_int}
\end{align}
where $\ell = I^m$. Hence, $\bss(\bz)$ and $\bss^{m}(\bz)$ differ in only the causal mechanism of the intervened node in environment $\mcE^m$. In \Cref{sec:analysis-score}, we investigate these discrepancies between $\bss$ and $\bss^{m}$ (or $\tilde \bss^{m}$) and characterize the relationship between the scores in the observational and interventional environments.

\subsection{Identifiability and Achievability Objectives}\label{sec:objectives}
The objective of CRL is to use observations $\bX$ generated by the observational and interventional environments and estimate the true latent variables $\bZ$ and causal relations among them captured by~$\mcG$. The first objective is \emph{identifiability}, which pertains to determining algorithm-agnostic sufficient conditions under which $\bZ$ and $\mcG$ can be recovered uniquely up to a permutation and element-wise transform, which is the strongest form of recovery in CRL from interventions as shown in~\citep{vonkugelgen2023twohard}. The second objective is \emph{achievability}, which refers to designing algorithms that are amenable to practical implementation and generate provably correct estimates for $\bZ$ and $\mcG$, foreseen by the identifiability guarantees. In this subsection, we provide the definitions needed for formalizing these objectives. 

We denote a generic estimator of $\bZ$ given $\bX$ by $\hat \bZ(\bX):\R^d\to\R^n$. We also consider a generic estimate of $\mcG$ denoted by $\hat\mcG$. To assess the fidelity of the estimates $\hat \bZ(\bX)$ and $\hat\mcG$ with respect to the ground truth $\bZ$ and $\mcG$, we provide the following identifiability measures, which will be achieved when we have a complete set of atomic interventions.
\newpage
\begin{definition}[Latent Graph Identifiability]\label{def:latent-graph-identifiability} For identifiability of the latent graph, we define:
    \begin{enumerate}[leftmargin=2em]
        \item \textbf{Perfect DAG recovery:} DAG recovery is said to be perfect if $\hat \mcG$ is isomorphic to $\mcG$.    
        \item \textbf{\em Transitive closure recovery:} DAG recovery is said to maintain transitive closure if $\hat \mcG$ and $\mcG$ have the same ancestral relationships, i.e., ${\hat \mcG_{\rm tr}}$ is isomorphic to $\mcG_{\rm tr}$.
    \end{enumerate}
\end{definition}
\begin{definition}[Latent Variable Identifiability]\label{def:latent-variable-identifiability} For identifiability of \emph{all} latent variables, we define:
    \begin{enumerate}[leftmargin=2em]
        \item \textbf{Componentwise latent recovery:} The estimator $\hat \bZ(\bX)$ satisfies componentwise latent recovery if $\hat \bZ(\bX)$ is a componentwise diffeomorphism of a permutation of $\bZ$, i.e., there exists a permutation~$\pi$ of~$[n]$ and a set diffeomorphisms $\{\phi_i: \R \to \R : i\in[n]\}$ such that we have 
        \begin{align}\label{eq:perfect-latent-recovery}
            \hat \bZ(\bX) = (\pi \circ \phi) (\bZ) \ , \qquad \forall \bz\in\R^n\ ,
        \end{align}
        where $\phi(\bZ)\triangleq (\phi_1(Z_1),\dots,\phi_n(Z_n))$.
        \item \textbf{Scaling consistency:}  The estimator $\hat \bZ(\bX)$ satisfies scaling consistency if there exists a permutation $\pi$ of $[n]$ and a constant \emph{diagonal} matrix $\bC_{\rm scale} \in\R^{n\times n}$ such that
        \begin{equation}\label{eq:scaling_cons}
            \hat \bZ(\bX) = \bP_{\pi} \cdot \bC_{\rm scale} \cdot \bZ \ ,\qquad \forall \bz\in\R^n\ .
        \end{equation}
        \item \textbf{Consistency up to mixing with parents:} The estimator $\hat \bZ(\bX)$ satisfies consistency up to mixing with parents if there exists a permutation $\pi$ of $[n]$ and a \emph{constant} matrix $\bC_{\rm pa} \in\R^{n\times n}$ such that
        \begin{equation}\label{eq:mixing_cons_ancestors}
            \hat \bZ(\bX) = \bP_{\pi} \cdot \bC_{\rm pa} \cdot \bZ \ , \qquad \forall \bz\in\R^n\ ,
        \end{equation}
        where $\bC_{\rm pa}$ has nonzero diagonal entries and for all $j\notin\Paplus(i)$, $[\bC_{\rm pa}]_{i,j}= 0$. Equivalently, $\hat Z_{\pi_i}$ is a linear function of $\{Z_j : j \in \Paplus(i)\}$ for all $i\in[n]$.
        \item \textbf{Consistency up to mixing with surrounding parents:} The estimator $\hat \bZ(\bX)$ satisfies consistency up to mixing with surrounding parents if there exists a permutation $\pi$ of $[n]$ and a \emph{constant} matrix $\bC_{\rm sur} \in\R^{n\times n}$ such that
        \begin{equation}\label{eq:mixing_cons_surrounding}
            \hat \bZ(\bX) = \bP_{\pi} \cdot \bC_{\rm sur} \cdot \bZ \ , \qquad \forall \bz\in\R^n\ ,
        \end{equation}
        where $\bC_{\rm sur}$ has nonzero diagonal entries and for all $j \notin \overline{\sur}(i)$, $[\bC_{\rm sur}]_{i,j}= 0$. Equivalently, $\hat Z_{\pi_i}$ is a linear function of $\{Z_j : j \in \overline{\sur}(i)\}$ for all $i\in[n]$.
    \end{enumerate}
\end{definition}

We note that scaling consistency is a special case of component-wise latent recovery in which the diffeomorphism is restricted to scaling. 
Next, we provide \emph{node-level partial identifiability} definitions which measure the recovery level of a \emph{single} latent variable. These will be useful to assess the identifiability guarantees of the algorithms under an incomplete set of interventions.

\begin{definition}[Node-level Partial Identifiability] For identifiability of a \emph{single} latent variable, we define:
    \begin{enumerate}[leftmargin=2em]
        \item \textbf{Partial latent recovery:}
        The estimator $\hat \bZ(\bX)$ satisfies partial latent recovery \textbf{for node~$i$} if 
        \begin{equation}
            [\hat \bZ(\bX)]_i = \phi_k(Z_k) \ ,
        \end{equation}
        for some $k \in [n]$ where $\phi_k : \R \to \R$ is a diffeomorphism.
        \item \textbf{Partial consistency up to mixing with parents:} The estimator $\hat \bZ(\bX)$ satisfies partial consistency up to mixing with parents \textbf{for node $i$} if
        \begin{equation}
            [\hat \bZ(\bX)]_i = \sum_{j \in \Paplus(k)} c_k \cdot Z_k \ ,
        \end{equation}
        for some $k \in [n]$ and constants $\{c_j : j \in \Paplus(k)\}$.
    \end{enumerate}
\end{definition}

\paragraph{Tightness of the identifiability definitions.} For general transformations, \citet[Proposition 3.1]{vonkugelgen2023twohard} show that identifiability up to component-wise diffeomorphisms is the best possible result under interventions, without additional assumptions. For linear transformations, \citet[Proposition 2]{squires2023linear} and \citet[Remark 2]{buchholz2023learning} show that under hard interventions, identifiability up to scaling consistency and permutation ambiguity is the best achievable outcome without further information on $\bZ$. Furthermore, for soft interventions, transitive closure recovery is the best latent graph identifiability result without additional assumptions on the causal model ~\citep[Appendix J]{squires2023linear}. Finally, \citet[Theorem 3]{jin2024learning} show that, under some non-degeneracy assumptions, consistency up to mixing with surrounding parents is the optimal result for soft interventions. Therefore, given a complete set of atomic interventions, the identifiability definitions presented in this section represent the ultimate objectives when using interventions.

\subsection{Algorithm-related Definitions}

For formalizing the achievability results and designing the associated algorithms, generating the estimates $\hat \bZ(\bX)$ and $\hat\mcG$ is facilitated by estimating the inverse of $g$ based on the observed data $\bX$. Specifically, an estimate of $g^{-1}$, where $g^{-1}$ denotes the inverse of $g$, facilitates recovering $\bZ$ via $\bZ = g^{-1}(\bX)$ . Throughout the rest of this paper, we refer to $g^{-1}$ as the \emph{true encoder}. To formalize the procedures of estimating $g^{-1}$, we define $\mcH$ as the set of possible valid encoders, i.e., candidates for $g^{-1}$. A function $h\in\mcH$ can be such a candidate if it is invertible, that is, there exists an associated decoder $h^{-1}$ such that $(h^{-1} \circ h)(\bX) = \bX$. Hence, the set of valid encoders is specified by
\begin{equation}\label{eq:mcH}
    \mcH \triangleq \{ h: \mcX \to \R^n \;\colon\; \exists h^{-1}: \R^n \to \R^d \;\;  \mbox{such that }\;\; (h^{-1} \circ h)(\bx) = \bx \ , \;\; \forall \bx \in \mcX\} \ . 
\end{equation}
Next, corresponding to any pair of observation $\bX$ and valid encoder $h\in\mcH$, we define $\hat \bZ(\bX;h)$ as an \emph{auxiliary} estimate of $\bZ$ generated by applying the valid encoder $h$ on $\bX$, i.e., 
\begin{equation}\label{eq:encoder}
    \hat \bZ(\bX;h) \triangleq h(\bX) = (h \circ g) (\bZ) \ , \qquad \forall \; h\in\mcH \ .
\end{equation}
The estimate $\hat \bZ(\bX;h)$ inherits its randomness from $\bX$, and its statistical model is governed by that of~$\bZ$ under the chosen $h$. To emphasize the dependence on $h$, we denote the score functions associated with the pdfs of $\hat \bZ(\bX;h)$ under environments $\mcE^{0}$, $\mcE^m$, and $\tilde\mcE^m$, respectively, by 
\begin{equation}
 \bss_{\hat \bZ}(\cdot ;h)\ , \quad    \bss_{\hat \bZ}^{m}(\cdot ;h), \quad \mbox{and} \quad \tilde \bss_{\hat \bZ}^{m}(\cdot;h)\ .
\end{equation}
We will be addressing both general and linear transformations $g$. In the linear transformation setting, the true linear transformation $g$ is denoted by the matrix $\bG \in \R^{d \times n}$. Accordingly, we denote a valid linear encoder by $\bH \in \R^{n \times d}$. For a given valid encoder $\bH$, the associated valid decoder is given by its Moore-Penrose inverse, i.e., $\bH^{\dag} \triangleq \bH^{\top} \cdot (\bH \cdot \bH^{\top})^{-1}$. 
\section{Properties of Score Functions under Interventions}\label{sec:analysis-score}

Score functions and their variations across different interventional environments play pivotal roles in our approach to identifying latent representations. In this section, we present the key properties of the score functions that will be leveraged in \Cref{sec:linear} and \Cref{sec:general} to construct identifiability results, along with the corresponding algorithms.

We first investigate score variations across pairs of environments, such as the observational environment and an interventional one (under both soft and hard atomic interventions) or two interventional environments, either coupled or uncoupled. 
The key insight is that an intervention causes changes in only certain coordinates of the score function, as indicated by the sparse changes in the decompositions of the score functions in \eqref{eq:s_z_decompose_obs} and \eqref{eq:s_z_decompose_int}. These sparse changes further reflect the graph structure in the score differences. For instance, for a single-node intervention $I^m = \{i\}$ in $m$-th environment, $j$-th coordinate of the score difference $\bss - \bss^{m}$ becomes
\begin{equation}\label{eq:def-score-diff-z-int-obs-single}
    \big[\bss(\bz) - \bss^{m}(\bz)\big]_j = \frac{\partial}{\partial z_j} \log p_{i}(z_i \mid \bz_{\Pa(i))}) - \frac{\partial}{\partial z_j} \log q_{i}(z_i \mid \bz_{\Pa(i)}) \ ,
\end{equation}
which is zero if $j \notin \Paplus(i)$. The following lemma formalizes this property for all relevant cases.

\begin{lemma}[Score Changes under Interventions]\label{lm:parent_change_comprehensive}
    Consider the observational environment $\mcE^0$ and an interventional environment $\mcE^m$ with unknown intervention target $I^m$. Then, for any causal model and intervention type,  $\E\Big[\big|\bss(\bZ) - \bss^{m}(\bZ)\big|_i\Big] \neq 0$ implies that $i \in \overline{\Pa}(I^m)$. For the reverse direction, we have the following results specified for each relevant case.
    \begin{enumerate}[label=(\roman*),leftmargin=2em, itemsep=-.05 in]
        \item \textbf{\em Hard interventions:} If the intervention in $\mcE^m$ (or $\tilde \mcE^m$) is hard, then score functions $\bss$ and $\bss^{m}$ (or $\tilde \bss^{m}$) differ in their $i$-th coordinate if and only if node $i$ or one of its children is intervened in $\mcE^m$ (or in $\tilde \mcE^m$). \vspace{-\baselineskip}
        \begin{align}
            \label{eq:scorechanges_iff2} \E\Big[\big|\bss(\bZ) - \bss^{m}(\bZ)\big|_i\Big] \neq 0  &\quad \iff \quad i \in \overline{\Pa}(I^m)\ ,  \\ 
            \label{eq:scorechanges_iff3} \mbox{and} \quad  \E\Big[\big|\bss(\bZ) - \tilde \bss^{m}(\bZ)\big|_i\Big] \neq 0 & \quad \iff \quad i \in \overline{\Pa}(\tilde I^m)\ . 
        \end{align}
        
        \item \textbf{\em Soft interventions:} If the intervention in $\mcE^m$ is soft and the latent causal model is an additive noise model, then score functions $\bss$ and $\bss^{m}$ differ in their $i$-th coordinate if and only if node $i$ or one of its children is intervened in $\mcE^m$.
        \begin{equation}\label{eq:scorechanges_iff1}
            \E\Big[\big|\bss(\bZ) - \bss^{m}(\bZ)\big|_i\Big] \neq 0  \quad \iff \quad i \in \overline{\Pa}(I^m)\ . 
        \end{equation}

        \item \textbf{\em Coupled environments $I^m = \tilde I^m$:} In the coupled environment setting, $\bss^{m}$ and $\tilde \bss^{m}$ differ in their $i$-th coordinate if and only if $i$ is intervened. 
         \begin{equation}\label{eq:scorechanges_iff4}
            \E\Big[\big|\bss^{m}(\bZ) - \tilde \bss^{m}(\bZ)\big|_i\Big] \neq 0 \quad \iff \quad i = I^m  \ . 
         \end{equation}

         \item \textbf{\em Uncoupled environments $I^m \neq \tilde I^m$:} Consider two interventional environments $\mcE^m$ and $\tilde \mcE^m$ with different intervention targets $I^m \neq \tilde I^m$, and consider additive noise models specified in \eqref{eq:additive-SCM}. Given that $p(\bZ)$ is twice differentiable, the score functions $\bss^{m}$ and $\tilde \bss^{m}$ differ in their $i$-th coordinate if and only if node $i$ or one of its children is intervened.
         \begin{equation}
               \E\Big[\big|\bss^{m}(\bZ) - \tilde \bss^{m}(\bZ)\big|_i\Big] \neq 0 \quad \iff \quad i \in \overline{\Pa}(I^m,\tilde I^m)  \ .  
         \end{equation}
    \end{enumerate}
\end{lemma}
\proof See \Cref{sec:proof-parent-change}.

\Cref{lm:parent_change_comprehensive} provides the necessary and sufficient conditions for the invariance of the coordinates of the score functions of latent variables. 
The core idea of the score-based framework is that tracing these sparse changes in the score functions of the latent variables guides finding reliable estimates for the inverse of the transformation $g$, which in turn facilitates estimating $\bZ$. Intuitively, we will look for the encoders $h \in \mcH$ such that the variations between the score estimates $\bss_{\hat \bZ}(\hat \bz;h)$, $\bss_{\hat \bZ}^{m}(\hat \bz;h)$, and $\tilde \bss_{\hat \bZ}^{m}(\hat \bz;h)$ will be similar to the true score variations given by \Cref{lm:parent_change_comprehensive}.
However, the scores of the latent variables are not directly accessible. To circumvent this, we need to understand the connection between score functions of $\bX$ and $\bZ$. In the following lemma, we leverage the change of variables formula for injective mappings and establish this relationship for \emph{any} injective mapping $f$ from latent to observed space. \looseness=-1

\begin{lemma}[Score Difference Transformation]\label{lm:score-difference-transform-general}
    Consider random vectors $\bY_1,\bY_2\in\R^r$ and $\bW_1$, $\bW_2 \in \R^s$ that are related through
    \begin{equation}
        \bY_1=f(\bW_1) \ , \quad \mbox{and} \qquad \bY_2=f(\bW_2)\ ,
    \end{equation}
    such that $r \geq s$, probability measures of $\bW_1,\bW_2$ are absolutely continuous with respect to the $\bss$-dimensional Lebesgue measure, and $f: \R^s \to \R^r$ is an injective and continuously differentiable function. The difference of the score functions of $\bY_1$ and $\bY_2$, and that of $\bW_1$ and $\bW_2$ are related as
    \begin{equation}\label{eq:score-observed-to-latent}
        s_{\bW_1}(\bw)-s_{\bW_2}(\bw) = \big[J_f(\bw)\big]^{\top} \cdot \big[s_{\bY_1}(\by) - s_{\bY_2}(\by)\big] \ , \quad \mbox{where} \;\; \by=f(\bw) \ ,
    \end{equation}
    where $J_{f}(\bw)$ denotes the Jacobian of $f$ at point $w$.
   Furthermore, for the reverse direction, we have
    \begin{equation}\label{eq:score-latent-to-observed}
        s_{\bY_1}(\by) - s_{\bY_2}(\by) = \Big[\big[J_f(\bw)\big]^{\dag}\Big]^{\top} \cdot \big[s_{\bW_1}(\bw)-s_{\bW_2}(\bw)\big] \ , \quad \mbox{where} \;\; \by=f(\bw) \ .
    \end{equation}
\end{lemma}
\proof See \Cref{sec:score-diff-lm-proof}.\\
For CRL under linear transformations in \Cref{sec:linear}, we use the following corollary of \Cref{lm:score-difference-transform-general}.
\begin{corollary}\label{corollary:score-difference-transform-linear}
    In \Cref{lm:score-difference-transform-general}, if $f$ is a linear transform, that is $\bY = \bF \cdot \bW$ for a full-rank matrix $\bF \in \R^{r \times s}$, then the score functions of $\bY$ and $\bW$ are related through $\bss_{\bW}(\bw) = \bF^{\top} \cdot \bss_{\bY}(\by)$ and $\bss_{\bY}(\by) = \big[\bF^{\dag}\big]^{\top} \cdot \bss_{\bW}(\bw)$, where $\by= \bF \cdot \bw$. 
\end{corollary}
For CRL under general transformations in \Cref{sec:general}, we customize \Cref{lm:score-difference-transform-general} as follows. Consider a candidate encoder $h \in \mcH$. Recall that  $\hat \bZ(\bX;h) = h(\bX) = (h \circ g)(\bZ)$. Then, by letting $f = (h \circ g)$, \Cref{lm:score-difference-transform-general} gives 
\begin{align}
    \mbox{between }\mcE^{0} \mbox{ and } \mcE^m&: \quad \bss_{\hat \bZ}(\hat \bz;h) - \bss_{\hat \bZ}^{m}(\hat \bz;h) = J_{f}^{-\top}(\bz) \cdot \big[\bss(\bz) - \bss^{m}(\bz)\big] \ ,  \label{eq:score-difference-z-zhat-1} \\ 
    \mbox{between }\mcE^{0} \mbox{ and } \tilde \mcE^m&: \quad \bss_{\hat \bZ}(\hat \bz;h) - \tilde \bss_{\hat \bZ}^{m}(\hat \bz;h) = J_{f}^{-\top}(\bz) \cdot \big[\bss(\bz) - \tilde \bss^{m}(\bz)\big] \ , \label{eq:score-difference-z-zhat-2}  \\
    \mbox{between }\mcE^m \mbox{ and } \tilde \mcE^m&: \quad \bss_{\hat \bZ}^{m}(\hat \bz;h) - \tilde \bss_{\hat \bZ}^{m}(\hat \bz;h) = J_{f}^{-\top}(\bz) \cdot \big[\bss^{m}(\bz) - \tilde \bss^{m}(\bz)\big] \ .  \label{eq:score-difference-z-zhat} 
\end{align}

\section{CRL under Linear Transformations}\label{sec:linear}

In this section, we consider CRL under linear transformation, in which the general transformation model in~\eqref{eq:data-generation-process} becomes:
\begin{equation}\label{eq:data-generation-process-linear}
    \bX = \bG \cdot \bZ \ ,
\end{equation}
where $\bG \in \R^{d \times n}$ is an \emph{unknown} full-rank matrix mapping the latent variables to the observed ones. We will present steps for leveraging the properties of the score functions presented in \Cref{sec:analysis-score} to design an algorithm that identifies the true encoder and recovers the true causal representations. The algorithm is referred to as {\bf L}inear {\bf S}core-based {\bf Ca}usal {\bf L}atent {\bf E}stimation via {\bf I}nterventions (LSCALE-I). The theoretical guarantees associated with the algorithm steps will also be presented, serving as constructive proof of identifiability.

LSCALE-I consists of three stages outlined in \Cref{alg:linear}. In the first stage, we will learn an encoder $\hat \bH$, row by row, where each row is an estimate of a row of the true encoder $\bGinv$. This encoder estimate will recover each $Z_i$ up to mixing with $\bZ_{\Pa(i)}$. In the second stage, we will use the estimated latent variables $\hat \bH$ to obtain the transitive closure of the latent graph $\mcG$. Finally, as an optional third stage for \emph{hard} interventions, we will leverage the independence statements implied by the hard interventions to refine our encoder estimate. We will demonstrate that this refinement yields scaling consistency and perfect DAG recovery. We will also discuss the relevance and distinctions of our results vis-à-vis the results in the existing literature.

\paragraph{Statistical diversity.} 
Similarly to the existing identifiability results, it is necessary to have some regularity conditions~\footnote{Some examples include \emph{generic interventions} in~\citep{squires2023linear}, \emph{no pure shift interventions} condition in~\citep{buchholz2023learning}, and the \emph{genericity condition} in~\citep{vonkugelgen2023twohard}.} on the probability distributions of observational and interventional environments. We adopt the following assumption for the case of linear transformations.
\begin{assumption}\label{assumption:rank-two}
For every possible pair $(i,k)$ where $i\in[n], k \in \Pa(i)$, the following term cannot be a constant function in $\bz$,
\begin{equation}\label{eq:ratio-SN}
     \frac{\partial}{\partial z_k} \left(\log\frac{p_i(z_i \mid \bz_{\Pa(i)})}{q_i(z_i \mid \bz_{\Pa(i)})} \right)\left[\frac{\partial}{\partial z_i} \log\frac{p_i(z_i \mid \bz_{\Pa(i)})}{q_i(z_i \mid \bz_{\Pa(i)})}\right]^{-1}\ .
\end{equation}
\end{assumption}
Essentially, $\frac{\partial}{\partial z_k} \big(\log\frac{p_i(z_i \mid \bz_{\Pa(i)})}{q_i(z_i \mid \bz_{\Pa(i)})} \big)$ captures the effect of intervening on node $i$ on the \emph{score} associated with node $k$. Therefore, \Cref{assumption:rank-two} ensures that an intervention sufficiently differentiates the target variable and its parents in the score function. We note that this is a very mild assumption and holds for a wide range of commonly used models (including additive noise models under hard interventions) and is discussed in more detail in \Cref{sec:discuss-assumptions}.

Before we explain the details of LSCALE-I and establish identifiability results in the following subsections, we first present the rationale and derive the partial identifiability result, which will serve as the building block. In the rest of the paper, for brevity, we denote the score difference functions for each $m\in[n]$ by
\begin{equation}\label{eq:def-score-diff-z-int-obs}
\bsd_{\bZ}^{m}(\bz) \triangleq \bss^{m}(\bz) - \bss(\bz)   \quad \mbox{and} \quad \bsd_{\bX}^{m}(\bx) \triangleq \bss_{\bX}^{m}(\bx) - \bss_{\bX}(\bx) \ .
\end{equation}

\subsection{Rationale for LSCALE-I and Partial Identifiability}\label{sec:linear-rationale}

To construct an algorithm that achieves identifiability under linear transformations, we first analyze the score differences for the linear transformation. Note that for linear transformation $\bX = \bG \cdot \bZ$, Jacobian $J_{\bG}(\bz)$ is independent of $\bz$ and equal to matrix $\bG$. Also, we have $\bz = \bGinv \cdot \bx$ for all $\bx \in \colspace(\bG)$. Then, using \Cref{corollary:score-difference-transform-linear} we have 
\begin{align}\label{eq:sx-from-sz-linear}
    \bsd_{\bX}^{m}(\bx) = [\bGinv]^{\top} \cdot \bsd_{\bZ}^m(\bz) \ .  
\end{align}
Let us denote the correlation matrices of $\bsd_{\bX}^{m}(\bx)$ and $\bsd_{\bZ}^{m}(\bz)$, respectively, by
\begin{equation}
    \bR_{\bX}^{m} \triangleq \E \big[\bsd_{\bX}^m(\bx) \cdot [\bsd_{\bX}^m(\bx)]^{\top} \big] \  \quad \mbox{and} \quad
    \bR_{\bZ}^{m} \triangleq \E \big[\bsd_{\bZ}^m(\bz) \cdot [\bsd_{\bZ}^m(\bz)]^{\top} \big]  \ . \label{eq:rx-rz-defn}
\end{equation}
Note that \eqref{eq:sx-from-sz-linear} implies that  $\bR_{\bX}^{m} = [\bGinv]^{\top} \cdot \bR_{\bZ}^{m} \cdot \bGinv$. The next lemma specifies that the structure of the column space of correlation matrix $\bR_{\bX}^{m}$ is heavily constrained by the true encoder~$\bGinv$, the graph $\mcG$, and the intervention target $I^m$ in $m$-th environment.
\begin{lemma}
\label{lm:dx-colspace}
    For any interventional environment $\mcE^m$, $\colspace(\bR_{\bX}^{m}) \subseteq \spann\big\{[\bGinv_i]^{\top} \,:\, i \in \Paplus(I^m)\big\}$.
\end{lemma}
\begin{proof}
    Using the sparse score changes property in \eqref{eq:def-score-diff-z-int-obs-single}, we know that $\big[\bsd_{\bZ}^{m}(\bz)\big]_{i} = 0$ for all $i \notin \Paplus(I^m)$. Then, using \eqref{eq:sx-from-sz-linear}, for any $\bx = \bG \cdot \bz$ we have
    \begin{equation}\label{eq:dx-colspace-spell}
        \bsd_{\bX}^{m}(\bx) = \sum_{i \in \Paplus(I^m)} [\bGinv_i]^{\top} \cdot \big[\bsd_{\bZ}^{m}(\bz)\big]_{i}  \; \in \;  \spann\big\{ \; [\bGinv_i]^{\top}  :  i \in \Paplus(I^m) \; \big\} \ . 
    \end{equation}
    Since this holds for any $\bx \in \colspace(\bG)$, we have 
    \begin{equation}
        \spann \big\{\, \bsd_{\bX}^{m}(\bx) \,:\, \bx \in \colspace(\bG) \,\big\} \subseteq \spann\big\{\,[\bGinv_i]^{\top} \,:\, i \in \Paplus(I^m)\,\big\} \ .
    \end{equation}
    By definition of $\bR_{\bX}^{m}$, we have $\colspace(\bR_{\bX}^{m}) = \spann \{\bsd_{\bX}^{m}(\bx) : \bx \in \colspace(\bG)\}$ which concludes the proof.     
\end{proof}
\Cref{lm:dx-colspace} crucially implies that we can achieve partial identifiability of a single latent variable given that we have an environment in which the said variable is intervened.

\begin{theorem}[Linear -- Node-level partial identifiability]\label{thm:linear-partial-identifiability}
  A single-node soft interventional environment $\mcE^m$ with intervention target $\ell = I^m$ is sufficient to recover $Z_\ell$ up to mixing with its parents.
\end{theorem}
\begin{proof}
    Let us pick $\by$ uniformly at random from unit sphere $\mathbb{S}^{d-1} \subset \R^{d}$ and let $\bh = \bR_{\bX}^{m} \cdot \by$. Note that $\big[\bsd_{\bZ}^{m}(\bz)\big]_{i} = 0$ for all $i \notin \Paplus(\ell)$ implies that $\big[\bR_{\bZ}^{m}\big]_i = \boldsymbol{0}$ for all $i \notin \Paplus(\ell)$. Then, using $\bR_{\bX}^{m} = [\bGinv]^{\top} \cdot \bR_{\bZ}^{m} \cdot \bGinv$, we have
    \begin{align}
        \bh = \bR_{\bX}^{m} \cdot \by &= [\bGinv]^{\top} \cdot \bR_{\bZ}^{m} \cdot \bGinv \cdot \by \\
        &= \sum_{i \in \Paplus(\ell)}  \big[\bR_{\bZ}^{m}\big]_i \cdot \bGinv \cdot \by  \cdot [\bGinv_i]^{\top} 
    \end{align}
    Denote $c_i = \big[\bR_{\bZ}^{m}\big]_i \cdot \bGinv \cdot \by$ for each $i \in \Paplus(\ell)$. Then, using $\bGinv_i \cdot \bG = \be_i$, $i$-th standard basis vector, we have $[\bh^{\top} \cdot \bG]_{i}=c_i$ for $i \in \Paplus(\ell)$. As such, we can use $\bh$ to obtain
    \begin{equation}\label{eq:partial-node-recovery}
        \bh^{\top} \cdot \bX = \bh^{\top} \cdot \bG \cdot \bZ = c_\ell \cdot Z_{\ell} + \sum_{i \in \Pa(\ell)} c_i \cdot Z_i \ . 
    \end{equation}
    Finally, note that $\big[\bR_{\bZ}^{m}\big]_\ell$ is not a zero vector since $[\bsd_{\bZ}^{m}]_{\ell} \neq 0$. Therefore, since rows of $\bGinv$ are linearly independent, we know that $\big[\bR_{\bZ}^{m}\big]_i \cdot \bGinv$ is a nonzero vector. Thus, $\by \in \mathbb{S}^{d-1}$ ensures that $c_{\ell}$ is nonzero with probability 1. Then, $\bh^{\top} \cdot \bX$ is an estimate of $Z_{\ell}$ that satisfies consistency up to mixing with parents.        
\end{proof}

\subsection{Identifiability via Soft Interventions}\label{sec:linear-soft}

After obtaining the partial identifiability result from a single interventional environment, our goal is to construct an algorithm that uses the complete set of interventional environments $\mcE = \{\mcE^1,\dots,\mcE^n\}$ with targets $\cup_{m \in [n]} I^m = [n]$ to form estimates of the true encoder $\bGinv$ and the latent graph $\mcG$.

\paragraph{Encoder estimation.}
In Stage~L1 of \Cref{alg:linear}, we form an encoder estimate $\hat \bH$ via using the column spaces of $\{\bR_{\bX}^{m} : m \in [n]\}$. Specifically, for each $m \in [n]$ we select a vector randomly from $\colspace(\bR_{\bX}^{m})$ and assign its transpose to $m$-th row of $\hat \bH$.
\begin{lemma}[Linear -- Encoder via Soft Interventions]\label{lm:linear-encoder-up-to-parents}
Under soft interventions, output $\hat \bH$ of \Cref{alg:linear} achieves identifiability up to mixing with parents. Specifically, $\hat \bZ(\bX;\hat \bH) = \bP_{\mcI} \cdot \bC_{\rm pa} \cdot \bZ$ such that diagonal entries of $\bC_{\rm pa}$ are nonzero and $\big[\bC_{\rm pa}\big]_{i,j}=0$ for all $j \notin \Paplus(i)$.
\end{lemma}
\begin{proof}
    Following \Cref{thm:linear-partial-identifiability} for all $m \in [n]$ immediately implies the desired result. Specifically, according to \eqref{eq:partial-node-recovery}, $[\hat \bH_{m}]^{\top} \in \colspace(\bR_{\bX}^{m})$ satisfies
    \begin{equation}
        [\hat \bZ(\bX; \hat \bH)]_m = \hat \bH_{m} \cdot \bX = c_{\ell} \cdot Z_{\ell} + \sum_{i \in \Pa(\ell)} c_i \cdot Z_i \ , \qquad \mbox{where} \; \ell = I^m \ ,
    \end{equation}
    for some constants $\{c_i : i \in \Paplus(\ell)\}$ where $c_{\ell} \neq 0$ with probability 1. Then, combining all $n$ identities, we have 
    \begin{equation}
        \hat \bZ(\bX; \hat \bH) = \bP_{\mcI} \cdot \bC_{\rm pa} \cdot \bZ \ ,
    \end{equation}
    where $\bP_{\mcI}$ is the permutation matrix of the intervention order $(I^1,\dots,I^n)$, and $\bC_{\rm pa}$ has nonzero diagonal entries and satisfies $\big[\bC_{\rm pa}\big]_{i,j} = 0$ for all $j \notin \Paplus(i)$.
\end{proof}

\paragraph{Latent graph estimation.}
Next, we use the encoder estimate $\hat \bH$ to recover the transitive closure of the latent graph $\mcG$. Motivated by \Cref{lm:parent_change_comprehensive}(ii), we form $\hat \mcG$ with parent sets 
\begin{equation}\label{eq:linear-soft-parent-set}
   \hat \Pa(m) \triangleq \Big\{i \neq m: \E\Big[\big| \bss_{\hat \bZ}(\hat \bZ; \hat \bH) - \bss_{\hat \bZ}^{m}(\hat \bZ; \hat \bH) \big|_i \Big] \neq 0 \Big\} \ .
\end{equation}
Using the result that $\hat \bZ(\bX; \hat \bH)$ recovers $\bZ$ up to mixing with parents, we show that the transitive closures of the estimate $\hat \mcG$ and the true graph $\mcG$ are the same under a graph isomorphism.

\begin{algorithm}[t]
    \caption{{\bf L}inear {\bf S}core-based {\bf Ca}usal {\bf L}atent {\bf E}stimation via {\bf I}nterventions (LSCALE-I)}
    \label{alg:linear}
    \begin{algorithmic}[1]
    \State \textbf{Input:} $\bR_{\bX}^{m}$ for all $m \in [n]$ \Comment{compute using \eqref{eq:rx-rz-defn}}
    \Statex \hrulefill
    \State \textbf{Stage~L1: Encoder estimation}
    \State{$\hat \bH \gets \bzero_{n \times d}$}
    \For{$m \in (1, \dots, n)$}
        \State Select any $\bh \in \colspace(\bR_{\bX}^{m})$ 
        \State $\hat \bH_{m} \gets \bh^{\top}$
    \EndFor
    \Statex \hrulefill
    \State \textbf{Stage~L2: Latent graph estimation}
    \State Construct latent DAG estimate $\hat \mcG$ with parent sets
    \begin{equation*}
        \hat \Pa(m) \gets \Big\{i \neq m: \E\Big[\big| \bsd_{\hat \bZ}^{m}(\hat \bZ; \hat \bH)\big|_i \Big] \neq 0 \Big\} \ , \quad \forall m \in [n]
    \end{equation*}
    \Statex \hrulefill
    \If{the interventions are hard}
    \State \textbf{Stage~L3: Unmixing procedure}
        \State $\pi \gets$ a topological order of $\hat \mcG$
        \For{$m\in(\pi_1,\dots,\pi_n)$} \Comment{refine rows of $\hat \bH$ sequentially}
            \State $\hat \bZ \gets \hat \bZ^{m}(\bX;\hat \bH)$
            \State $\bu \gets \Cov(\hat Z_m, \hat \bZ_{\hat \Pa(m)}) \cdot \big[\Cov(\hat \bZ_{\hat \Pa(m)})\big]^{-1} $
            \State $\hat \bH_{m} \gets \hat \bH_{m} - \bu \cdot \hat \bH_{\hat \Pa(m)}$
        \EndFor
    \State Refine $\hat \mcG$ according to 
    \begin{equation*}
        \hat \Pa(m) \gets \Big\{i \neq m: \E\Big[\big| \bsd_{\hat \bZ}^{m}(\hat \bZ; \hat \bH)\big|_i \Big] \neq 0 \Big\} \ , \quad \forall m \in [n]
    \end{equation*}
    \EndIf
    \Statex \hrulefill
    \State \textbf{Return} $\hat \mcG$, $\hat \bH$, and $\hat \bZ$
    \end{algorithmic}
\end{algorithm}%

\begin{lemma}[Linear -- Graph via Soft Interventions]\label{lm:linear-graph-transitive-closure}
    Under \Cref{assumption:rank-two} and soft interventions, output $\hat \mcG$ of \Cref{alg:linear} recovers the transitive closure of $\mcG$.
\end{lemma}
\begin{proof}
    Let $\hat \bZ$ denote $\hat \bZ(\bX; \hat \bH) = \bP_{\mcI} \cdot \bC_{\rm pa} \cdot \bZ$. Recall that $\big[\bC_{\rm pa}\big]_{i,j}=0$ for all $j \notin \Paplus(i)$, and since nodes $(1, \dots, n)$ are topologically ordered, $\bC_{\rm pa}$ is a lower triangular matrix. By using standard linear algebra for the expansion of an inverse of a lower triangular matrix, we can show that $\big[\bC_{\rm pa}^{-1}\big]_{i,j}=0$ for all $j \notin \Anplus(i)$ (see \Cref{lm:binary-matrices} in \Cref{sec:auxiliary-lemmas} for the detailed steps). 
    Then, using \Cref{corollary:score-difference-transform-linear} for the scores of $\bZ$, we have 
    \begin{align}\label{eq:dzhat-upto-parents}
        \bsd_{\hat \bZ}^{m}(\hat \bZ) &= \bP_{\mcI} \cdot [\bC_{\rm pa}^{-1}]^{\top} \cdot \bsd_{\bZ}^{m}(\bZ) \ . 
    \end{align}
    Next, taking the transpose of $\bC_{\rm pa}^{-1}$, we have $\big[\bC_{\rm pa}^{-\top}\big]_{i,j}=0$ for all $j \notin \Deplus(i)$. Combining this observation with the fact that $[\bsd_{\bZ}^{m}]_j=0$ for all $j \notin \Paplus(I^m)$, we obtain
    \begin{equation}
     \big[\bsd_{\hat \bZ}^{m}(\hat \bZ)\big]_k =    \big[\bC_{\rm pa}^{-\top}\big]_{I^k} \cdot \bsd_{\bZ}^{m}(\bZ) =  \sum_{j \; \in \; \Deplus(I^k) \; \cap \; \Paplus(I^m)} \big[\bC_{\rm pa}^{-\top}\big]_{I^k,j} \cdot [\bsd_{\bZ}^{m}]_j 
    \end{equation}
    Subsequently, $\E\big[|\bsd_{\hat \bZ}^{m}(\hat \bZ)|_k\big] \neq 0$ implies that $\Deplus(I^i) \cap \Paplus(I^m) \neq \emptyset$, and $I^k \notin \Anplus(I^m)$. Therefore, the estimated parent set in \eqref{eq:linear-soft-parent-set} satisfies $\hat \Pa(m) \subseteq \Anplus(I^m)$, and the transitive closure of $\mcG$ is a supergraph of $\hat \mcG$ under relabeling of the nodes with permutation $\mcI$. Next, we show that $\hat \mcG$ contains all the edges in the transitive reduction of $\mcG$. Let $I^k \rightarrow I^m$ be an edge in the transitive reduction of $\mcG$, i.e., there is no other directed path between nodes $I^k$ and $I^m$ in $\mcG$. This implies that $\Deplus(I^k) \cap \Paplus(I^m) = \{I^k, I^m\}$, and 
    \eqref{eq:dzhat-upto-parents} becomes 
    \begin{equation}
        \big[\bsd_{\hat \bZ}^{m}(\hat \bZ)\big]_{k} =  \big[\bC_{\rm pa}^{-\top}\big]_{I^k,I^k} \cdot [\bsd_{\bZ}^{m}]_{I^k} + \big[\bC_{\rm pa}^{-\top}\big]_{I^k,I^m} \cdot [\bsd_{\bZ}^{m}]_{I^m} \ .
    \end{equation}
    \Cref{assumption:rank-two} ensures that $\big[\bsd_{\hat \bZ}^{m}(\hat \bZ)\big]_{k}$ is not constantly zero, which means that $\hat \mcG$ contains $k \rightarrow m$ edge, preserving $I^k \rightarrow I^m$ in $\mcG$. Therefore, $\hat \mcG$ preserves all the edges in the transitive reduction of $\mcG$, and taking its transitive closure gives $\mcG_{\rm tc}$, under relabeling of the nodes with permutation $\mcI$.    
\end{proof}

\begin{theorem}[Linear -- Soft Interventions for General SCMs]\label{th:linear-soft}
Under \Cref{assumption:rank-two} for linear transformations, using observational data and interventional data from one \emph{soft} intervention per node suffice to identify (i) the transitive closure of the latent DAG $\mcG$ and (ii) the latent variables $\bZ$ with consistency up to mixing with parents. Specifically, \Cref{alg:linear} achieves these identifiability guarantees.
\end{theorem}
\begin{proof}
    Combining \Cref{lm:linear-encoder-up-to-parents}
 and \Cref{lm:linear-graph-transitive-closure} gives the statements of the theorem.
\end{proof}

We note that the existing literature on CRL with linear transformations and one soft intervention requires the latent causal model to be either linear Gaussian~\citep{squires2023linear,buchholz2023learning} or satisfy nonlinearity conditions~\citep{zhang2023identifiability}\footnote{The ``linear interventional faithfulness'' \citep[Assumption 2]{zhang2023identifiability} implies nonlinearity, which we elaborate in \Cref{appendix:zhang2023-nonlinearity}.}. In contrast, \Cref{th:linear-soft} achieves identifiability guarantees for soft interventions without imposing any restrictions (distributional or structural) on the latent causal model. Furthermore, recovering $\bZ$ up to mixing with \emph{parents} improves upon the existing guarantees of recovering up to mixing with \emph{ancestors}~\citep{squires2023linear,buchholz2023learning}. Finally, we note that $\mcG$ cannot be identified beyond its transitive closure when using soft interventions without making additional assumptions~\citep[Appendix~J]{squires2023linear}. Therefore, the graph identifiability result in \Cref{th:linear-soft} is tight.

\subsection{Identifiability via Hard Interventions}\label{sec:linear-unmixing}

Next, we investigate hard interventions (i.e., perfect interventions) for general latent causal models. Hard interventions are special cases of soft interventions, in which the intervened node becomes functionally independent of its parents. Consequently, the identifiability guarantees for hard interventions are usually stronger, as we show in this section. Our analysis builds on leveraging the following property, which is exclusive to hard interventions.
\begin{proposition}\label{fact:hard-indep}
    For the environment $\mcE^{m}$ in which node $\ell = I^m$ is hard intervened, we have
    \begin{align}\label{eq:hard-indep}
        Z_{\ell}^{m} \ci Z_j^{m}\ , \quad \forall j \in \nd(\ell) \ ,
    \end{align}
    where $\nd(\ell)$ is the set of non-descendants of $\ell$ in $\mcG$.
\end{proposition}
The Markov property readily implies this property, that is, each variable in a DAG is independent of its non-descendants, given its parents. When node $\ell$ is hard-intervened, it has no parents, and the statement follows directly. Motivated by this property, we want to ensure that the estimated latent variables conform to \Cref{fact:hard-indep}. To this end, we aim to update rows of the encoder estimate $\hat \bH$ to remove the effects of $\bZ_{\Pa(i)}$ from the estimate of $Z_i$. We achieve this objective as follows. 

Note that \Cref{fact:hard-indep} provides us with random variable pairs that are supposed to be independent, and the covariance of two independent random variables is necessarily zero. In Stage~L3 of \Cref{alg:linear}, for each row $\hat \bH_m$, we consider $\hat \bZ(\bX; \hat \bH)$ in the corresponding environment in which $\hat Z_m$ is intervened, and compute an unmixing vector $\bu$ via a linear minimum mean square error (MMSE) estimator,
\begin{equation}
    \bu = \Cov(\hat Z_m, \hat \bZ_{\hat \Pa(m)}) \cdot \big[\Cov(\hat \bZ_{\hat \Pa(m)}) \big]^{-1} \ .
\end{equation}
Then, we update the corresponding row of the encoder as $\hat \bH_m \gets \hat \bH_m - \bu \cdot \hat \bH_{\hat \Pa(m)}$. Finally, for the perfect DAG recovery, we compute the latent score differences under the final encoder estimate and reconstruct the graph $\hat \mcG$ with parent sets 
\begin{equation}\label{eq:linear-hard-parent-set}
   \hat \Pa(m) \triangleq \Big\{ i \neq m :  \E\Big[\big| \bsd_{\hat \bZ}^{m}(\hat \bZ; \hat \bH)\big|_i \Big] \neq 0 \Big\} \ .
\end{equation}

\begin{theorem}[Linear -- Hard Interventions]\label{th:linear-hard}
    Under \Cref{assumption:rank-two} for linear transformations, using observational data and interventional data from one \emph{hard} intervention per node suffices to identify (i) the latent DAG $\mcG$ perfectly and (ii) the latent variables $\bZ$ with scaling consistency. Specifically, \Cref{alg:linear} achieves these identifiability guarantees.
\end{theorem}
\proof See \Cref{proof:linear-hard}.

Similar to the restrictions in the existing results for soft interventions, the identifiability results for hard interventions in the existing literature restrict the latent causal model, e.g., to linear causal models~\citep{squires2023linear,buchholz2023learning}. In contrast, \Cref{th:linear-hard} does not impose any restriction on the latent causal model and shows that one stochastic hard intervention per node is sufficient for the identifiability of general latent causal models. 

Finally, we comment on the differences between the identifiability results for soft and hard interventions under missing intervention targets and discuss the identifiability of latent subgraphs under special cases.
\begin{remark}\label{remark:non-exhaustive-interventions}
    Let $\mcJ \subset [n]$ denote a non-complete set of intervention targets, i.e., $\mcJ \neq [n]$. We note the immediate implications of our results as follows.
    \begin{enumerate}[leftmargin=2em]
        \item \textbf{Partial identifiability of latent variables.} \Cref{thm:linear-partial-identifiability} immediately implies that for all $i \in \mcJ$ we can recover $Z_i$ up to mixing with parents. However, the unmixing procedure for hard interventions resolves the mixing with parents sequentially. Specifically, to identify $Z_i$ up to scaling, we use the fact that the parents of $Z_i$ are already identified up to scaling. Therefore, even if we have a hard intervention on $Z_i$, we cannot identify $Z_i$ up to scaling consistency when its parents are not identified up to scaling.
        \item \textbf{Ancestrally closed set of interventions.} Let $\mcG_{\mcJ}$ denote the induced subgraph of $\mcG$ over the set of nodes $\mcJ$. If $\mcJ$ is not ancestrally closed, i.e., $\Pa(\mcJ) \not\subseteq \mcJ$, then the marginal distribution of $Z_{\mcJ}$ is not a causally sufficient model. Conversely, if $\mcJ$ is an ancestrally closed set, then the marginal of $Z_{\mcJ}$ comes from a causally sufficient model with graph $\mcG_{\mcJ}$. Subsequently, using Stage~L2 of Algorithm~1 for score differences of $\hat \bZ_{\mcJ}$, we obtain the transitive closure of $\mcG_{\mcJ}$. Finally, performing the unmixing procedure over the ancestrally closed set of \emph{hard} interventions, Stage~L3 of Algorithm~1 ensures the perfect recovery of $\mcG_{\mcJ}$.
    \end{enumerate}
\end{remark}

\subsection{Identifiability up to Surrounding Parents}\label{sec:linear-full-rank}

For linear transformations, we finally investigate the conditions under which soft interventions are guaranteed to achieve identifiability results stronger than transitive closure and mixing up to parents. In particular, we specify one condition on the rank of the score function differences $\{\bsd_{\bZ}^{m}: m \in [n]\}$, which is formalized next.
\begin{assumption}[Full-rank Score Difference]\label{assumption:full-rank}
    For all interventional environments $\mcE^m \in \mcE$ we have
    \begin{equation}
        \rank(\bR_{\bZ}^{m}) = |\Paplus(I^m)| \ .
    \end{equation}
\end{assumption}
For insight into this assumption, it can be readily verified that for linear Gaussian latent models $\rank(\bR_{\bZ}^{m}) \leq 2$ and on the other hand, for sufficiently nonlinear causal models, $\rank(\bR_{\bZ}^{m})$ is $\Paplus(I^m)$. This assumption is stronger than \Cref{assumption:rank-two} since it implies that the effects of an intervention on \emph{all} parents of the target variable are different. We will provide more discussions on this assumption in \Cref{sec:discuss-assumptions}. Under this condition, \Cref{alg:linear-full-rank} identifies the latent graph and the latent variables with the following key intuitions.

\paragraph{Latent graph estimation.}
Recall that $\bR_{\bX}^{m} = [\bGinv]^{\top} \cdot \bR_{\bZ}^{m} \cdot \bGinv$. Since $\bGinv$ has full row-rank $n$, \Cref{assumption:full-rank} implies that $\rank(\bR_{\bX}^{m}) = |\Paplus(I^m)|$ for all $m\in[n]$, which further implies that \Cref{lm:dx-colspace} becomes an equality as well, i.e.,
\begin{equation}
    \colspace(\bR_{\bX}^{m}) = \spann\big\{[\bGinv_i]^{\top} \,:\, i \in \Paplus(I^m)\big\} \ .
\end{equation}
Then, for any $t, k \in [n]$, we have
\begin{align}
    \colspace(\bR_{\bX}^{t}) \cap \colspace(\bR_{\bX}^{k}) &= \spann\big\{[\bGinv_i]^{\top} \,:\, i \in \Paplus(I^t) \cap \Paplus(I^k) \big\} \ , \label{eq:dx-col-intersection} \\
    \dim\big( \colspace(\bR_{\bX}^{t}) \cap \colspace(\bR_{\bX}^{k})\big) &= |\Paplus(I^t) \cap \Paplus(I^k)| \\
    &= |\Pa(I^t) \cap \Pa(I^k)| + \mathds{1}\big\{\{I^t \in \Pa(I^k)\} \lor \{I^k \in \Pa(I^t) \} \big\}
    \ . \label{eq:dx-col-intersection-dim}
\end{align}
In \Cref{alg:linear-full-rank}, we first obtain a topological order of the true DAG using the graph estimate from \Cref{alg:linear} and initialize the new estimate $\hat \mcG$ with the empty graph. Consider a pair $(t,k)$ such that $I^t \in \Pa(I^k)$. If $\hat \Pa(k)$ does not contain node $t$ in $\hat \mcG$, then by \eqref{eq:dx-col-intersection-dim} we have $ \dim\big( \colspace(\bR_{\bX}^{t}) \cap \colspace(\bR_{\bX}^{k})\big) > |\hat \Pa(t) \cap \hat \Pa(k)|$. We use this observation to sequentially identify the parents of each node, starting from the root node(s) and gradually advancing to the leaf node(s). In the end, resulting $\hat \mcG$ equals $\mcG$ under permutation $\mcI$. 

\begin{algorithm}[t]
    \caption{LSCALE-I for sufficiently nonlinear latent causal models}
    \label{alg:linear-full-rank}
    \begin{algorithmic}[1]
    \State \textbf{Input:} $\bR_{\bX}^{m}$ for all $m \in [n]$, topological order $\pi$ from \Cref{alg:linear} 
    \State \textbf{\underline{Latent graph estimation}}
    \State Initialize $\hat \mcG$ with empty graph
    \For{$k \in (\pi_1,\dots,\pi_n)$}
        \For{$t \in (\pi_1,\dots,\pi_{k-1})$}
            \If{$\dim\big(\colspace(\bR_{\bX}^{t}) \cap \colspace(\bR_{\bX}^{k})\big) > |\hat \Pa(t) \cap \hat \Pa(k)|$}
                \State Update $\hat \Pa(k) \gets \hat \Pa(k) \cup \{t\}$
            \EndIf
        \EndFor
    \EndFor
    \State \textbf{\underline{Encoder estimation}}
    \State{Initialize $\hat \bH \gets \bzero_{n \times d}$}
    \For{$m \in (1, \dots, n)$}
        \State Select any $\bh \in \!\!\! \textstyle\bigcap\limits_{k \, \in \, \hat \Ch(m) \, \cup \, m} \!\!\! \colspace(\bR_{\bX}^{k})$ 
        \State $\hat \bH_{m} \gets \bh^{\top}$
    \EndFor
    \State \textbf{Return} $\hat \mcG$, $\hat \bH$, and $\hat \bZ$
    \end{algorithmic}
\end{algorithm}%

\paragraph{Encoder estimation.} In \Cref{alg:linear-full-rank}, we refine the encoder estimation step of \Cref{alg:linear} as follows. Extending \eqref{eq:dx-col-intersection} and using the graph recovery result, we have 
\begin{align}
    \textstyle\bigcap\limits_{k \, \in \, \hat \Ch(m) \, \cup \, m } \colspace(\bR_{\bX}^{k})&= \spann\big\{[\bGinv_i]^{\top} \,:\, i \in \textstyle\bigcap\limits_{k \, : \, I^k \in  \Chplus(I^m)} \Paplus(I^k) \big\} \ . \label{eq:dx-col-intersection-full-rank}
\end{align}
Now, consider $I^t \notin \sur(I^m)$. By the definition of surrounding parents, there exists $I^j \in \Chplus(I^m)$ such that $I^j \notin \Ch(I^t)$. Then, $I^j$ is not in the intersection of the parent sets in the right-hand-side of \eqref{eq:dx-col-intersection-full-rank}, which means $[\bGinv_j]^{\top}$ is not in the column space of the left-hand-side of \eqref{eq:dx-col-intersection-full-rank}. Therefore, by choosing $\hat \bH_m$ from this column space, we obtain for all $m \in [n]$,
\begin{align}
    [\hat \bZ(\bX; \hat \bH)]_m = \hat \bH_{m} \cdot \bX &= c_{\ell} \cdot Z_{\ell} + \sum_{i \in \sur(\ell)} c_i \cdot Z_i \ , \qquad \mbox{where} \; \ell = I^m \ , \\
    \hat \bZ(\bX; \hat \bH) &= \bP_{\mcI} \cdot \bC_{\rm pa} \cdot \bZ \ ,
\end{align}
where $\bP_{\mcI}$ is the permutation matrix of the intervention order $(I^1,\dots,I^n)$, $\bC_{\rm sur}$ has nonzero diagonal entries and satisfies $\big[\bC_{\rm sur}\big]_{i,j} = 0$ for all $j \notin \overline{\sur}(i)$. \endproof

\begin{theorem}[Linear -- Soft Interventions for Nonlinear SCMs]\label{th:linear-soft-full-rank}
Under \Cref{assumption:full-rank} for linear transformations, using observational data and interventional data from one \emph{soft} intervention per node suffice to identify (i) the latent DAG $\mcG$ perfectly, and (ii) the latent variables $\bZ$ with consistency up to mixing with surrounding parents, and recovered latent variables satisfy Markov property with respect to $\mcG$. Specifically, \Cref{alg:linear-full-rank} achieves these identifiability guarantees.
\end{theorem}
\begin{proof}
    Proof of the consistency of latent variables up to mixing with surrounding parents is given above in complete detail. For the proof of perfect graph recovery and $\hat \bZ$ satisfies Markov property with respect to $\hat \mcG$, see \Cref{proof:linear-soft-full-rank}.     
\end{proof}

\Cref{th:linear-soft-full-rank} has two important implications. First, the latent DAG can be identified using \emph{only} soft interventions under mild nonlinearity assumptions on the latent causal model. To our knowledge, this is the first result in the literature for fully recovering latent DAG with soft interventions without restricting the graphical structure, e.g., \citet{zhang2023identifiability} require \emph{linear faithfulness} assumption to achieve similar results, which is only shown to hold for nonlinear latent models with polytree structure. Secondly, the estimated latent variables reveal the true conditional independence relationships since they satisfy the Markov property with respect to the estimated latent graph, which is isomorphic to the latent graph. Recalling that the motivation of CRL is to learn useful representations that preserve causal relationships, our result shows that this can be achieved without perfect identifiability for a large class of models. Furthermore, \citet[Theorem 3]{jin2024learning} establish that under some nondegeneracy assumptions, mixing consistency up to surrounding parents is the best possible result when using single-node soft interventions. Hence, the results in \Cref{th:linear-soft-full-rank} are tight for the considered setting. \looseness=-1 

\subsection{Discussion on Assumptions~\ref{assumption:rank-two} and \ref{assumption:full-rank}}\label{sec:discuss-assumptions}

In this subsection, we elaborate on \Cref{assumption:rank-two} and \Cref{assumption:full-rank}, which are relevant to \Cref{th:linear-soft} and \Cref{th:linear-soft-full-rank}, respectively. \Cref{assumption:rank-two} essentially states that score changes in the coordinates of the intervened node and a parent of the intervened node are linearly independent. This property holds for (but is not limited to) the widely adopted additive noise models specified in \eqref{eq:additive-SCM} when we apply hard interventions.
\begin{lemma}[Proof in \Cref{sec:discuss-assumption-rank-two}]\label{lm:rank-two-additive-hard}
    \Cref{assumption:rank-two} is satisfied for additive noise models under hard interventions.
\end{lemma}

When soft interventions are applied, the transitive closure of $\mcG$ cannot be identified without making assumptions about the effect of the interventions. Specifically, for linear latent causal models, \citet{buchholz2023learning} prove impossibility results for pure shift interventions and \citet{squires2023linear} show that a \emph{genericity condition} is necessary for identifying the transitive closure of the latent DAG. Therefore, \Cref{assumption:rank-two} can be interpreted as the counterpart of the commonly adopted assumptions in the literature on soft interventions adapted to the setting of general latent causal models. Finally, the next example demonstrates the working of \Cref{assumption:rank-two} on a linear Gaussian latent model.
\begin{example}\label{example:gaussian-assumption-rank-two}
    Consider a linear Gaussian latent model with $\bZ \sim \mcN(0, \Sigma)$. Score function of $\bZ$ is given by $\bss(\bz) = -\Sigma^{-1} \cdot \bz$. Let $Z_i = \bw \cdot \bZ_{\Pa(i)} + N_i$ where $N_i \sim \mcN(0,\sigma_i^2)$ for the node $i$. Consider an intervention on node $i$ on environment $\mcE^m$ such that $Z^m_i = \bar \bw \cdot \bZ^m_{\Pa(i)} + \bar N_i$, where $Z_m \sim \mcN(0, \bar \Sigma)$, and $\bar N_i \sim \mcN(0, \bar \sigma_i^2)$, which yields $\bsd_{\bZ}^{m}(\bz) = - (\Sigma^{-1} - \bar \Sigma^{-1}) \cdot \bz$. Then, for a node $k \in \Pa(i)$, we obtain 
    \begin{align}
        [\bsd_{\bZ}^{m}(\bz)]_i 
        &= \left(\frac{1}{\sigma_i^2} - \frac{1}{\bar \sigma_i^2}\right) z_i - \left(\frac{\bw}{\sigma_i^2} - \frac{\bar \bw}{\bar \sigma_i^2}\right) \bz_{\Pa(i)} \ , \label{eq:example1-1} \\
        [\bsd_{\bZ}^{m}(\bz)]_k 
        &= - \left(\frac{w_k}{\sigma_i^2} - \frac{\bar w_k}{\bar \sigma_i^2}\right) z_i + \left(w_k \frac{\bw}{\sigma_i^2} - \bar w_k \frac{\bar \bw}{\bar \sigma_i^2}\right) \bz_{\Pa(i)} \ , \label{eq:example1-2}          
    \end{align}
    where $w_k$ and $\bar w_k$ correspond to weights of the parent $Z_k$ in observational and interventional models, respectively. Note that, for \Cref{assumption:rank-two} to be violated, there must exists a constant $\kappa \in \R$ such that 
    \begin{equation}
        \kappa \cdot \big[\bsd_{\bZ}^{m}(\bz)\big]_i = \big[\bsd_{\bZ}^{m}(\bz)\big]_k  \ , \quad \forall \bz \in \R^n \ .
    \end{equation}
    However, using \eqref{eq:example1-1} and \eqref{eq:example1-2}, this is possible if and only if $w_k = \bar w_k$. Therefore, if the weight of the node $k \in \Pa(i)$ changes, \Cref{assumption:rank-two} is satisfied for the node pair $(i,k)$.
\end{example}

Given the known result that perfect identifiability is impossible for linear Gaussian models given soft interventions~\citep{squires2023linear}, the purpose of \Cref{assumption:full-rank} is to get more insight into the extent of identifiability guarantees under soft interventions. Intuitively, the mentioned impossibility results for linear Gaussian models are due to the rank deficiency of score differences for linear models -- specifically, we know that $\rank(\bR_{\bZ}^{m}) \leq 2$ for linear Gaussian models. In contrast, for sufficiently nonlinear causal models, $\rank(\bR_{\bZ}^{m})$ can be as high as $\Paplus(I^m)$. \Cref{assumption:full-rank} ensures that this upper bound is satisfied with equality for all nodes. This condition holds for the class of sufficiently nonlinear models, such as quadratic causal models. In particular, we show that this condition holds for the two-layer neural networks (NNs) as a function class that can effectively approximate any continuous function. This result is formalized in the next lemma. 
\begin{lemma}\label{prop:nn}
    Consider the additive model in \eqref{eq:additive-SCM} where $f_i$ is a two-layer NN with sigmoid activation function, and weight matrices $\bW^i$ and $\bar \bW^i$ for observational and interventional mechanisms, respectively. If $\max\{\rank(\bW^i)\, , \, \rank(\bar \bW^i)\} = |\Pa(i)|$ for all $i \in [n]$, then \Cref{assumption:full-rank} holds.
\end{lemma}
We discuss the nonlinearity and the proof of \Cref{prop:nn} in \Cref{appendix:analyze-assumption-full-rank}.

\section{CRL under General Transformations}\label{sec:general}

In this setting, we consider general transformations without any parametric assumption for the transformation $g$. In the previous section, we exploited the transformation's linearity and recovered the true encoder's parameters sequentially using \emph{one} intervention per node. For general transformations (parametric or nonparametric), however, we cannot use the same parametric approach and rely on the properties of linear transforms. To rectify these and design the general CRL algorithm, we use more information in the form of \emph{two} interventions per node. Specifically, we will present the steps of leveraging the score functions under two interventions and design an algorithm that identifies the true encoder and recovers the true causal representations. The algorithm is referred to as {\bf G}eneral {\bf S}core-based {\bf Ca}usal {\bf L}atent {\bf E}stimation via {\bf I}nterventions (GSCALE-I), which will be summarized in \Cref{alg:general}. We also present additional results and discuss the role of various inputs and the distinction of our results compared to the existing literature.

\paragraph{Inputs.} The inputs of GSCALE-I are the observed data from the observational environment, the data from two interventional environments per node, whether environments are coupled/uncoupled, and a set of valid encoders $\mcH$. Two sets of interventional environments are denoted by $\mcE = \{\mcE^1,\dots,\mcE^n\}$ and $\tilde\mcE = \{\tilde \mcE^1,\dots, \tilde \mcE^n\}$, as defined in \Cref{sec:intervention-mechanisms}. For these inputs, we compute the score functions $\bss_{\bX}$, $\{\bss_{\bX}^1,\dots,\bss_{\bX}^n\}$ and $\{\tilde \bss_{\bX}^1, \dots, \tilde \bss_{\bX}^n\}$. Note that we will use \Cref{lm:score-difference-transform-general} again to ensure access to the latent score differences by using these observed score functions.

\paragraph{Statistical diversity.}
For being more informative than a single intervention mechanism, we assume that the two intervention mechanisms per node are sufficiently distinct. This is formalized by defining \emph{interventional discrepancy}~\citep{liang2023causal} among the causal mechanisms of a latent variable. 
\begin{definition}[Interventional Discrepancy]\label{assumption:int-discrepancy}
Two intervention mechanisms with pdfs $p,q:\R \to \R$ are said to satisfy \emph{interventional discrepancy} if 
    \begin{equation}\label{eq:int-discrepancy}
        \frac{\partial}{\partial u} \frac{p(u)}{q(u)} \neq 0 \ , \quad {\forall u\in \R \setminus \mcT} \ , 
    \end{equation}
    where $\mcT$ is a null set (i.e., has a zero Lebesgue measure). 
\end{definition}
This condition ensures that the two distributions are sufficiently distinct, formally expressed as the partial derivative of their ratio with respect to the intervened variable is nonzero almost everywhere. For instance, two univariate, distinct Gaussians trivially satisfy this discrepancy condition. As shown by~\citet{liang2023causal}, even when the latent graph $\mcG$ is known, for identifiability via one intervention per node, it is necessary to have interventional discrepancy between observational distribution $p_i$ and interventional distribution $q_i$, for all $\bz_{\Pa(i)}\in\R^{|\Pa(i)|}$. \looseness=-1

\subsection{Rationale of GSCALE-I and Partial Identifiability}\label{sec:general-rationale}

Similarly to LSCALE-I, analyzing GSCALE-I involves leveraging score differences. However, in contrast to linear transformations, estimated and true latent score differences are not always related by a constant matrix for general transformations. To circumvent this issue, we rely on \Cref{lm:parent_change_comprehensive}(iii), i.e., the score difference of \emph{coupled} interventional environments is one-sparse. As such, the key idea is to find an encoder $h \in \mcH$ that adheres to this sparsity structure, and we will show that such an encoder ensures component-wise identifiability of the latent variables.
To formalize these, corresponding to each valid encoder $h \in \mcH$, we define the score change matrix $\bD_{\rm t}(h)$ with entries:
\begin{align}
    [\bD_{\rm t}(h)]_{i,m} &\triangleq  \E\Big[\, \big|\bss_{\hat \bZ}^{m}(\hat \bZ;h) - \tilde \bss_{\hat \bZ}^{m}(\hat \bZ;h)\big|_i \Big] \ , \label{eq:Dh-entry} 
\end{align}
where expectations are under the measures of latent score functions induced by the probability measure of observational data. We also denote the \emph{true} score change matrix under true encoder $g^{-1}$ by $\bD_{\rm t} \triangleq \bD_{\rm t}(g^{-1})$ with entries
\begin{equation}\label{eq:true-score-change-matrix}
    [\bD_{\rm t}]_{i,m} \triangleq \E\Big[\big|\bss^{m}(\bZ)- \tilde \bss^{m}(\bZ)\big|_i\Big] \ , \quad \forall i,m \in [n] \ .
\end{equation} 
For clarity in the exposition of the ideas, we consider coupled environments here, i.e., $I^m = \tilde I^m$ for all $m\in[n]$. Then, using \eqref{eq:scorechanges_iff4} in \Cref{lm:parent_change_comprehensive} we have 
\begin{equation}
    [\bD_{\rm t}]_{i,m} \neq 0 \;\; \iff \;\; i = I^m \ .
\end{equation}
This implies that $\mathds{1}\{\bD_{\rm t}\}$ is a permutation matrix, $\mathds{1}\{\bD_{\rm t}\} = \bP_{\mcI}^{\top}$ and
\begin{equation}\label{eq:column-Dt-coupled}
    \mathds{1}\{[{\bD_{\rm t}}]_{:,m}\} = \be_{\ell} \ , \quad \mbox{where} \; \ell = I^m \ ,
\end{equation}
and $\be_{\ell}$ denotes the $\ell$-th standard basis vector.
Subsequently, the key idea for identifying the encoder is that the changes between score estimates $\bss_{\hat \bZ}^{m}(\hat \bz;h)$ and $\tilde \bss_{\hat \bZ}^{m}(\hat \bz;h)$ should match the sparsity structure for those under the true encoder $g^{-1}$. In the next result, we leverage the sparsity structure in \eqref{eq:column-Dt-coupled} for two interventions on the same node to identify an intervened variable, formalized as follows.

\begin{theorem}[General -- Node-level Identifiability]\label{thm:general-partial-identifiability}
Two hard interventions (with interventional discrepancy) on the same target node $I^m=\tilde I^m=\ell$ suffice to recover $Z_\ell$ up to a diffeomorphism. In particular, the following encoder
\begin{equation}\label{eq:gscalei-partial-optimization}
    \hat h_m = \arg\min\limits_{h \in \mcH} \; \big\| [\bD_{\rm t}(h)]_{:,m} - \be_m \big\|^2 
\end{equation}
satisfies $\big[\hat \bZ(\bX; \hat h_m)\big]_m = \phi_{\ell}(Z_\ell)$ for a diffeomorphism $\phi_{\ell}: \R \to \R$.
\end{theorem}
We will show shortly that given a complete set of coupled interventions, \Cref{thm:general-partial-identifiability} will readily imply the identifiability of all variables (in a similar spirit to going from partial to complete identifiability in \Cref{sec:linear}). Furthermore, the results for uncoupled interventions will also leverage this key result. Hence, we provide the detailed proof of \Cref{thm:general-partial-identifiability} as follows.

\begin{proof}
    We start by providing a direct result of \Cref{corollary:score-difference-transform-linear}, which will be used in multiple instances in this section.
    Consider an encoder $h\in\mcH$ and an invertible matrix $\bA \in {\sf GL}(n, \R)$. Then, using \Cref{corollary:score-difference-transform-linear}, score functions of $\hat \bZ(\bX;h)$ and $\hat \bZ(\bX; \bA \cdot h)$ are related by
    \begin{equation}\label{eq:score-multiply-encoder}
        \bss_{\hat \bZ}(\hat \bZ; \bA \cdot h) = \bA^{-\top} \cdot \bss_{\hat \bZ}(\hat \bZ; h) \ .
    \end{equation}
    Recall that score change matrices are defined via the absolute value of the score differences. Therefore, \eqref{eq:score-multiply-encoder} implies that for a positively scaled permutation matrix $\bA$ we have
    \begin{equation}\label{eq:Dt-multiply-encoder}
        \bD_{\rm t}(\bA \cdot h) = \bA^{-\top} \cdot \bD_{\rm t}(h) \ , \quad \forall h \in \mcH \ .
    \end{equation}
    Denote the interventional environments by $\mcE^m$ and $\tilde \mcE^m$ with $I^m = \tilde I^m = \ell$. We show that the minimum value of \eqref{eq:gscalei-partial-optimization} is zero as follows. Let $h^* = \bD_{\rm t}^{\top} \cdot g^{-1}$. Recall that $\mathds{1}\{\bD_{\rm t}\} = \bP_{\mcI}^{\top}$, and by definition, $\bD_{\rm t}$ has only nonzero entries. Thus, using \eqref{eq:Dt-multiply-encoder}, we have $\bD_{\rm t}(\rm h^*) = \bD_{\rm t}^{-1} \cdot \bD_{\rm t} = \bI_{n \times n}$, which makes the objective in \eqref{eq:gscalei-partial-optimization} zero. Next, we show that \emph{any} solution of \eqref{eq:gscalei-partial-optimization} recovers $\hat Z_{\ell}$ up to a diffeomorphism. 
    
    Consider encoder $\hat h_m \in \mcH$ that satisfies $[\bD_{\rm t}(\hat h_m)]_{:,m} = \be_m$. Let $\hat \bZ$ denote $\hat \bZ(\bX; \hat h_m)$ and $f \triangleq \hat h_m \circ g$, so we have $\hat \bZ = f(\bZ)$ and $\bZ = f^{-1} (\hat \bZ)$. Using score difference transformation property in \eqref{eq:score-difference-z-zhat} and one-sparse property of $[\bss^{m}(\bZ) - \tilde \bss^{m}(\bZ)]$ via \Cref{lm:parent_change_comprehensive}(iii), we have
    \begin{align}
        [\bD_{\rm t}(\hat h_m)]_{i,m} &= \E \Big[\big| \bss_{\hat \bZ}^{m}(\hat \bZ) - \tilde \bss_{\hat \bZ}^{m}(\hat \bZ) \big|_i \Big] \\
        &= \E \Big[ \big| \big[J_{f}^{-\top}(\bZ)\big]_i \cdot \big[\bss^{m}(\bZ) - \tilde \bss^{m}(\bZ)\big] \big| \Big] \\ 
        &= \E \Big[ \big[J_{f}^{-1}(\bZ)\big]_{\ell,i} \cdot \big[\bss^{m}(\bZ) - \tilde \bss^{m}(\bZ)\big]_{\ell} \Big] \ .
    \end{align}
    By definition of $\hat h_m$, $[\bD_{\rm t}(\hat h_m)]_{i,m} = 0$ for all $i \neq m$. Also, interventional discrepancy between $q_\ell(z_l)$ and $\tilde q_\ell(z_l)$ implies that $\big[\bss^{m}(\bz) - \tilde \bss^{m}(\bz)\big]_{\ell} \neq 0$ except for a null set. Hence, if $\big[J_{f}^{-1}(\bz)\big]_{\ell,i}$ is nonzero over a nonzero-measure set within $\R^n$, then $[\bD_{\rm t}(\hat h_m)]_{i,m}$ would not be zero. Therefore, we have $\big[J_{f}^{-1}(\bz)\big]_{\ell,i} = 0$ except for a null set. Since $J_{f}^{-1}$ is a continuous function, this implies that 
    \begin{equation}\label{eq:jacobian-partial-id}
        \big[J_{f}^{-1}(\bz)\big]_{\ell,i}  = 0 \ , \quad \forall i \in [n] \setminus m, \; \forall \bz \in \R^n,    
    \end{equation}
    Furthermore, since $J_{f}^{-1}(\bz)$ must be invertible, we have $[J_{f}^{-1}(\bz)]_{\ell,m} \neq 0$ for all $\bz \in \R^n$. Consider $Z_\ell = [f^{-1}(\hat \bZ)]_{\ell}$. Since $J_{f}^{-1}(\bz) = J_{f^{-1}}(\hat \bz)$, \eqref{eq:jacobian-partial-id} implies that $Z_{\ell} = \psi(\hat Z_{m})$ for some diffeomorphism $\psi : \R \to \R$. Thus, $\hat Z_m = \psi^{-1}(\hat Z_{\ell})$ where $\ell=I^m$, which concludes the proof.    
\end{proof}

\subsection{Complete Identifiability under Coupled Interventions}\label{sec:general-coupled}

After achieving the node-level latent variable identifiability from one pair of coupled interventions, we extend the objective function in \Cref{thm:general-partial-identifiability} for a \emph{complete} set of coupled interventions.

\paragraph{Encoder estimation.} In Stage~G1 of \Cref{alg:general}, we combine the objectives in \eqref{eq:gscalei-partial-optimization} for all $m\in[n]$ and form our encoder estimate by solving the following optimization problem 
\begin{equation}\label{eq:general-OPT1-procedure}
    \min\limits_{h \in \mcH} \big\|\bD_{\rm t}(h) - \bI_{n \times n}\big\|_{\rm F}^{2} \ .
\end{equation}
Recall that the valid encoders set $\mcH$ consists of invertible encoders over the input space $\mcX$. Hence, \eqref{eq:general-OPT1-procedure} is readily equivalent to solving the following problem for any $\lambda \in \R_+$:
\begin{equation}\label{eq:general-OPT1}
    \min\limits_{h} \big\|\bD_{\rm t}(h) - \bI_{n \times n}\big\|_{\rm F}^{2} \; + \; \lambda \; \E\Big[ \big\| h^{-1}(h(\bX)) - \bX \big\|_2^2\Big] \ .
\end{equation}
\begin{lemma}\label{lm:general-encoder-coupled}
    Given a complete set of coupled interventions that satisfy interventional discrepancy, output $\hat h$ of \Cref{alg:general} satisfies component-wise latent recovery. 
\end{lemma}
\begin{proof}
    Following \Cref{thm:general-partial-identifiability} for all $m\in[n]$ immediately implies the desired result. First, recall that $h^* = \bD_{\rm t} \cdot g^{-1}$ satisfies $\bD_{\rm t}(h^*)=\bI_{n \times n}$ and belongs to $\mcH$. Therefore, $h^*$ makes the objective \eqref{eq:general-OPT1} zero. Next, consider a solution $\hat h$ that makes \eqref{eq:general-OPT1} zero. Since $\hat h$ makes the reconstruction loss term zero, i.e., it has a valid inverse $\hat h^{-1}$, $\hat h$ belongs to the set of valid encoders $\mcH$. Then, for all $m\in[n]$, we have $[\bD_{\rm t}(\hat h)]_{:,m} = \be_m$. By \Cref{thm:general-partial-identifiability}, this implies 
    \begin{equation}
        \hat \bZ(\bX; \hat h) = \bP_{\mcI} \cdot \phi(\bZ) \ ,
    \end{equation} 
    where $\phi(\bZ) = (\phi_1(Z_1),\dots,\phi_n(Z_n))$ is a componentwise diffeomorphism, which concludes the proof.
\end{proof}

\paragraph{Latent graph estimation.}
Next, we use the encoder estimate $\hat h$ for recovering the latent graph $\mcG$. In Stage~G2 of \Cref{alg:general}, we use the score differences between observational and interventional environments and form $\hat \mcG$ with parent sets 
\begin{equation}\label{eq:general-coupled-parent-set}
    \hat \Pa(m) \triangleq \Big\{i \neq m: \E\Big[\big| \bss_{\hat \bZ}(\hat \bZ; \hat h) - \bss_{\hat \bZ}^{m}(\hat \bZ; \hat h) \big|_i \Big] \neq 0 \Big\} \ , \quad \forall m \in [n] \ .
\end{equation}
\begin{lemma}\label{lm:general-graph-identifiability}
    Given a complete set of coupled interventions that satisfy interventional discrepancy, \Cref{alg:general} output $\hat \mcG$ and true latent DAG $\mcG$ are related through a graph isomorphism.
\end{lemma}
\begin{proof}
    Denote $f = \hat h \circ g$, so we have $\hat \bZ = f(\bZ)$. Using the score transform between $\bZ$ and $\hat \bZ$ and the structure of $J_f^{-1}$ given in \eqref{eq:jacobian-partial-id}, we have 
    \begin{equation}
        \big| \bss_{\hat \bZ}(\hat \bZ; \hat h) - \bss_{\hat \bZ}^{m}(\hat \bZ; \hat h) \big|_i = \Big| \big[J_f^{-1}(\bZ)\big]_{i, I^i} \cdot \big[\bss(\bZ) - \bss^{m}(\bZ) \big]_i \Big| 
    \end{equation} 
    Since $[J_f^{-1}(\bZ)]_{i, I^i} \neq 0$ for all $\bz \in \R^n$, using \Cref{lm:parent_change_comprehensive}(i), we find
    \begin{equation}
        i \in \hat \Pa(m) \;\; \iff \;\; \E\Big[\big[\bss(\bZ) - \bss^{m}(\bZ) \big]_{I^i} \Big] \neq 0  \;\; \iff \;\; I^i \in \Pa(I^m)  ,
    \end{equation}
    which concludes the proof that $\mcG$ and $\hat \mcG$ are related through a graph isomorphism.
\end{proof}

\begin{algorithm}[t]
    \caption{{\bf G}eneralized {\bf S}core-based {\bf Ca}usal {\bf L}atent {\bf E}stimation via {\bf I}nterventions (GSCALE-I)}
    \label{alg:general}
    \begin{algorithmic}[1]
    \State \textbf{Input: } $\mcH$, samples of $\bX$ from environment $\mcE^0$ and coupled environments $\{(\mcE^m,\tilde \mcE^m) : m \in [n]\}$
    \State Compute score functions: $\bss_{\bX}, \bss_{\bX}^{m}$, and $\tilde \bss_{\bX}^{m}$ for all $m\in[n]$.
    \vspace{-0.10in}
    \Statex \hrulefill
    \State \textbf{Stage~G1: Encoder estimation}:  
        \[ \hat h \gets \arg\min\limits_{h} \big\|\bD_{\rm t}(h) - \bI_{n \times n}\big\|_{\rm F}^{2} \, + \, \lambda \, \E\Big[ \big\| h^{-1}(h(\bX)) - \bX \big\|_2^2\Big] \] 
    \State Latent variable estimates: $\hat \bZ = \hat h(\bX)$.
    \vspace{-0.10in}
    \Statex \hrulefill
    \State \textbf{Stage~G2: Latent graph estimation} 
    \State Construct latent DAG estimate $\hat \mcG$ with parent sets
    \[ \hat{\Pa}(m) \triangleq \Big\{i \neq m: \E\Big[\big|\bss_{\hat \bZ}^{m}(\hat \bZ) - \tilde \bss_{\hat \bZ}^{m}(\hat \bZ)\big|_i\Big] \neq 0 \Big\} \ , \quad \forall m \in [n] \ .  \]
    \State \Return $\hat \mcG$, $\hat h$, and $\hat \bZ$
    \end{algorithmic}
\end{algorithm}%

\begin{theorem}[General -- Coupled Environments]\label{th:general-coupled}
    Using observational data and interventional data from two \emph{coupled} hard environments for which the pair $(q_i, \tilde q_i)$ satisfies interventional discrepancy for all $i \in [n]$, suffice to identify (i) the latent DAG $\mcG$ perfectly and (ii) the latent variables $\bZ$ up to componentwise diffeomorphisms specified in \Cref{def:latent-variable-identifiability}. Specifically, \Cref{alg:general} achieves these identifiability guarantees.
\end{theorem}
\begin{proof}
    Combining \Cref{lm:general-encoder-coupled} and \Cref{lm:general-graph-identifiability} gives the theorem statements.
\end{proof}

We emphasize that a key contribution of \Cref{th:general-coupled} beyond the identifiability result is that it provides a well-defined objective function, a minimizer of which provably recovers the latent variables and the graph up to identifiability guarantees. To our knowledge, this is the first such result for interventional CRL with general transformations.
In comparison, \citet{vonkugelgen2023twohard} also use two hard interventions per node to prove identifiability, albeit without providing a provably correct algorithm. Furthermore, the result in~\citep{vonkugelgen2023twohard} requires that the estimated latent distribution is faithful to the associated candidate graph for all $h \in \mcH$. Although a faithfulness assumption does not compromise the identifiability result, it is a strong requirement to verify and poses challenges to devise recovery algorithms. In contrast, we only require observational data, which is generally accessible in practice.

Next, we shed light on the role of observational data. \Cref{lm:general-encoder-coupled} requires only interventional data for identifying the encoder, whereas \Cref{lm:general-graph-identifiability} uses observational data. We further tighten this result by showing that for DAG recovery, the observational data becomes unnecessary when we have additive noise models and a weak faithfulness condition holds.

\begin{theorem}[No Observational Data]\label{th:general-coupled-no-observational}
Using interventional data from two \emph{coupled} hard environments for which the pair $(q_i,\tilde q_i)$ satisfies interventional discrepancy for all $i \in [n]$, suffices to identify the latent variables $\bZ$ up to component-wise diffeomorphisms. If the latent causal model has additive noise, $p(\bZ)$ is twice differentiable and satisfies the adjacency-faithfulness with respect to $\mcG$, then the latent DAG $\mcG$ is also identifiable. 
\end{theorem}
\textit{Proof sketch:} The recovery of latent variables follow from \Cref{lm:general-encoder-coupled}. For the recovery of latent DAG, we leverage \Cref{lm:parent_change_comprehensive}(iv) under the additive noise assumption. Specifically, using the score difference between environments $\mcE^i$ and $\tilde \mcE^j$, we obtain $\overline{\Pa}(i,j)$ for all $i, j \in [n], i\neq j$. Finally, we perform at most $n$ conditional independence tests to identify all parent sets $\{\Pa(i) : i \in [n]\}$. See \Cref{proof:general-coupled-no-observational} for the complete proof.

\subsection{Complete Identifiability under Uncoupled Interventions}\label{sec:general-uncoupled}

In this setting, we consider the case where the interventional environments corresponding to the same nodes are not specified in pairs. That is, not only is it unknown which node is intervened in an environment, but the learner also does not know which two environments intervene on the same node. Hence, additionally, we need to determine the correct \emph{coupling} between the interventional environment sets $\mcE$ and $\tilde \mcE$ as well. However, this is not a straightforward objective since we cannot readily determine the intervention target in an environment. To address this issue, we solve the problems of finding the correct coupling and optimizing the estimated score variations together. Let $\sigma$ denote the permutation that takes $\tilde \mcI = (\tilde I^1,\dots,\tilde I^n)$ to $\mcI = (I^1,\dots,I^n)$, i.e., $\sigma(\tilde I^m) = I^m$ for all $m\in[n]$. We will modify the objective function in \eqref{eq:general-OPT1} to solve for permutation $\pi$ as an estimate of $\sigma$, in addition to solving for encoder $h$. To this end, we first modify the score change matrix in \eqref{eq:Dh-entry} and define $\bD_{\rm t}(h;\pi)$ as 
\begin{equation}
    [\bD_{\rm t}(h;\pi)]_{i,m} \triangleq  \E\Big[\, \big|\bss_{\hat \bZ}^{m}(\hat \bZ;h) - \tilde \bss_{\hat \bZ}^{\pi_m}(\hat \bZ;h)\big|_i \Big] \ , \label{eq:Dh-pi-entry} 
\end{equation}
Unlike the case of coupled interventions, we will also leverage the observational data to identify the encoder. As such, we also define the score change matrices for $\mcE$ and $\tilde \mcE$ with respect to $\mcE$ with entries 
\begin{align}
    [\bD(h)]_{i,m} &\triangleq \E\Big[\big|\bss_{\hat \bZ}(\hat \bZ;h) - \bss_{\hat \bZ}^{m}(\hat \bZ;h)\big|_i \Big] \ , \label{eq:Dh-obs-entry} \\
    [\tilde \bD(h)]_{i,m} &\triangleq \E\Big[\big|\bss_{\hat \bZ}(\hat \bZ;h) - \tilde \bss_{\hat \bZ}^{m}(\hat \bZ;h)\big|_i \Big] \ . \label{eq:Dh-obs-tilde-entry}
\end{align}

\paragraph{Encoder and coupling estimation.} We modify the objective in \eqref{eq:general-OPT1} and solve the following constrained optimization problem for any $\lambda \in \R_+$:
\begin{equation}\label{eq:general-OPT2}
    \begin{aligned}
        \min\limits_{h, \pi} \big\|\bD_{\rm t}(h;\pi) &- \bI_{n \times n}\big\|_{\rm F}^{2} \; + \; \lambda \; \E\Big[ \big\| h^{-1}(h(\bX)) - \bX \big\|_2^2\Big] \\
        \mbox{such that} \quad & \mathds{1}\{\bD(h)\}=   \mathds{1}\{\tilde \bD(h)\} \cdot \bP_{\pi}^{\top} \\
            & \mathds{1}\{\bD(h)\} \odot \mathds{1}\{[\bD(h)]^{\top}\} = \bI_{n \times n} \ .
    \end{aligned}
\end{equation}
Note that for the true encoder $g^{-1}$, both $\bD(g^{-1})$ and $\tilde \bD(g^{-1})$ capture the graph structure via \Cref{lm:parent_change_comprehensive}(i), albeit under different permutations. Hence, this equivalence of graph structures for the learned encoder is ensured by the first constraint of $\mathds{1}\{\bD(h)\}= \bP_{\pi} \cdot \mathds{1}\{\tilde \bD(h)\}$. Also, the acyclicity of the graphs implied by the score differences is ensured by the second constraint, that is $\bD(h)$ does not contain 2-cycles. To prove that solving this problem leads to the identifiability of the latent variables $\bZ$, we first show that there exists a global minimizer of this problem.

\begin{lemma}[Existence]\label{lm:uncoupled-existence}
    Encoder $h^* = \big[\bD_{\rm t}(g^{-1}; \sigma)\big]^{\top} \cdot g^{-1}$ and permutation $\pi^*= \sigma$ minimize the objective in \eqref{eq:general-OPT2}.
\end{lemma}
\textit{Proof sketch:} See \Cref{proof:uncoupled-existence} for the complete proof. The main intuition is that $\pi^* = \sigma$ makes the problem more similar to the case of coupled interventions. The proof then follows by using the results of the coupled interventions to show that $\bD_{\rm t}(g^{-1}; \sigma)$ is a scaled permutation matrix, using \eqref{eq:score-multiply-encoder} to derive the relationship between the score differences under $\hat h^*$ and the true encoder $g^{-1}$, and using the sparsity structure of $\bD(g^{-1})$ to show that the constraints are satisfied. 

Next, we show that the objective in \eqref{eq:general-OPT2} can attain its minimum value zero \emph{only} for the correct coupling.
\begin{lemma}[Feasibility]\label{lm:uncoupled-feasibility}
  The optimization problem in \eqref{eq:general-OPT2} attains its minimum value zero only for the correct coupling $\pi = \sigma$.   
\end{lemma}
\noindent\textit{Proof sketch:} See \Cref{proof:uncoupled-feasibility} for the complete proof. 
The main intuition is that the constraints make it impossible to achieve $\bD_{\rm t}(h;\pi)=\bI_{n \times n}$ for an incorrect coupling $\pi \neq \sigma$. We prove it by contradiction. We assume that $h^*$ is a solution, hence, $\bD_{\rm t}(h^*) = \bI_{n \times n}$, and $\mathds{1}\{\bD(h)\}=\mathds{1}\{\tilde \bD(h)\}$. Then, by scrutinizing the \emph{eldest} node in $\mcG$ with mismatched interventional environments, we show that $\bD(h^*) \odot [\bD(h^*)]^{\top}$ cannot be a diagonal matrix, which contradicts the premise that $h^*$ is a feasible solution of \eqref{eq:general-OPT1}.

Finally, we denote the minimizer encoder of \eqref{eq:general-OPT2} by $\hat h$ and form the graph estimate similarly to the coupled intervention case via \eqref{eq:general-coupled-parent-set}. Combining these results, we present our strongest result for general transformations. 

\begin{theorem}[General -- Uncoupled Environments]\label{th:general-uncoupled}
    Using observational data and interventional data from two \emph{uncoupled} hard environments for which each pair in $(p_i, q_i, \tilde q_i)$ satisfies interventional discrepancy for all $i \in [n]$, suffice to identify (i) the latent DAG $\mcG$ perfectly and (ii) the latent variables $\bZ$ up to component-wise diffeomorphisms.
\end{theorem}
\begin{proof}
\Cref{lm:uncoupled-feasibility} shows that minimizers of the objective in \eqref{eq:general-OPT2} necessarily have $\pi = \sigma$. Under this correct coupling, any minimizer of the constrained optimization problem in \eqref{eq:general-OPT2} that attains the minimum value zero is also a minimizer of the unconstrained optimization problem in \eqref{eq:general-OPT1}. Finally, \Cref{lm:general-encoder-coupled} and \Cref{lm:uncoupled-existence} show that such a shared minimizer exists, that is $h^* = [\bD_{\rm t}(g^{-1})]^{\top} \cdot g^{-1}$. Hence, by \Cref{th:general-coupled}, minimizer $\hat h$ of \eqref{eq:general-OPT2} satisfies component-wise latent recovery. Similarly, using $\hat h$, proof of the perfect graph recovery follows from \Cref{lm:general-graph-identifiability}.
\end{proof}

\Cref{th:general-uncoupled} shows that using observational data enables us to resolve any mismatch between the uncoupled environment sets and shows identifiability in the setting of uncoupled environments. This generalizes the identifiability result of \citet{vonkugelgen2023twohard}, which requires coupled environments. 

\paragraph{Comparison to \Cref{th:general-coupled}.}
We note that \Cref{th:general-uncoupled} requires slightly stronger interventional discrepancy conditions than \Cref{th:general-coupled}. In particular, when environments are coupled, we only need $(q_i,\tilde q_i)$ to satisfy the interventional discrepancy. On the other hand, to find the correct coupling while performing CRL, \Cref{th:general-uncoupled} requires each of the pairs $\{(q_i,\tilde q_i),(p_i,q_i), (p_i,\tilde q_i)\}$ to satisfy interventional discrepancy.

\begin{remark}\label{remark:achievability}
    For the nonparametric identifiability results, having an oracle that solves the functional optimization problems \eqref{eq:general-OPT1}~and~\eqref{eq:general-OPT2} is sufficient. Solving these two problems in their most general form requires calculus of variations. These two problems, however, for any desired parameterized family of functions $\mcH$ (e.g., linear, polynomial, and neural networks), reduce to parametric optimization problems.
\end{remark}

\subsection{Intervention Extrapolation without CRL}\label{sec:intervention-extrapolation}

So far, we have studied learning the latent causal representations using the data from single-node interventional environments. A related research problem is extrapolating to unseen combinations of interventions. Namely, given a set of interventions $\mcI = \{I^1,\dots,I^k\}$, it is desired to emulate the data from an unseen combination of these interventions, e.g., accessing multi-node interventional data using only the single-node interventional data. This is especially important in domains where interventions can be costly or not viable, e.g., not all combinations of different drugs can be clinically tested. Causal representation learning literature has also taken an interest in this problem. Specifically, \citet[Theorem 3]{zhang2023identifiability} show for polynomial transformations that given interventions $\mcI = \{I^1,\dots,I^k\}$, one can sample from any intervention $I \subseteq \mcI$ \emph{after} learning the latent representations. We argue that extrapolation to unseen combinations of interventions can be achieved on the observed space \emph{without performing CRL} for the general transformations.

Consider two single-node interventional environments, $\mcE^1$ and $\mcE^2$. Without loss of generality, suppose that $I^1 = \{1\}$ and $I^2 = \{2\}$. Also consider the observational environment $\mcE^0$ with $I^0 = \emptyset$ and the \emph{unseen} double-node interventional environment $\mcE^{m}$ with $I^{m} = \{1,2\}$. First, using the score function decompositions in \eqref{eq:s_z_decompose_obs} and \eqref{eq:s_z_decompose_int}, we have
\begin{align}
    \bss(\bz) &= \nabla_{\bz} \log p_1(z_1 \mid \bz_{\Pa(1)}) + \nabla_{\bz} \log p_2(z_2 \mid \bz_{\Pa(2)}) + \sum_{k=3}^{n} \nabla_{\bz}\log p_k(z_k \mid \bz_{\Pa(k)}) \ , \\
    \bss^1(\bz) &= \nabla_{\bz} \log q_1(z_1 \mid \bz_{\Pa(1)}) + \nabla_{\bz} \log p_2(z_2 \mid \bz_{\Pa(2)}) + \sum_{k=3}^{n} \nabla_{\bz}\log p_k(z_k \mid \bz_{\Pa(k)}) \ , \\
    \bss^2(\bz) &= \nabla_{\bz} \log p_1(z_1 \mid \bz_{\Pa(1)}) + \nabla_{\bz}\log q_2(z_2 \mid \bz_{\Pa(2)}) + \sum_{k=3}^{n} \nabla_{\bz}\log p_k(z_k \mid \bz_{\Pa(k)}) \ , \\ 
    \bss^m(\bz) &= \nabla_{\bz} \log q_1(z_1 \mid \bz_{\Pa(1)}) + \nabla_{\bz} \log q_2(z_2 \mid \bz_{\Pa(2)}) + \sum_{k=3}^{n} \nabla_{\bz} \log p_k(z_k \mid \bz_{\Pa(k)}) \ .
\end{align}
Then, we have
\begin{align}\label{eq:dsz-extrapolation}
    \big(\bss^1(\bz) - \bss(\bz)\big) + \big(\bss^2(\bz) - \bss(\bz)\big) = \bss^m(\bz) - \bss(\bz) \ .
\end{align}
Next, recall that \Cref{lm:score-difference-transform-general} gives us the relationship for going from the latent score differences to observed score differences. Applying it to observed $\bX$ in environment pairs $(\mcE^1,\mcE^0)$, $(\mcE^2,\mcE^0)$, and $(\mcE^m,\mcE^0)$, and using \eqref{eq:dsz-extrapolation} we obtain
\begin{align}
    \big(\bss^{1}_{\bX}(\bx) - \bss_{\bX}(\bx)\big) + \big(\bss^{2}_{\bX}(\bx) - \bss_{\bX}(\bx)\big)  &= \Big[[J_g(\bz)]^\dag\Big]^{\top} \cdot \Big[ \big(\bss^1(\bz) - \bss(\bz)\big) + \big(\bss^2(\bz) - \bss(\bz)\big) \Big] \\ 
    &= \Big[[J_g(\bz)]^\dag\Big]^{\top} \big( \bss^m(\bz) - \bss(\bz) \big) \\ 
    \overset{\eqref{eq:score-latent-to-observed}}&{=} \bss^m_{\bX}(\bx) - \bss_{\bX}(\bx) \ . \label{eq:dsx-extrapolation}
\end{align}
This means that, given score functions of observed data from environments $\mcE^0$, $\mcE^1$, and $\mcE^2$, we can obtain the score function of the unseen interventional environment with $I^m = \{I^1, I^2\}$, and subsequently generate data from this synthetic environment (using techniques like Langevin sampling \citep{welling2011bayesian} that uses score function as the drift vector field). Note that this process does not require learning the latent causal representations and can be applied to obtain the observed score function $\bss_{\bX}^m$ for any intervention $I^m \subset \mcI$.

Finally, we note that concurrent work by \citet{jain2024automated} makes a similar observation for the extrapolation of two single-node interventions without performing CRL. In particular, \citep{jain2024automated} focuses on detecting pairwise interactions between two biological perturbations (i.e., interventions), investigates the conditions under which two perturbations are \emph{separable}, e.g., distinct single-node interventions, and designs statistical tests for the separability of two perturbations. If the two perturbations are confirmed to be separable, our results indicate that we can use their score functions to generate synthetic environments that emulate combined perturbation effects.

\section{Empirical Studies}\label{sec:experiments}

In this section, we provide empirical assessments of our theoretical guarantees. Specifically, we empirically evaluate the performance of the LSCALE-I (\Cref{sec:experiments-linear}) and GSCALE-I (\Cref{sec:experiments-general}) algorithms for recovering the latent causal variables $\bZ$ and the latent DAG $\mcG$ on synthetic data. In \Cref{sec:experiments:comparisons}, we compare the performance of LSCALE-I to that of the existing algorithms in the closely related literature on both synthetic and biological data.
Next, we also apply GSCALE-I on image data to demonstrate the potential of our approach in realistic high-dimensional datasets (\Cref{sec:experiments-image}). Finally, we note that any desired score estimator can be modularly incorporated into our algorithms. In \Cref{sec:experiments-sensitivity}, we assess our performance's sensitivity to the estimators' quality.~\footnote{The codebase for the algorithms and simulations is available at: \\ \url{https://github.com/acarturk-e/score-based-crl}.} Additional results and further implementation details are deferred to \Cref{appendix:experiments}.

\paragraph{Evaluation metrics.} The objectives are recovering the graph $\mcG$ and the latent variables $\bZ$. We use the following metrics to evaluate the accuracy of LSCALE-I and GSCALE-I in recovering these (depending on the specifics of the transformations and interventions, we will also have more specific metrics). For each metric, we will report the mean and standard error over multiple runs.
\begin{itemize}[leftmargin=2em]
    \item {\bf Structural Hamming distance:} For assessing the recovery of the latent DAG, we report structural Hamming distance (SHD) between the estimate $\hat \mcG$ and true DAG $\mcG$. This captures the number of edge operations (add, delete, flip) needed to transform $\hat \mcG$ to $\mcG$. \looseness=-1
    
    \item {\bf Mean correlation coefficient:} For the recovery of the latent variables, we use the mean correlation coefficient (MCC), which was introduced in~\citep{khemakhem2020ice} and commonly used as a standard metric in CRL. Specifically, MCC measures the linear correlation between the estimated and ground-truth latent variables. Since the recovery of latent variables is up to permutations, it is reported for the best matching permutation between the components of $\hat \bZ$ and $\bZ$, i.e.,
    \begin{equation}\label{eq:MCC}
        {\rm MCC}(\bZ,\hat \bZ) \triangleq \max_{\pi} \frac{1}{n}\sum_{i \in [n]} {\rm corr}(Z_i, \hat Z_{\pi(i)}) \ .
    \end{equation}
\end{itemize}

\paragraph{Score functions.}
LSCALE-I and GSCALE-I algorithms, in their first steps, compute estimates of the score differences in the observational environment. The designs of our algorithms are agnostic to how this is performed, i.e., any reliable method for estimating the score differences can be adopted and incorporated into our algorithms in a modular way. In our experiments, we adopt two score estimators necessary for describing the different aspects of our results. 
\begin{itemize}[leftmargin=2em]
    \item {\bf Perfect score oracle for identifiability:} Identifiability, by definition, refers to the possibility of recovering the causal graph and latent variables under all idealized assumptions for the data. Assessing the identifiability guarantees formalized in Theorems~\ref{th:linear-soft}--\ref{th:general-uncoupled} requires using \emph{perfect} estimates for the score differences. Hence, we adopt a perfect score oracle for evaluating identifiability. Specifically, we use a perfect score oracle that computes the score difference inputs in LSCALE-I and GSCALE-I by leveraging \Cref{lm:score-difference-transform-general} and using the ground truth score functions $\bss,\bss^m$ and $\tilde \bss^m$ (see \Cref{appendix:experiments-linear} for details).

    \item {\bf Data-driven score estimates:} For evaluating the accuracy of our algorithms in practice, we need real score estimates, which are inevitably noisy.
    For this purpose, when the pdf $p_\bX$ has a parametric form and the score function $\bss_{\bX}$ has a closed-form expression, then we can estimate the parameters to form an estimated score function. For instance, when $\bX$ follows a linear Gaussian distribution, then we have $\bss_{\bX}(\bx)= - \Theta \cdot \bx$ in which $\Theta$ is the precision matrix of $\bX$ and can be estimated from samples of $\bX$. In other cases in which $\bss_{\bX}$ does not have a known closed-form, we adopt nonparametric score estimators. In particular, we use sliced score matching with variance reduction (SSM-VR) for score estimation due to its efficiency and accuracy for downstream tasks~\citep{song2020sliced}. We also introduce a classification-based score difference estimation method, inspired by \citet[Section~2.1]{gutmann2012noise}. The key observation is that given two distributions $p$ and $q$, the optimal minimum cross-entropy classifier for distinguishing the samples from two distributions is the log density ratio function $\log p/q$. The difference between the score functions of these distributions, $\nabla \log p - \nabla \log q = \nabla \log p/q$ is exactly the gradient of the learned function, which enables us to estimate score differences using a classifier directly.
\end{itemize}

\subsection{LSCALE-I Algorithm for Linear Transformations}\label{sec:experiments-linear}

\paragraph{Data generation.}
To generate $\mcG$, we use Erd{\H o}s-R{\' e}nyi model with density $0.5$ and $n \in \{5,8\}$ nodes, which is generally the size of the latent graphs considered in CRL literature. We consider the observed dimension $d=100$ and generate 100 latent graphs. For the causal mechanisms, we adopt both linear and nonlinear models:
\begin{enumerate}[leftmargin=*]
    \item \textbf{Linear causal model:} We adopt the linear Gaussian model with 
    \begin{equation}
    Z_i = \bA_i \cdot \bZ + N_i \ , \qquad \forall i \in [n] \ , \label{eq:linear-SEM--experiments}
    \end{equation}
    where $\bA_i \in \R^{1 \times n}$ are the rows of the weight matrix $\bA$ in which $\bA_{i,j}\neq 0$ if and only $j \in \Pa(i)$. The nonzero edge weights are sampled from ${\rm Unif}(\pm[0.5, 1.5])$, and the noise terms are zero-mean Gaussian variables with variances $\sigma_{i}^2$ sampled from ${\rm Unif}([0.5, 1.5])$. For node $i$, a hard intervention is given by $Z_i = \bar {N}_i$ where $\bar N_i \sim \mcN(0,\frac{\sigma_i^2}{4})$, and a soft intervention is given by $Z_i =\frac{\bar \bA_{i}}{2} \cdot \bZ + N_{i}$.

    \item \textbf{Quadratic causal model:} We adopt an additive noise model with
    \begin{align}
    Z_i =\sqrt{\bZ_{\Pa(i)}^{\top} \cdot \bQ_{i} \cdot \bZ_{\Pa(i)}} + N_{i}\ , \label{eq:quadratic-model--general-experiments}
    \end{align}
    where $\{\bQ_{i}:i\in[n]\}$ are positive-definite matrices, and the noise terms are zero-mean Gaussian variables with variances $\sigma_{i}^2$ sampled randomly from ${\rm Unif}([0.5,1.5])$. For node $i$, a hard intervention is given by $Z_i = \bar {N}_i$ where $\bar N_i \sim \mcN(0,5 \cdot \sigma_i^2)$ to ensure a non-negligible change. A soft intervention is given by $$Z_i =\frac{1}{2}\sqrt{\bZ_{\Pa(i)}^{\top} \cdot \bQ_{i} \cdot \bZ_{\Pa(i)}} + \bar N_{i}\ .$$

    \item \textbf{Multi-layer perceptron (MLP) causal model}: We adopt an additive noise model with $Z_i = f_i(\bZ_{\Pa(i)}) + N_i$ where $f_i$ is parameterized as a randomly initialized two-layer MLP, with hidden dimension 32 and ReLU activation function. For node $i$, a hard intervention is given by $Z_i = \bar N_i$ where $\mcN(0,5 \cdot \sigma_i^2)$, and a soft intervention is given by $Z_i =\frac{1}{2} f_i(\bZ_{\Pa(i)}) + \bar N_{i}$.
\end{enumerate}
For each graph, we sample $n_{\rm s}$ independent and identically distributed (i.i.d.) samples of $\bZ$ from each environment. We consider $n_{\rm s} \in \{5000, 10000, 50000 \}$ to investigate the effect of the number of samples on the performance of LSCALE-I. The observed variables $\bX$ are generated according to $\bX = \bG \cdot \bZ$, in which $\bG \in \R^{d \times n}$ is a randomly sampled full-rank matrix.

\subsubsection{Hard Interventions}\label{sec:experiments-linear-hard}

\Cref{th:linear-hard} ensures scaling consistency and perfect DAG recovery under hard interventions for linear transformations. As such, we assess the recovery of the latent DAG by the SHD between the estimate $\hat \mcG$ and true graph $\mcG$. For latent variable recovery, we report MCC between $\hat \bZ$ and $\bZ$. 
Note that $\hat \bZ = (\bH \cdot \bG) \cdot \bZ $ implies that the effective transform recovery can also be measured by the closeness of $\bH \cdot \bG$ to the identity matrix. Therefore, in addition to MCC, we also report the normalized \emph{effective transform error}, defined as
\begin{equation}\label{eq:norm_scale_error}
    \ell_{\rm scale} \triangleq \| \bH \cdot \bG - \bI_{n \times n} \|_2 \ .
\end{equation}
\Cref{tab:lscalei-linear-hard} shows the performance of the LSCALE-I algorithm using perfect scores and noisy scores under hard interventions on linear causal models. The first observation is that we achieve excellent performance in latent variable recovery, as demonstrated by a perfect MCC and nearly zero effective transform error $\ell_{\rm scale}$, even when using noisy score estimates. For graph recovery, SHD between the estimated and true latent graphs reduces to less than 0.1 even when using noisy scores given enough samples, e.g., $n_s=50000$ in \Cref{tab:lscalei-linear-hard}. 
Another key observation is that the results remain consistent while the dimension of observed variables increases from $d=25$ to $d=200$, as shown in \Cref{tab:lscalei-linear-hard-n5-varyd}. This confirms our analysis that the performance of LSCALE-I is agnostic to the dimension of the observations. Hence, we suffice by using $d=100$ for the rest of the experiments.

\begin{table}[t]
    \centering
        \footnotesize
        \caption{LSCALE-I for a linear causal model with \textbf{one hard} intervention per node (varying~$n_s$)}
        \label{tab:lscalei-linear-hard}
        \begin{tabular}{ccc|ccc|ccc}
            \toprule
              & & & \multicolumn{3}{c|}{\bf perfect scores} & \multicolumn{3}{c}{\bf noisy scores} \\
             $n$ & $d$ & $n_{\rm s}$ & MCC & $ \ell_{\rm scale}$  & ${\rm SHD}(\mcG,\hat \mcG) $ & MCC & $ \ell_{\rm scale}$ & ${\rm SHD}(\mcG,\hat \mcG) $ \\
             \midrule
            $5$ & $100$ & $5000$ & $1.00 \pm 0.00$ & $0.02 \pm 0.00$ & $0.17 \pm 0.04$ & $1.00 \pm 0.00$ & $0.06 \pm 0.01$ & $0.11 \pm 0.03$  \\
            $5$ & $100$ & $10000$ & $1.00 \pm 0.00$ & $0.01 \pm 0.00$ & $0.09 \pm 0.03$ & $1.00 \pm 0.00$ & $0.04 \pm 0.00$ & $0.09 \pm 0.04$  \\
            $5$ & $100$ & $50000$ & $1.00 \pm 0.00$ & $0.01 \pm 0.00$ & $0.03 \pm 0.01$ & $1.00 \pm 0.00$ & $0.03 \pm 0.00$ & $0.07 \pm 0.03$  \\
            \midrule
            $8$ & $100$ & $5000$ & $1.00 \pm 0.00$ & $0.03 \pm 0.00$ & $0.03 \pm 0.02$ & $1.00 \pm 0.00$ & $0.08 \pm 0.00$ & $0.46 \pm 0.08$  \\
            $8$ & $100$ & $10000$ & $1.00 \pm 0.00$ & $0.02 \pm 0.00$ & $0.06 \pm 0.02$ & $1.00 \pm 0.00$ & $0.06 \pm 0.00$ & $0.18 \pm 0.04$  \\
            $8$ & $100$ & $50000$ & $1.00 \pm 0.00$ & $0.01 \pm 0.00$ & $0.02 \pm 0.01$ & $1.00 \pm 0.00$ & $0.03 \pm 0.00$ & $0.04 \pm 0.02$  \\
            \bottomrule
        \end{tabular}
\end{table}
\begin{table}[t]
    \centering
        \footnotesize
        \caption{LSCALE-I for a linear causal model with \textbf{one hard} intervention per node (varying~$d$)}
        \label{tab:lscalei-linear-hard-n5-varyd}
        \begin{tabular}{ccc|ccc|ccc}
            \toprule
              & & & \multicolumn{3}{c|}{\bf perfect scores} & \multicolumn{3}{c}{\bf noisy scores} \\
             $n$ & $d$ & $n_{\rm s}$ & MCC & $ \ell_{\rm scale}$  & ${\rm SHD}(\mcG,\hat \mcG) $ & MCC & $ \ell_{\rm scale}$ & ${\rm SHD}(\mcG,\hat \mcG) $ \\
             \midrule
            $5$ & $25$ & $10000$ & $1.00 \pm 0.00$ & $0.01 \pm 0.00$ & $0.28 \pm 0.05$ & $1.00 \pm 0.00$ & $0.03 \pm 0.00$ & $0.03 \pm 0.02$  \\
            $5$ & $50$ & $10000$ & $1.00 \pm 0.00$ & $0.01 \pm 0.00$ & $0.23 \pm 0.04$ & $1.00 \pm 0.00$ & $0.03 \pm 0.00$ & $0.04 \pm 0.02$  \\
            $5$ & $100$ & $10000$ & $1.00 \pm 0.00$ & $0.01 \pm 0.00$ & $0.06 \pm 0.02$ & $1.00 \pm 0.00$ & $0.04 \pm 0.00$ & $0.09 \pm 0.03$  \\
            $5$ & $200$ & $10000$ & $1.00 \pm 0.00$ & $0.01 \pm 0.00$ & $0.07 \pm 0.03$ & $1.00 \pm 0.00$ & $0.05 \pm 0.00$ & $0.16 \pm 0.04$  \\
            \bottomrule
        \end{tabular}
\end{table}

\begin{table}[t]
    \centering
        \small
        \caption{LSCALE-I for a quadratic causal model with  \textbf{one hard} intervention per node ($n_{\rm s}=50000$). }
        \label{tab:lscalei-quadratic-hard-n5}
        \begin{tabular}{cc|ccc|ccc}
            \toprule
              & &  \multicolumn{3}{c|}{\bf perfect scores} & \multicolumn{3}{c}{\bf noisy scores} \\
             $n$ & $d$ & MCC & $\ell_{\rm scale}$ & ${\rm SHD}(\mcG,\hat \mcG) $ & MCC &  $\ell_{\rm scale}$ & ${\rm SHD}(\mcG,\hat \mcG) $ \\
            \midrule
             $5$ & $100$ & $1.00 \pm 0.00$ & $0.01 \pm 0.00$ & $0.03 \pm 0.02$ & $0.93 \pm 0.01$ & $0.69 \pm 0.02$ & $2.62 \pm 0.20$  \\
             \bottomrule
        \end{tabular}
\end{table}

\paragraph{Causal models with more complex score functions.}
Next, we test the performance of LSCALE-I on quadratic and MLP causal models specified earlier.  A difference in this setting compared to linear causal models is that we estimate the score functions of the observed variables using SSM-VR~\citep{song2020sliced} as $p_\bX$ is not amenable to parameter estimation. 
For quadratic causal models, \Cref{tab:lscalei-quadratic-hard-n5} shows that using perfect scores, LSCALE-I performs nearly perfectly. Under noisy scores, an MCC of 0.93 indicates a strong performance for the latent variable recovery. The graph recovery performance suffers more from noisy score estimates, yet it remains reasonable with an approximate SHD of $2.62$ (where the expected number of true edges is $5$). \Cref{tab:lscalei-mlp-hard-n5} shows a similar trend, where perfect scores ensure perfect performance, whereas noisy score estimates lead to a degradation in graph recovery performance.

\paragraph{Effect of noisy score estimates on graph recovery under noisy scores.}
\Cref{tab:lscalei-quadratic-hard-n5} and \Cref{tab:lscalei-mlp-hard-n5} show that graph recovery performance suffers more than the latent variable recovery under noisy score estimates. This discrepancy is due to the difficulty of applying \eqref{eq:linear-hard-parent-set} under noisy scores. Specifically, we observe that the nonzero entries in true score differences (entries of $\bsd^m$) can be small, especially for MLP causal models. As such, under noisy score estimates, it becomes more difficult to successfully threshold the empirical score difference quantities to identify the edges of the latent graph. In \Cref{sec:experiments-sensitivity}, we provide further empirical evaluations on the effect of score estimation quality on the performance of LSCALE-I.

\begin{table}[t]
    \centering
        \small
        \caption{LSCALE-I for an MLP causal model with  \textbf{one hard} intervention per node ($n_{\rm s}=~50000$). }
        \label{tab:lscalei-mlp-hard-n5}
        \begin{tabular}{cc|ccc|ccc}
            \toprule
              & &  \multicolumn{3}{c|}{\bf perfect scores} & \multicolumn{3}{c}{\bf noisy scores} \\
             $n$ & $d$ & MCC & $\ell_{\rm scale}$ & ${\rm SHD}(\mcG,\hat \mcG) $ & MCC &  $\ell_{\rm scale}$ & ${\rm SHD}(\mcG,\hat \mcG) $ \\
            \midrule
             $5$ & $100$ & $1.00 \pm 0.00$ & $0.03 \pm 0.00$ & $0.01 \pm 0.01$ & $0.94 \pm 0.01$ & $0.62 \pm 0.02$ & $4.27 \pm 0.20 $  \\
             \bottomrule
        \end{tabular}
\end{table}

\subsubsection{Soft Interventions}\label{sec:experiments-linear-soft}
To evaluate soft interventions, we first consider linear causal models. In this setting, \Cref{th:linear-soft} ensures the recovery of the latent variables and the latent DAG up to ancestors when using soft interventions. So, we report the SHD between the transitive closure of the estimate $\hat \mcG_{\rm tc}$ and that of the true graph $\mcG_{\rm tc}$.
Recall that recovery of latent variables for generic soft interventions is guaranteed up to mixing with parents. To measure the accuracy of $\hat \bZ = \bH \cdot \bG \cdot \bZ$ with respect to this guarantee, we define $\bL_{\rm pa}$ with entries $[\bL_{\rm pa}]_{i,j} = \mathds{1}\{j \in \Paplus(i)\}$, and define $\ell_{\rm pa}$ by
\begin{equation}\label{eq:norm_parents_error}
    \ell_{\rm pa} \triangleq \| \bH \cdot \bG \odot  (\boldsymbol{1}_{n \times n} - \bL_{\rm pa}) \|_2 \ ,
\end{equation}
which effectively measures the effect of incorrect mixing in estimated latent variables.
\Cref{tab:lscalei-linear-soft} shows that by using perfect scores from as few as $n_{\rm s}=5000$ samples, we meet the theoretical identifiability guarantees, implied by zero values of SHD and incorrect mixing norm $\ell_{\rm pa}$. Similar to the case of hard interventions, we observe that increasing the number of samples improves the performance under noisy scores. For instance, given $n_{\rm s}=50000$ samples, the average SHD between the transitive closures $\hat \mcG_{\rm tc}$ and $\mcG_{\rm tc}$ is approximately $0.5$ for $n=8$, where graph density being $0.5$ implies that the expected number of edges is $14$. Also, the latent variable recovery becomes near perfect, indicated by MCC of $0.98$ and near-zero incorrect mixing norm $\ell_{\rm pa}$.

\begin{table}[t]
    \centering
        \footnotesize
        \caption{LSCALE-I for a linear causal model with  \textbf{one soft} intervention per node.}
        \label{tab:lscalei-linear-soft}
        \begin{tabular}{ccc|ccc|ccc}
            \toprule
              & & & \multicolumn{3}{c|}{\bf perfect scores} & \multicolumn{3}{c}{\bf noisy scores} \\
             $n$ & $d$ & $n_{\rm s}$ & MCC & $ \ell_{\rm pa}$  & ${\rm SHD}(\mcG_{\rm tc},\hat \mcG_{\rm tc}) $ & MCC & $ \ell_{\rm pa}$ & ${\rm SHD}(\mcG_{\rm tc},\hat \mcG_{\rm tc}) $ \\
             \midrule
            $5$ & $100$ & $5000$ & $0.98 \pm 0.00$ & $0.00 \pm 0.00$ & $0.01 \pm 0.00$ & $0.98 \pm 0.00$ & $0.04 \pm 0.00$ & $0.59 \pm 0.11$  \\
            $5$ & $100$ & $10000$ & $0.98 \pm 0.00$ & $0.00 \pm 0.00$ & $0.00 \pm 0.00$ & $0.98 \pm 0.00$ & $0.03 \pm 0.00$ & $0.36 \pm 0.08$  \\
            $5$ & $100$ & $50000$ & $0.98 \pm 0.00$ & $0.00 \pm 0.00$ & $0.00 \pm 0.00$ & $0.98 \pm 0.00$ & $0.01 \pm 0.00$ & $0.28 \pm 0.06$  \\
            \midrule
            $8$ & $100$ & $5000$ & $0.98 \pm 0.00$ & $0.00 \pm 0.00$ & $0.00 \pm 0.00$ & $0.98 \pm 0.00$ & $0.07 \pm 0.00$ & $3.84 \pm 0.36$  \\
            $8$ & $100$ & $10000$ & $0.98 \pm 0.00$ & $0.00 \pm 0.00$ & $0.00 \pm 0.00$ & $0.98 \pm 0.00$ & $0.05 \pm 0.00$ & $1.23 \pm 0.20$  \\
            $8$ & $100$ & $50000$ & $0.98 \pm 0.00$ & $0.00 \pm 0.00$ & $0.00 \pm 0.00$ & $0.98 \pm 0.00$ & $0.02 \pm 0.00$ & $0.49 \pm 0.10$  \\
            \bottomrule
        \end{tabular}
\end{table}

Next, we consider quadratic causal models. In this case, \Cref{assumption:full-rank} is satisfied, and \Cref{th:linear-soft-full-rank} ensures the perfect recovery of the latent DAG and the recovery of the latent variables up to surrounding variables. Hence, we report SHD between true $\mcG$ and estimate $\hat G$, and the incorrect mixing norm for this setting given by
\begin{equation}\label{eq:norm_sur_error}
    \ell_{\rm sur} \triangleq \| \bH \cdot \bG \odot  (\boldsymbol{1}_{n \times n} - \bL_{\rm sur}) \|_2 \ ,
\end{equation}
where $\bL_{\rm sur}$ is defined with entries $[\bL_{\rm pa}]_{i,j} = \mathds{1}\{j \in \overline{\sur}(i)\}$.
\Cref{tab:lscalei-quadratic-soft-n5} shows that when using perfect scores, LSCALE-I performs nearly perfectly, verifying the results of \Cref{th:linear-soft-full-rank}. Similarly to the case of hard interventions, the performance of LSCALE-I suffers under noisy score estimations, while it remains reasonable, e.g., an SHD of $2.79$ with respect to the true graph when using only soft interventions.

\begin{table}[t]
    \centering
        \small
        \caption{LSCALE-I for a quadratic causal model with  \textbf{one soft} intervention per node ($n_{\rm s}=50000$). }
        \label{tab:lscalei-quadratic-soft-n5}
        \begin{tabular}{cc|ccc|ccc}
            \toprule
              & &  \multicolumn{3}{c|}{\bf perfect scores} & \multicolumn{3}{c}{\bf noisy scores} \\
             $n$ & $d$ & MCC & $\ell_{\rm sur}$ & ${\rm SHD}(\mcG,\hat \mcG) $ & MCC &  $\ell_{\rm sur}$ & ${\rm SHD}(\mcG,\hat \mcG) $ \\
            \midrule
             $5$ & $100$ & $0.87 \pm 0.01$ & $0.34 \pm 0.04$ & $0.53 \pm 0.09$ & $0.60 \pm 0.01$ & $ 0.51 \pm 0.03$ & $2.79 \pm 0.18$ \\
             \bottomrule
        \end{tabular}
\end{table}

\subsection{GSCALE-I Algorithm for General Transformations}\label{sec:experiments-general}

Next, we focus on nonlinear transformations to showcase the settings for which the existing interventional CRL literature lacks provably correct algorithms and provides only identifiability results under nonlinear transformations.

\paragraph{Choice of nonlinearity.} In this section, we consider M-layer perceptrons with $\tanh$ activation as our class of nonlinear transformations. Specifically, we consider functions of the form
\begin{equation}\label{eq:transf_exp-2}
    \bX = g(\bZ) = (\tanh \circ \bA^{M} \circ \cdots \tanh \circ \bA^{1}) (\bZ) \ ,
\end{equation}
in which $\tanh$ is applied element-wise, and parameters $\{\bA^{1}, \dots, \bA^{M}\}$ are compatible, randomly sampled \emph{full column rank} matrices. We consider two settings of increasing complexity: 1-layer MLP, and 3-layer MLP (for nonparametric transformations, see experiments on images in \Cref{sec:experiments-image}).

\subsubsection{Single-layer MLP}

First, consider the single-layer perceptron as the nonlinear transformation $g$, given by
\begin{equation}\label{eq:transf_exp-1}
    \bX = g(\bZ) = \tanh (\bG \cdot \bZ) \ ,
\end{equation}
where $\bG \in \R^{d \times n}$.
Because it is a relatively simple nonlinear transformation, this setting enables us to evaluate the performance of GSCALE-I more effectively. Specifically, leveraging \eqref{eq:transf_exp-1}, we parameterize valid encoders $h$ with parameter $\bH \in \R^{n \times d}$, which gives
\begin{align}\label{eq:tanh-GLM-parameterization}
    \hat \bZ(\bX;h) &= h(\bX) = \bH \cdot \arctanh (\bX) \ , \\
    \hat \bX &= h^{-1}(\hat \bZ(\bX;h)) = \tanh (\bH^{\dag} \cdot \hat \bZ(\bX;h)) \ , \\
     \hat \bZ &=  \bH \cdot \bG \cdot \bZ \ .
\end{align}
This is equivalent to reducing the original nonparametric functional estimation problem to a finite-dimensional parameter estimation problem. This enables stable training and large-scale testing. Furthermore, this parametrization enables us to compute the ground truth score functions of the observed variables, allowing us to directly assess the effect of score estimation errors on \text{GSCALE-I} performance.

\paragraph{Data generation.} In this setting, we focus on quadratic causal models. To generate $\mcG$ we use the Erd{\H o}s-R{\' e}nyi model with density $0.5$ and $n \in \{5,8\}$ nodes. For observational causal mechanisms, we use Equation \eqref{eq:quadratic-model--general-experiments} and the corresponding parameterization. For the two hard interventions on node $i$, $Z_i$ is set to $N_{q,i} \sim \mcN(0,\sigma_{q,i}^2)$ and $N_{\tilde q,i} \sim \mcN(0,\sigma_{\tilde q,i}^2)$, where we set $\sigma_{q,i}^2 = \sigma_{i}^2 / 4$ and $\sigma_{\tilde q,i}^2 = 4 \cdot \sigma_{i}^2$. Similarly to LSCALE-I experiments, we consider a target dimension of $d=100$ and generate 20 latent graphs, with $n_{\rm s} = 30000$ samples per environment for each graph. 

\paragraph{Results.} Recall that for general transformations, we can guarantee $\hat Z_i = \phi_i(Z_i)$ for a diffeomorphism $\phi_i$. However, MCC$(Z_i, \phi(Z_i))$ can largely deviate from 1 for a nonlinear $\phi_i$. For instance, for $Z_i \sim \mcN(0,1)$ and $\phi(Z_i) = Z_i^3 + 0.1 Z_i$, we have MCC$(Z_i, \phi(Z_i)) \approx 0.786$. On the other hand, when we use the parameterization in \eqref{eq:tanh-GLM-parameterization},  the only element-wise diffeomorphism between $\bZ$ and $\hat \bZ$ is an element-wise scaling. Therefore, MCC remains a perfectly informative metric. \Cref{tab:gscalei-quadratic} shows that we can almost perfectly recover the latent variables (indicated by MCC of $1$) latent DAG for $n=5$ nodes by using perfect scores. Furthermore, increasing the latent dimension to $n=8$ has no significant impact on performance.

\begin{table}[t]
    \centering
        \caption{GSCALE-I for a quadratic causal model with  \textbf{two coupled hard} interventions per node. Noisy scores are obtained using SSM-VR with $n_{\rm score}=30000$ samples.}
        \label{tab:gscalei-quadratic}
        \begin{tabular}{cccc|cc|cc}
            \toprule
              & &  & expected num. & \multicolumn{2}{c|}{\bf perfect scores} & \multicolumn{2}{c}{\bf noisy scores} \\
             $n$ & $d$ & $n_{\rm s}$  & edges in $\mcG$  & MCC  & ${\rm SHD}(\mcG,\hat \mcG) $ & MCC  & ${\rm SHD}(\mcG,\hat \mcG) $ \\
            \midrule
            $5$ & $100$  & $200$ & $5$ & $1.00 \pm 0.00$ & $0.00 \pm 0.00$ & $0.85 \pm 0.02$  & $4.50 \pm 0.38$  \\
            \midrule
            $8$ & $100$ & $500$ & $14$ &$0.95 \pm 0.01$ & $1.50 \pm 0.27$ & $0.75 \pm 0.02$ & $12.9 \pm 0.44$ \\
            \bottomrule
        \end{tabular}
\end{table}

\paragraph{The effect of the quality of score estimation.} When using noisy scores, the performance of GSCALE-I degrades significantly. For instance, MCC goes down to approximately $0.85$ for $n=5$ and $0.75$ for $n=8$. This degradation is similar to what happens when using LSCALE-I on quadratic causal models in \Cref{tab:lscalei-quadratic-soft-n5}. It is noteworthy that the transition from perfect to noisy scores is remarkably smoother when using LSCALE-I on linear causal models (\Cref{tab:lscalei-linear-hard} and \Cref{tab:lscalei-linear-soft}). This discrepancy is attributed to the distinct score estimation procedures adopted in the two experimental settings. Specifically, linear Gaussian latent models allow us to directly estimate the parameters of the closed-form score function $\bss_{\bX}$. However, when using a quadratic latent model, we rely on a nonparametric score estimation via SSM-VR~\citep{song2020sliced}. This comparison between two experiment settings and results underscores that the performance gap between the theoretical guarantees and practical results can be significantly mitigated through advances in general score estimation techniques. In \Cref{sec:experiments-sensitivity}, we provide a further empirical evaluation of how the quality of the score estimation affects the final performance.

\subsubsection{Multi-layer nonlinearity}

Next, we consider general MLPs by setting the number of layers to 3, which is in line with the MLP depths considered in CRL literature~\citep{liang2023causal}. In this setting, we do not use any knowledge of the transformation parametrization and aim to learn a generic NN-based encoder-decoder pair $h, h^{-1}$ parameterized by $\theta_{\rm enc}, \theta_{\rm dec}$. 
In essence,  this enables us to test our algorithm's performance on a standard procedure of function approximation via neural networks.

\paragraph{Data generation.}
In this setting, we focus on linear causal models. To generate $\mcG$, we use the Erd{\H o}s-R{\' e}nyi model with density $0.5$ and $n=5$ nodes. For observational causal mechanisms, we use Equation \eqref{eq:linear-SEM--experiments} and its associated parameterization. For the two hard interventions on node $i$, $Z_i$ is set to $N_{q,i} \sim \mcN(0,\sigma_{q,i}^2)$ and $N_{\tilde q,i} \sim \mcN(0,\sigma_{\tilde q,i}^2)$, where we set $\sigma_{q,i}^2 = \sigma_{i}^2/4$ and $\sigma_{\tilde q,i}^2 = \sigma_{i}^2 \cdot 4$. Due to the increased computational complexity of training, we consider target dimension value $d=5$, and we generate 5 latent graphs and $n_{\rm s} = 10000$ samples per graph.

\paragraph{Results.} GSCALE-I algorithm ensures latent variables recovery up to element-wise transformations. As a proxy, we report its linear analogue, i.e., MCC. We summarize the results of these experiments in \Cref{tab:mlp-mlp-mcc-results}.

\begin{table}
    \centering
    \caption{MCC across different runs in MLP transform + MLP causal model setting}
    \label{tab:mlp-mlp-mcc-results}
    \begin{tabular}{ccccc|c}
        \toprule
        \multicolumn{5}{c}{Runs} & Mean (std. error) \\
        \midrule
        0.64 & 0.54 & 0.43 & 0.58 & 0.54 & $0.54 \pm 0.03$ \\
        \bottomrule
    \end{tabular}
\end{table}
In this setting, we obtain a smaller average MCC rate ($0.54$) compared to the 1-layer nonlinearity experiments in \Cref{tab:gscalei-quadratic}. Performance degradation can be attributed to two factors: First, making the transformation more complex and making it more difficult to construct \emph{any} kind of autoencoder. In contrast to the parametric single-layer MLP setting in the previous section, fitting a generic NN has considerably higher sample complexity. Second, our methodology requires the estimation of score functions with relatively good accuracy. However, complex distributions can push forward distributions into very complex and, similarly, difficult-to-estimate landscapes, which can limit the accuracy of our algorithms.

\subsection{Comparison with the Existing CRL Studies} \label{sec:experiments:comparisons}

In this section, we compare the performance of LSCALE-I with those of the approaches designed for comparable settings in~ \citep{squires2023linear,zhang2023identifiability}.
\subsubsection{Comparison with linear causal models on synthetic data} \label{sec:experiments:comparisons:squires}
A closely related study in the linear transformation setting is \citep{squires2023linear}. As discussed in \Cref{sec:linear}, \citep{squires2023linear} presents identifiability guarantees and provides algorithms for linear transformations with linear causal models. For a fair comparison, we perform comparisons for linear Gaussian SEMs under the parameterization detailed in Section~7.1. Experiments with synthetic linear causal models in \citep{squires2023linear} mostly focus on recovering the parameters of the true encoder $\bG^\dag$. Hence, in \Cref{tab:lscalei-comparison-squires}, we report performance comparisons in terms of the latent variable recovery metric $\ell_{\rm scale}$. \Cref{tab:lscalei-comparison-squires} shows that LSCALE-I achieves close to perfect accuracy with as few as $n_s=5000$ samples, whereas the algorithm of \citet{squires2023linear} requires significantly more samples (more than $n_s=50,000$) to achieve a reasonable performance.
\begin{table}[t]
    \centering
        \caption{Comparison of LSCALE-I with the algorithm of \citet{squires2023linear} for a linear causal model with \textbf{one hard} intervention per node.}
        \label{tab:lscalei-comparison-squires}
        \begin{tabular}{ccc|cc|cc}
            \toprule
              & & & \multicolumn{2}{c|}{\bf LSCALE-I } & \multicolumn{2}{c}{\bf \citep{squires2023linear}} \\
             $n$ & $d$ & $n_{\rm s}$ & mean $\ell_{\rm scale}$ & fraction of $ \ell_{\rm scale} < 0.1$  & mean $\ell_{\rm scale}$ & fraction of $ \ell_{\rm scale} < 0.1$ \\
             \midrule
            $5$ & $100$ & $5000$ & $0.06$ & $0.92 $ & $0.66$ & $0.67$  \\
            $5$ & $100$ & $10000$  & $0.04$ & $0.96 $ & $0.51$ & $0.74$  \\
            $5$ & $100$ & $50000$  & $0.03$ & $0.94 $ & $0.24$ & $0.88$  \\
            \midrule
            $8$ & $100$ & $5000$  & $0.08$ & $0.86 $ & $1.28$ & $0.36$  \\
            $8$ & $100$ & $10000$  & $0.06$ & $0.95 $ & $1.09$ & $0.45$  \\
            $8$ & $100$ & $50000$  & $0.03$ & $0.96 $ & $0.68$ & $0.66$  \\
            \bottomrule
        \end{tabular}
\end{table}

\subsubsection{Comparison with DiscrepancyVAE on synthetic data} \label{sec:experiments:comparisons:discrepancy-vae-synth-data}

In this section, we compare LSCALE-I with the DiscrepancyVAE algorithm of~\citep{zhang2023identifiability} on the setting described in \Cref{sec:experiments-linear}. We consider 10 runs of the quadratic causal model setting with soft interventions, latent dimension $n = 5$, observed dimension $d = 100$, and $50000$ samples per environment. We observe that the average SHD of DiscrepancyVAE on this dataset is $3.1 \pm 0.7$, which is comparable to LSCALE-I performance of $2.79 \pm 0.18$ reported in \Cref{tab:lscalei-quadratic-soft-n5}.

To provide a visual comparison, we consider a single run and plot the ground truth latent graph (\Cref{fig:discrepancy-vae-graph-n5:gt}), the estimate generated by LSCALE-I (\Cref{fig:discrepancy-vae-graph-n5:lscalei}), and the estimate generated by DiscrepancyVAE (\Cref{fig:discrepancy-vae-graph-n5:discvae}). The differences between the latter and the ground truth graph are shown in these figures by red (added) and dashed (deleted) edges. We refer further comparison between these methods to \Cref{sec:discrvae-interv-enc-problem}.
\begin{figure}
    \small
    \centering
    \begin{subfigure}{0.3\textwidth}
        \centering
        \includegraphics[width=0.8\linewidth]{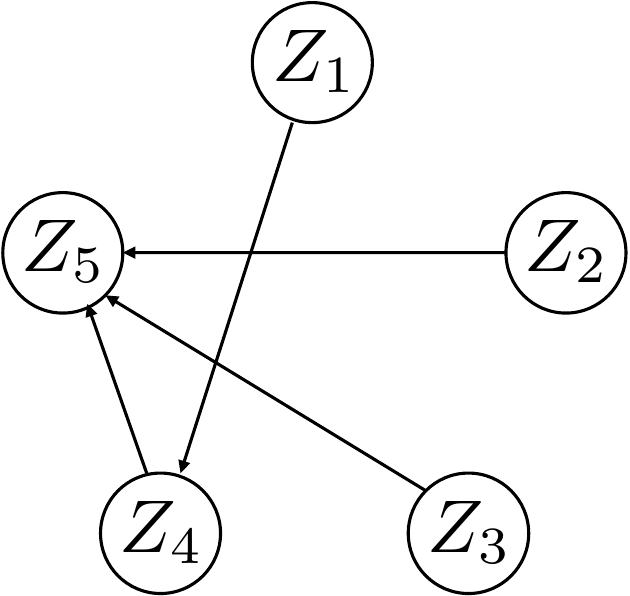}
        \caption{Ground truth}
        \label{fig:discrepancy-vae-graph-n5:gt}
    \end{subfigure}
    \begin{subfigure}{0.3\textwidth}
        \centering
        \includegraphics[width=0.8\linewidth]{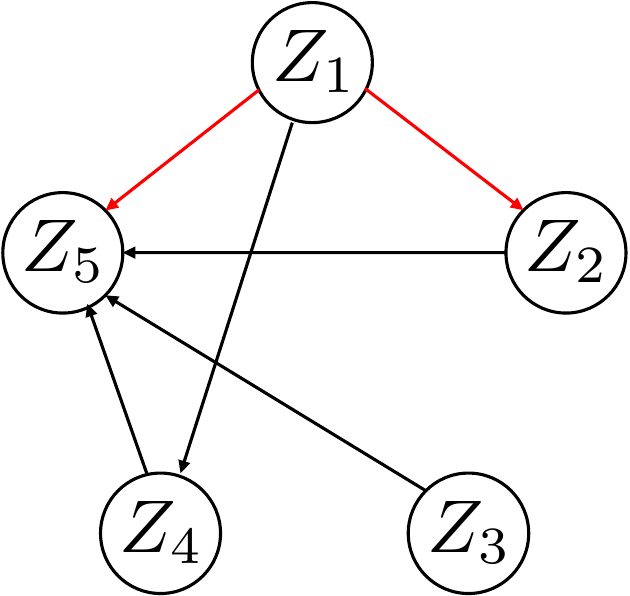}
        \caption{LSCALE-I}
        \label{fig:discrepancy-vae-graph-n5:lscalei}
    \end{subfigure}
    \begin{subfigure}{0.3\textwidth}
        \centering
        \includegraphics[width=0.8\linewidth]{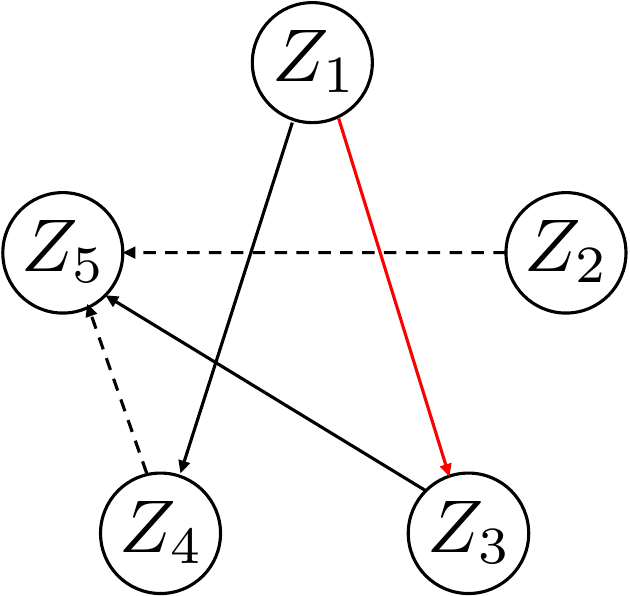}
        \caption{DiscrepancyVAE}
        \label{fig:discrepancy-vae-graph-n5:discvae}
    \end{subfigure}
    \caption{Latent graph recovery comparison. Extra edges are marked in red, and dashed lines mark the missing edges.}
    \label{fig:discrepancy-vae-graph-n5}
\end{figure}

\subsubsection{Comparison with DiscrepancyVAE on biological data} \label{sec:experiments:comparisons:perturb-seq-graph}
In this section, we apply our algorithms to the Perturb-seq dataset of \citet{norman2019exploring}, which is also used in \citet{zhang2023identifiability}. Following the same pre-processing steps, there are 8,907 unperturbed and 99,590 perturbed cells. Each cell (sample) is denoted by a sparse vector of dimension 5000, which represents the expression levels of select genes. The perturbed cells are generated by CRISPR activation~\citep{gilbert2014genome} on one or two target genes out of 105 genes. Hence, this is modeled by a 5000-dimensional observable variable $\bX$ and a 105-dimensional latent variable $\bZ$ representing the perturbed genes. In this dataset, the ground truth latent graph and variables are unknown, with the exception of limited gene interactions that have been experimentally verified~\citep{norman2019exploring}. For graph estimation, instead of working with a latent dimension of 105, \citet{zhang2023identifiability} first learns groups of perturbation targets subject to a regularity constraint and then learns a latent graph among these groups. Since the methodology of LSCALE-I does not naturally extend to learning such groups of latent nodes, we use the perturbation groupings reported by \citet{zhang2023identifiability} as our super-nodes and run LSCALE-I to learn a latent graph over them. We list representative perturbation targets from each super-node in ~\Cref{tab:perturb-seq-super-nodes}.

The most important observation is that LSCALE-I and DiscrepancyVAE both estimate a causal edge from DUSP9 to the MAPK1/ETS2 group, which is a gene activation relation that is demonstrated experimentally~\citep{norman2019exploring}. Due to the lack of ground truth for other graph edges, comparing the algorithms for recovering other edges is not informative. Nevertheless, we provide the estimated latent graph from LSCALE-I (\Cref{fig:perturb-seq-graph-est:lscalei}) and DiscrepancyVAE (\Cref{fig:perturb-seq-graph-est:discrvae}) for completeness.
This verifies that our score-based methodology can recover causal genomics relations.

\begin{figure}
    \centering
    \begin{minipage}{0.28\linewidth}
        \begin{table}[H]
            \centering
            \caption{Representative target genes from super nodes.}
            \label{tab:perturb-seq-super-nodes}
    \scalebox{0.8}{
            \begin{tabular}{cc}
                \toprule
                Super node & Representative \\ index & gene(s) \\
                \midrule
                 0 & OSR2 \\
                 1 & TBX2 \\
                 2 & DUSP9 \\
                 3 & MAPK1, ETS2 \\
                 4 & COL2A1 \\
                 5 & SET \\
                 6 & KLF1 \\
                 \bottomrule
            \end{tabular}
    }
        \end{table}
    \end{minipage}
    \hfill
    \begin{minipage}{0.68\linewidth}
        \begin{figure}[H]
            \centering
            \begin{subfigure}{0.45\linewidth}
                \includegraphics[width=0.9\linewidth]{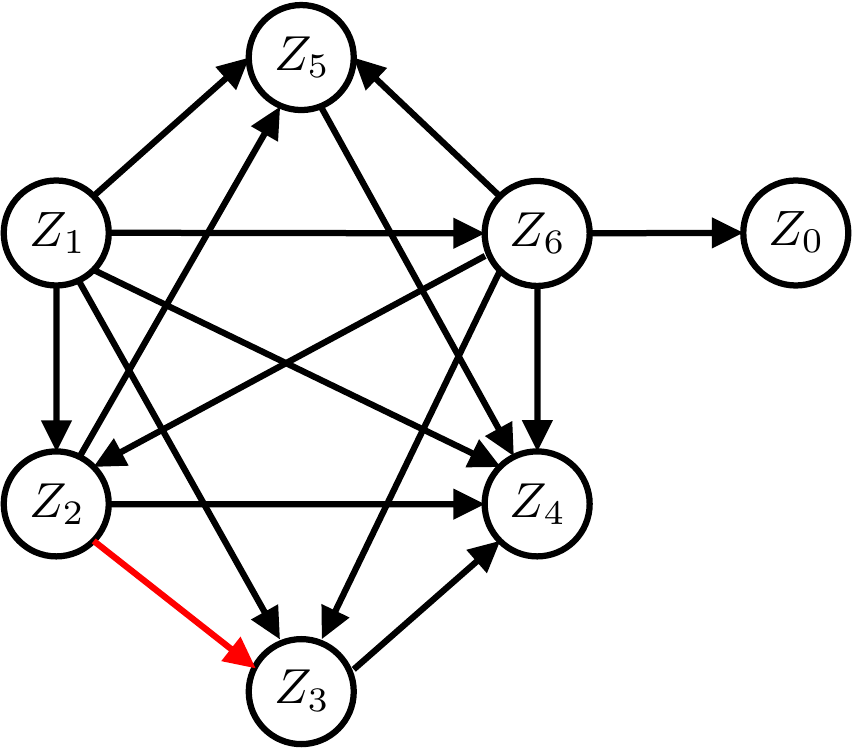}
                \caption{LSCALE-I}
                \label{fig:perturb-seq-graph-est:lscalei}
            \end{subfigure}
            \begin{subfigure}{0.45\linewidth}
                \includegraphics[width=0.9\linewidth]{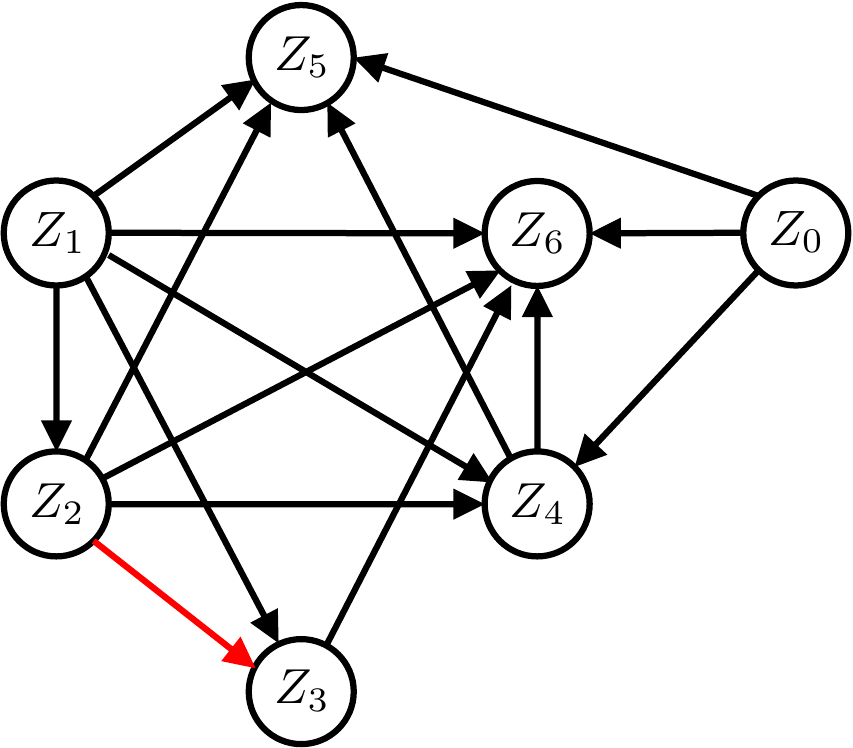}
                \caption{DiscrepancyVAE}
                \label{fig:perturb-seq-graph-est:discrvae}
            \end{subfigure}
            \caption{Graph estimate for Perturb-seq dataset. The edge from DUSP9 to MAPK1/ETS2 is marked in red.}
            \label{fig:perturb-seq-graph-est}
        \end{figure}
    \end{minipage}
\end{figure}

\subsection{Intervention Extrapolation on Biological Data} \label{sec:experiments:interv-extrp-perturbseq}

In \Cref{sec:intervention-extrapolation}, we have demonstrated that it is possible to sample from unseen combinations of given interventions via only learning the score functions of observed variables. We verify this on the Perturb-seq data discussed in \Cref{sec:experiments:comparisons} through the following steps:
\begin{enumerate}[itemsep=0.5pt,leftmargin=2em]
    \item Collect observed variables' data from observational distribution $p_\bX$,  two single-node interventions $p_\bX^1$ and $p_\bX^2$, and double-node intervention $p_\bx^{\{1,2\}}$. 
    \item Compute score functions $\bss_{\bX}$, $\bss_{\bX}^1$, and $\bss_{\bX}^2$.
    \item Perform score function extrapolation for double-node intervention via \eqref{eq:dsx-extrapolation}: $\tilde \bss_{\bX}^{\{1,2\}} = \bss_{\bX}^1 + \bss_{\bX}^2 - \bss_{\bX}$.
    \item Single-node setting: Generate samples from $\bss_{\bX}^1$ and $\bss_{\bX}^2$ using Langevin dynamics~\citep{welling2011bayesian}, compare them to the original samples from $p_\bX^1$ and $p_\bX^2$.
    \item Double-node intervention extrapolation: Generate samples from $\tilde \bss_{\bX}^{\{1,2\}}$ using Langevin dynamics, compare them to the original samples from $p_\bX^{\{1,2\}}$.
\end{enumerate}
We note that \citet{zhang2023identifiability} also report experiments on the same dataset for the single-node data generation and double-node intervention extrapolation settings. However, as discussed in \Cref{sec:intervention-extrapolation}, they perform these tasks using the VAE model for CRL, whereas our approach shows that learning the latent causal representations is unnecessary. Similarly to \citet{zhang2023identifiability}, we use the maximum mean discrepancy (MMD)~\citep{gretton2012kernel} to measure performance quantitatively. In \Cref{tab:perturb-seq-mmd-comparison}, we show that our approach performs at least as well as DiscrepancyVAE on the Perturb-seq dataset. Finally, in \Cref{sec:umap}, we illustrate the simulated interventional distributions via UMAP clustering, which visually supports the performance implied by \Cref{tab:perturb-seq-mmd-comparison}. These results validate the theoretical analysis that we do not need to learn latent variables to perform intervention extrapolation.
\begin{table}[htbp]
    \centering
    \caption{MMD evaluation for single-node and double-node intervention extrapolation of our score-based approach and DiscrepancyVAE~\citep{zhang2023identifiability} on Perturb-seq dataset.}
    \label{tab:perturb-seq-mmd-comparison}
    \begin{tabular}{lc|c}
        \toprule
        & Single-node & Double-node \\
        \midrule
        Score-based & $0.057 \pm 0.013$ & $0.208 \pm 0.036$ \\
        DiscrepancyVAE & $\geq 0.15$ & $\geq 0.2$ \\
        \bottomrule
    \end{tabular}
\end{table}

\subsection{Experiments on Image Data}\label{sec:experiments-image}

In this section, we perform experiments on CRL where the general transformation is image rendering, a highly nonlinear transformation, by applying the GSCALE-I algorithm to synthetic image data.
\paragraph{Image data generation.} For image-based experiments, we follow the setup of the closely related studies in~\citep{ahuja2023interventional,buchholz2023learning}. Specifically, we consider images of the form in \Cref{fig:image-reconstr-samples}, which are generated as follows. The pairs of latent variables $(Z_{2i-1},Z_{2i})$ describe the coordinates of the $i$-th ball's center in a $64 \times 64 \times 3$ RGB image. We use two and three balls in our experiments, which corresponds to $n \in \{4, 6\}$ latent variables. We sample the latent graph $\mcG$ from Erd{\H o}s--R{\' e}nyi model for $n$ nodes and set the expected number of edges to $2n$. Given this graph $\mcG$, we adopt a \emph{truncated} linear Gaussian latent causal model. The details of the linear model are the same as the \Cref{sec:experiments-linear}, with the addition of a second set of interventions represented by the causal mechanisms $\tilde{q}_{i}$, by setting $Z_{i} = \tilde{N}_{i}$, where $\tilde{N}_{i} \sim \mcN(0, \frac{\sigma_{i}^{2}}{4})$. Latent variables $\bZ$ are sampled from the distribution defined by this linear model, but samples with any coordinate absolute value greater than $1$ are discarded to truncate the distribution in the box $[-1, 1]^{n}$. For each latent sample $\bZ$, we render an image of the balls centered at the specified coordinates (shifted and scaled to fit entirely within rendered image boundaries) with a radius of 8 pixels and use these as our observed data $\bX$. Balls are color-coded, filled with three different colors, and the background is white. For both two and three balls, we generate 5 graphs and $10000$ samples from each graph under each environment.
\begin{figure}
    \centering
    \includegraphics[width=0.7\linewidth]{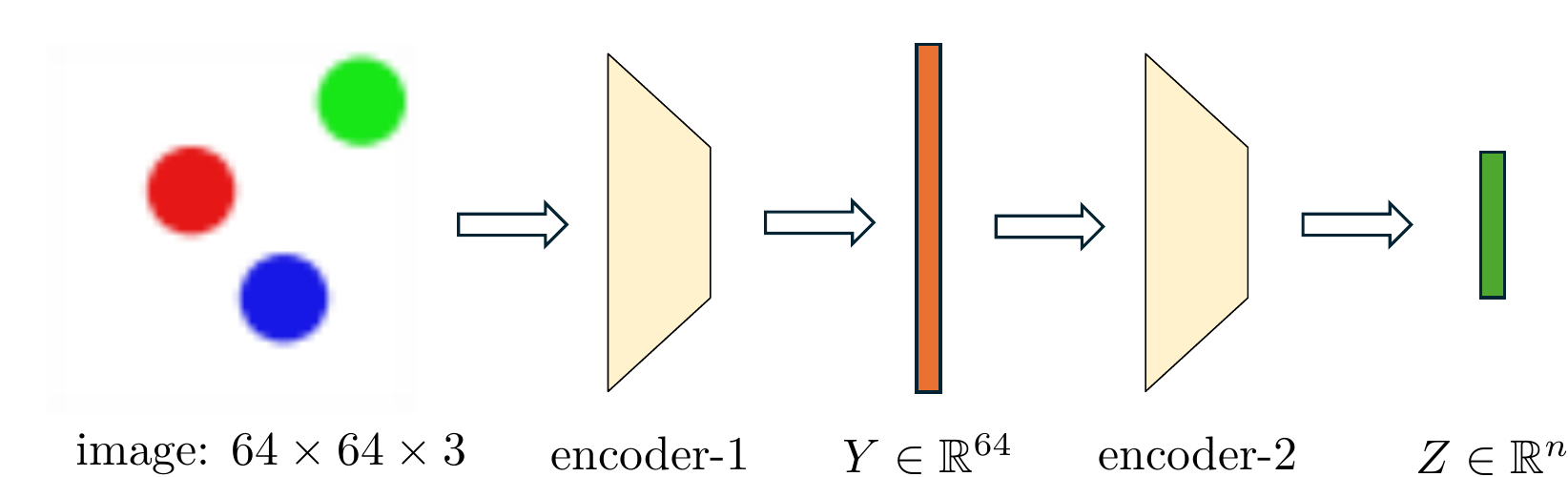}
    \caption{Learning the latent variables from images. Autoencoder-1 is trained to minimize only the reconstruction loss. Autoencoder-2 is trained to minimize the score-based loss in~\eqref{eq:general-OPT1-procedure} while using the reconstruction loss as regularization.}
    \label{fig:image-architecture}
\end{figure}
\paragraph{Candidate encoder and training.}
For learning the latent variables, we construct a two-step autoencoder, depicted in \Cref{fig:image-architecture}. Specifically, in step one, we train an autoencoder on the observational image dataset with a bottleneck dimension of 64 using only a reconstruction objective. In step two, we train another autoencoder on the 64-dimensional output of the first encoder with the bottleneck dimension of $n$ (the latent dimension) using both reconstruction loss and the score-based loss described earlier in \Cref{sec:experiments-general}. Architecture details and training schedule are given in \Cref{tab:autoencoder-architecture}. Finally, we repeat the training of the second autoencoder 5 times for each dataset.
\paragraph{Score difference estimation.}
The main loss function for the second autoencoder of our setup is $\norm{\bD_{\rm t}(h)-\bI_{n \times n}}_{1,1}$, which is equivalent to the score-based loss in~\eqref{eq:general-OPT1-procedure}. For computing $\bD_{\rm t}(h)$ for an encoder $h$, we need to compute the score differences of image datasets. To this end, we adopt a binary classifier-based log density ratio (LDR) estimator and use the gradients of these learned LDRs as our score difference functions, as described earlier. Specifically, we parameterize the LDR at any image through a CNN-based model, details of which are given in \Cref{tab:ldr-architecture}. The output of this model is used to compute class probabilities, which are used to minimize cross-entropy to train the model. We train one LDR model for each pair of hard interventions on the same node and one for the observational-interventional environment pairs for one set of hard interventions.
\begin{figure}[t]
    \centering
    \includegraphics[width=0.7\linewidth]{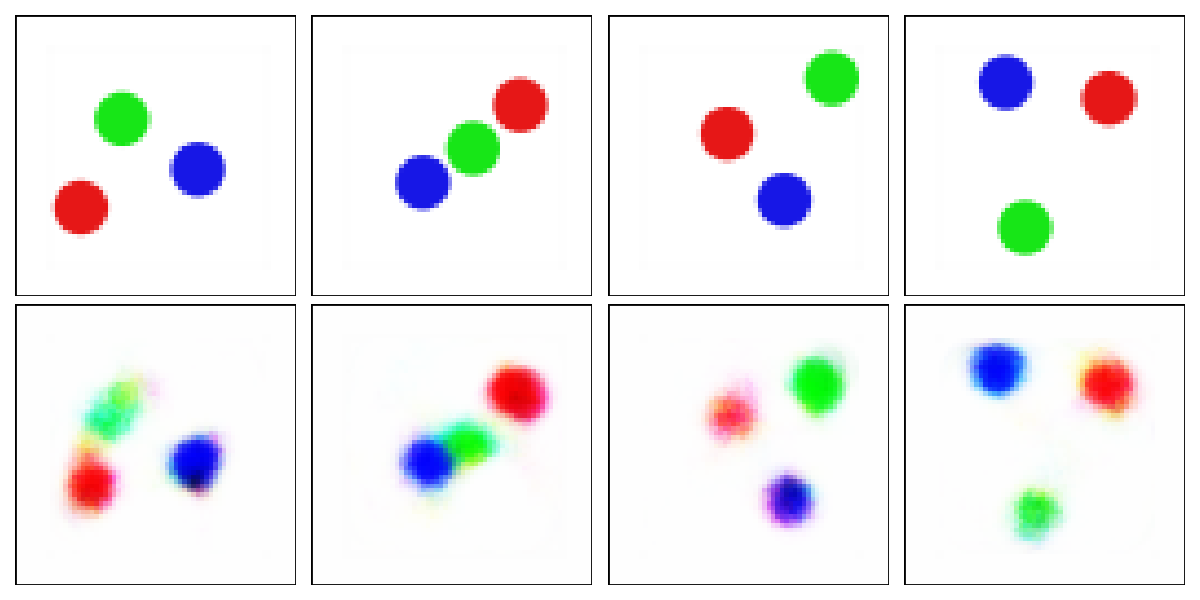}
    \caption{Sample images (top row) versus their reconstructions (bottom row).}
    \label{fig:image-reconstr-samples}
\end{figure}
\paragraph{Results.} GSCALE-I ensures perfect graph recovery and latent variables recovery up to permutation and element-wise scaling. Similarly to related work~\citep{ahuja2023interventional,buchholz2023learning}, we report MCC between $\hat \bZ$ and $\bZ$ as a metric of latent variables recovery, shown in \Cref{tab:image-balls-mcc-comparison}.
\begin{table}
    \centering
    \small
    \caption{MCC comparison in image experiments (over 5 runs).}
    \label{tab:image-balls-mcc-comparison}
    \begin{tabular}{cccc|cccccc|c}
        \toprule
        Algorithm & SCM & \# balls & \# int. / node & int. type & mean (std. error) \\
        \midrule
        GSCALE-I & linear & 2 & 2 & hard & $0.80 \pm 0.03$ \\
        GSCALE-I & linear & 3 & 2 & hard & $0.76 \pm 0.08$ \\
        GSCALE-I & nonlinear & 2 & 2 & hard & $0.93 \pm 0.02$ \\
        \midrule
        GSCALE-I & linear & 2 & 1 & hard & $0.79 \pm 0.03$ \\
        GSCALE-I & nonlinear & 2 & 1 & hard & $0.92 \pm 0.02$  \\
        \midrule
        \citet{ahuja2023interventional} & linear & 2 & 1 & do & $0.13 \pm 0.03$ \\
        \citet{ahuja2023interventional} & linear & 2 & 3 & do & $0.73 \pm 0.03$ \\
        \citet{ahuja2023interventional} & linear & 2 & 5 & do & $0.83 \pm 0.03$ \\
        \midrule
        \citet{buchholz2023learning} & linear & 2 & 1 & hard & $0.87 \pm 0.03$ \\
        \citet{buchholz2023learning} & linear & 5 & 1 & hard & $0.94 \pm 0.01$ \\
        \bottomrule
    \end{tabular}
\end{table}
During training, we observed that MCC highly correlates with the reconstruction performance, i.e., when reconstruction errors in both autoencoders converge to small values, latent variable recovery is also successful. 
We note that the reconstruction does not have to be perfect. For instance, the samples in \Cref{fig:image-reconstr-samples} are from a run with 3 balls with an MCC of 0.98. Conversely, for the runs with non-convergent reconstruction loss in either of the autoencoder training steps, latent variable recovery also struggled.
Hence, as long as we can ensure reasonable reconstruction performance when we are minimizing the score-based loss, our algorithm manages to disentangle the true causal variables embedded in the latent domain.
We report the results from related work in \Cref{tab:image-balls-mcc-comparison} and discuss them as follows. 
Overall, we achieve a better performance than that of \citet{ahuja2023interventional}, and slightly weaker than that of \citet{buchholz2023learning} for linear SCMs. This can perhaps be explained by the versatility of our approach, which works for general SCMs, whereas the algorithm of \citet{buchholz2023learning} is designed to work exclusively for linear SCMs, both theoretically and due to its loss function design. We demonstrate in \Cref{tab:image-balls-mcc-comparison} that our algorithm performs even better when the latent SCM is nonlinear (specifically, quadratic) than when it is linear, thereby verifying its versatility. Finally, we also test our algorithm when using one intervention per node, using the score differences between interventional and observational environments instead of two interventional environments.
While two interventions per node are required for theoretical guarantees, our experiments show that using one intervention per node generally leads to competitive performance, as illustrated in \Cref{tab:image-balls-mcc-comparison}.

\paragraph{Gap between the theory and practice in CRL.} 
We highlight that the nature of the results presented in this paper is theoretical. We have established unknown identifiability guarantees in several settings. Furthermore, we have designed algorithms by defining a differentiable loss function whose global optima achieve identifiability for general CRL and established that the algorithms generate provably correct representations. The experiments presented on various synthesized and real datasets demonstrate the potential of the score-based framework for complex nonlinear transforms, such as image rendering. We achieve this goal by showing reasonable performance in latent variable recovery. That being said, the considered image dataset and other synthetic datasets are still relatively simple compared to real-world problems, and a gap remains between theoretical guarantees and practical applications. In other words, the theory developed is only a \emph{necessary} step for guaranteed practical performance. From our perspective, closing this gap and obtaining \emph{sufficient} conditions for practical performance requires improvements in two aspects. First, using better score difference estimators (which we empirically investigate in \Cref{sec:experiments-sensitivity}) would enable a more effective application of our score-based CRL framework. Second, designing more suitable architectures for leveraging the invariance properties of causal models can help scale up CRL applications.

\subsection{Sensitivity to Score Estimation Noise}
\label{sec:experiments-sensitivity}

As discussed earlier, score estimators can be modularly incorporated into our algorithm.
For cases where $p_{\bX}$ is not amenable to parameter estimation, we can use the noisy score estimates generated by SSM-VR as achievable baselines, as presented in this section. In this subsection, we evaluate the potential performance improvement that can be achieved as the score estimates become more accurate. To this end, we test the LSCALE-I algorithm under varying levels of score estimation noise. Specifically, we run the LSCALE-I algorithm using the scores generated by the following model:
\begin{equation}
    \hat \bss_{\bX}(\bx; \sigma^2) = \bss_{\bX}(\bx) \cdot \big( 1 + \Xi \big) \ , \quad \mbox{where} \quad \Xi \sim \mcN(0, \sigma^2 \cdot \bI_{d \times d}) \ .
\end{equation}
We consider hard interventions on graph size of $n = 5$, set $d = 25$, and vary the value of $\sigma^2$ within $[10^{-4}, 10^{-2}]$. We repeat the experiments 100 times with $n_{\rm s}=10000$ samples from each environment. We plot mean normalized $\bZ$ error, defined as  
\begin{equation}\label{eq:normalized_loss}
    \ell_{\rm norm}(\bZ,\hat \bZ) \triangleq \frac{\|\bZ-\hat \bZ\|_2}{\norm{\bZ}_2}\ .
\end{equation}   
and mean SHD in \Cref{fig:vary-noise-n-5} with respect to the signal-to-noise ratio (SNR). It is clear that when the score estimation error is small, indicated by a high SNR value, the LSCALE-I algorithm demonstrates a strong performance in recovering both latent causal variables and the latent graph. We also note that the baseline results with noisy scores computed via SSM-VR in \Cref{tab:lscalei-quadratic-hard-n5} yield similar success at graph recovery at the SNR of approximately $25$ dB and latent recovery at the SRM of approximately $8$ dB. The curves also confirm our observations in \Cref{sec:experiments-linear} and \Cref{sec:experiments-general} that graph recovery is more sensitive to score estimation errors. The trend of the curves in \Cref{fig:vary-noise-n-5} indicates that our algorithm would greatly benefit from a better score estimator.

\begin{figure}[t]
    \centering
    \begin{subfigure}[t]{0.45\textwidth}
        \centering
        \includegraphics[width=\linewidth]{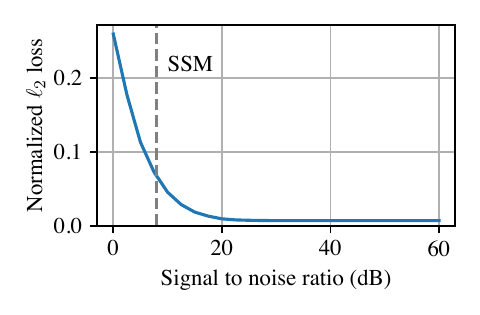}
        \caption{Normalized $\ell_2$ loss versus SNR}
        \label{fig:norm-z-err-vs-noise-n-5}
    \end{subfigure}
    \begin{subfigure}[t]{0.45\textwidth}
        \centering
        \includegraphics[width=\linewidth]{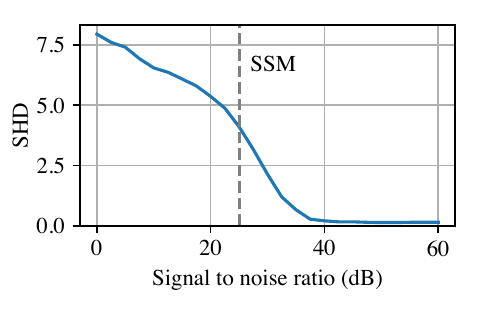}
        \caption{SHD versus SNR}
        \label{fig:shd-h-vs-noise-n-5}
    \end{subfigure}
    \caption{The performance of LSCALE-I under noisy scores with varying SNR for $n = 5$ and $d = 25$. The dashed vertical lines correspond to SNR values that correspond to the performance attained by SSM-VR.}
    \label{fig:vary-noise-n-5}
\end{figure}

\section{Discussion and Concluding Remarks}\label{sec:conclusion}

In this paper, we have proposed a score function-based CRL framework that uses stochastic interventions to learn latent causal representations and the latent causal graph underlying them. In this framework, by uncovering novel connections between score functions and CRL, we have established identifiability results for linear and general transformations without restricting the latent causal models, and designed LSCALE-I and GSCALE-I algorithms that achieve these identifiability guarantees. There are several exciting directions for future work.

As discussed in \Cref{sec:objectives} after defining the identifiability objectives, and emphasized when presenting the corresponding results through the paper, our results are tight \emph{given a complete set of atomic interventions.} A natural direction for future work is relaxing the atomic intervention requirement. In this aspect, the paper \citep{varici2024linear} extends the score-based framework for linear transformations into the setting of unknown multi-node interventions by using combinations of multi-node score differences. Given a sufficiently diverse multi-node intervention set, similar identifiability results to those in \Cref{sec:linear} are shown. That being said, establishing the \emph{necessary} conditions for multi-node interventions and the case of general transformations remains an open problem. 
Next, we note that our LSCALE-I algorithm for \emph{linear transformations} is agnostic to the intervention type and latent causal model, meaning that it can be used for both soft and hard interventions and different causal models. Designing similar universal CRL algorithms that can handle general transformations under different sizes and types of interventional environments is the ultimate goal for studying the identifiability of CRL from interventions. A missing component of existing CRL literature is the finite-sample analysis. Probabilistic identifiability results for a given number of interventional data samples can be useful, especially in applications where performing interventions is costly. In this direction, the recent paper \citep{acarturk2024sample} establishes finite sample guarantees for CRL under linear transformations. Extending such analysis to general transformations can provide insights for bridging the gap between theory and practice in CRL.

\acks{
This work was supported in part by IBM through the IBM-Rensselaer Future of Computing Research and the National Science Foundation under Award ECCS-193310.
}

\newpage
\appendix
\part{Appendix}
\parttoc
\begin{table}[H]\caption{{Notation.}} 
    {\small
       \begin{tabular}{l r c p{12cm} }
            \toprule
            \multicolumn{3}{c}{}\\
           & $[n]$ & : & $\{1, \dots, n\}$ \\
           & $\mathds{1}$ & : & indicator function \\
           ground-truth &    $\mcG$ &: & latent causal graph over $\bZ$ \\
           variables &  $\mcG_{\rm tc}$ & : & transitive closure of $\mcG$ \\
           &  $\mcG_{\rm tr}$ & : & transitive reduction of $\mcG$ \\
           &  $\Pa(i)$ & : & parents of node $i$ in $\mcG$ \\
           &  $\Ch(i)$ & : & children of node $i$ in $\mcG$ \\
           &  $\An(i)$ & : & ancestors of node $i$ in $\mcG$ \\
           &  $\De(i)$ & : & descendants of node $i$ in $\mcG$ \\
            & $\bX$ & : & $[X_1,\dots,X_d]^{\top}$ \;\; observed random variables \\
            & $\bZ$ & : & $[Z_1,\dots,Z_n]^{\top}$ \;\; latent random variables \\
             &   $\bZ_{\Pa(i)}$ & : & vector formed by $Z_i$ for all $i\in \Pa(i)$ \\
            & $g$ & : & true decoder\\
            & $h$ & : & a valid encoder\\
            & $\mcH$ & : & the set of valid encoders $h$ \\
            &  $\bG$ & : & true linear decoder\\
            & $\bH$ & : & a valid linear encoder \\ 
            intervention  
            &  $\mcE^0$ &: & observational environment \\
           notations &  $\mcE$ &: & $(\mcE^1,\dots,\mcE^n)$ \; first interventional environments \\
           &  $\tilde \mcE$ &: & $(\tilde \mcE^1,\dots,\tilde \mcE^n)$ \; second interventional environments \\
           &  $I^m$ & : & the intervened node(s) in $\mcE^m$ \\
           &  $\tilde I^m$ & : & the intervened node(s) in $\tilde \mcE^m$ \\
           &  $\mcI$ & : & the set of intervened nodes $(I^1,\dots,I^n)$\\
           &  $\tilde \mcI$ & : & the set of intervened nodes $(\tilde I^1,\dots,\tilde I^n)$  \\
           statistical 
           &  $\hat \bZ(\bX)$ & : & generic estimator of $\bZ$ given $\bX$\\
           models &  $\hat \bZ(\bX;h)$ & : & an auxiliary estimator of $\bZ$ given $\bX$ and encoder $h$\\
           &  $\hat \mcG$ & : & estimate of $\mcG$\\
           &  $p, p^{m}, \tilde p^{m}$ &: & pdfs of $\bZ$ in $\mcE^0$, $\mcE^m$, and $\tilde \mcE^m$ \\
           &  $p_{\bX}, p_{\bX}^{m}, \tilde p_{\bX}^{m}$ &: & pdfs of $\bX$ in $\mcE^0$, $\mcE^m$, and $\tilde \mcE^m$ \\     
           &  $\bss, \bss^{m}, \tilde \bss^{m}$ &: & score functions of $\bZ$ in $\mcE^0$, $\mcE^m$, and $\tilde \mcE^m$ \\
           &  $\bss_{\bX}, \bss_{\bX}^{m}, \tilde \bss_{\bX}^{m}$ &: & score functions of $\bX$ in $\mcE^0$, $\mcE^m$, and $\tilde \mcE^m$ \\
           &  $\bss_{\hat \bZ}, \bss_{\hat \bZ}^{m}, \tilde \bss_{\hat \bZ}^{m}$ & : & score functions of $\hat \bZ$ in $\mcE^0$, $\mcE^m$, and $\tilde \mcE^m$ for encoder $h$ \\
            &  $\hat \Pa(i)$ & : & parents of node $i$ in $\hat \mcG$ \\
           matrix 
            & $\bA^{\dag}$ & : & Pseudo-inverse of matrix $\bA$\\
            notations & $\bA_{i}$ &: & row $i$ of matrix $\bA$\\
            &  $\bA_{i,j}$ &: & entry of matrix $\bA$ at row  $i$ and column $j$ \\
           & $\bP_\pi$ & : & Permutation matrix associated with permutation $\pi$ of $[n]$ \\
           &  $\bD, \tilde \bD, \bD_{\rm t}$ & : & True score change matrices \\
           &  $\bD(h), \tilde \bD(h), \bD_{\rm t}(h)$ & : & Score change matrices under encoder $h$ \\
            \multicolumn{3}{c}{}\\
            \bottomrule
        \end{tabular}
        \label{tab:TableOfNotation}
    }
\end{table}

\newpage

\section{Proofs of Score Function Properties and Transformations}\label{sec:proofs-score}

We start by providing the following facts that will be used repeatedly in the proofs.
\begin{proposition}\label{prop:continuity-argument}
    Consider two continuous functions $f,g: \R^n \to \R$ with full support. Then, for any $\alpha > 0$,
    \begin{equation}
        \exists \bz \in \R^n \;\; f(\bz) \neq g(\bz) \quad \iff \quad \E\Big[\big|f(\bZ)-g(\bZ)\big|^{\alpha}\Big] \neq 0 \ .
    \end{equation}
    Specifically, for $\alpha = 1$, we have
    \begin{equation}
        \exists \bz \in \R^n \;\; f(\bz) \neq g(\bz) \quad \iff \quad \E\Big[\big|f(\bZ)-g(\bZ)\big|\Big] \neq 0 \ .
    \end{equation}
\end{proposition}
\proof If there exists $\bz \in \R^n$ such that $f(\bz)\neq g(\bz)$, then $f(\bz)-g(\bz)$ is nonzero over a nonzero-measure set due to continuity. Then, $\E\big[|f(\bZ)-g(\bZ)|^{\alpha}\big] \neq 0$ since $p$ (pdf of $\bZ$) has full support. On the other direction, if $f(\bz)=g(\bz)$ for all $\bz \in \R^n$, then $\E\big[|f(\bZ)-g(\bZ)|^{\alpha}\big]=0$. This means that $\E\big[|f(\bZ)-g(\bZ)|^{\alpha}\big]\neq 0$ implies that there exists $\bz \in \R^n$ such that $f(\bz) \neq g(\bz)$.

\subsection{Proof of Lemma~\ref{lm:parent_change_comprehensive}}\label{sec:proof-parent-change}

Our score-based methodology builds on the changes in score functions under interventions. For proving \Cref{lm:parent_change_comprehensive}, we start by showing that $ \E\Big[\big|\bss(\bZ) - \bss^{m}(\bZ)\big|_i\Big] \neq 0 \implies i \in \Paplus(I^m) $, which holds true regardless of the causal model and the intervention type.

\paragraph{Proof of $ \E\Big[\big|\bss(\bZ) - \bss^{m}(\bZ)\big|_i \Big] \neq 0 \implies i \in \Paplus(I^m) $:} Let $\ell$ denote the node intervened in $\mcE^m$, i.e.,  $I^m = \ell $. Recalling \eqref{eq:s_z_decompose_obs} and \eqref{eq:s_z_decompose_int}, the latent scores $\bss(\bz)$ and $\bss^{m}(\bz)$ are decomposed as 
\begin{align}
    \bss(\bz) &= \nabla_{\bz} \log p_{\ell}(z_{\ell} \med \bz_{\Pa(\ell)}) + \sum_{i \neq \ell} \nabla_{\bz} \log p_i(z_i \med \bz_{\Pa(i)}) \ , \label{eq:s_z_decompose_obs_appendix} \\
    \mbox{and} \quad \bss^{m}(\bz) &= \nabla_{\bz} \log q_{\ell}(z_{\ell} \med \bz_{\Pa(\ell)}) + \sum_{i \neq \ell} \nabla_{\bz} \log p_i(z_i \med \bz_{\Pa(i)}) \ . \label{eq:s_z_decompose_int_appendix}
\end{align}
Hence, $\bss(\bz)$ and $\bss^{m}(\bz)$ differ in only the causal mechanism of node $\ell$. Next, we check the derivatives of $p_{\ell}(z_{\ell} \med \bz_{\Pa(\ell)})$ and $q_{\ell}(z_{\ell} \med \bz_{\Pa(\ell)})$ in their $i$-th coordinates. Note that these two depend on $\bZ$ only through $\{Z_j : \; j \in \Paplus(\ell) \}$. Therefore, if $i \notin \Paplus(\ell)$,
\begin{align}
    \pdv{z_i} \log p_{\ell}(z_{\ell} \med \bz_{\Pa(\ell)}) = \pdv{z_i} \log q_{\ell}(z_{\ell} \med \bz_{\Pa(\ell)}) = 0 \ , \label{eq:s_z_zero_indices}
\end{align}
which indicates that if $i \notin \Paplus(\ell)$, then $[\bss(\bz)]_i = [\bss^{m}(\bz)]_i$ for all $\bz$. This, equivalently, means that if $\E\big[|\bss(\bZ) - \bss^{m}(\bZ)|_i \big] \neq 0$, then $i \in \Paplus({\ell})$.  \endproof

For the reverse direction, we will use the following intermediate result which formalizes the weakest possible requirement for a meaningful intervention and shows that it is a property of (i) hard interventions under any causal model, and (ii) additive noise model under either soft or hard interventions.
\begin{lemma}[Interventional Regularity]\label{lm:pq-parent-dependence-new}
    Causal mechanisms $p_i$ and $q_i$ of node $i$ are said to satisfy \emph{interventional regularity} if
    \begin{align}\label{eq:int-regularity}
       \exists \bz \in \R^n \;\; \mbox{such that} \quad \frac{\partial }{\partial z_k}\frac{q_i(z_i\mid \bz_{\Pa(i)})}{p_i(z_i\mid \bz_{\Pa(i)})} \neq 0 \ , \qquad \forall k\in \Pa(i)\ . 
    \end{align}
    Then, $p_i$ and $q_i$ satisfy interventional regularity if at least one of the following conditions is true:
    \begin{enumerate}[leftmargin=2em,itemsep=-0.2em]
        \item The intervention is hard, i.e., $q_i(z_i \med \bz_{\Pa(i)})=q_i(z_i)$.
        \item The causal model is an additive noise model in which the pdfs of the noise variables are analytic.
    \end{enumerate}
\end{lemma}
\proof See \Cref{proof:pq-parent-dependence-new}.

\paragraph{Case (i) and Case (ii).} $ \E\Big[\big|\bss(\bZ) - \bss^{m}(\bZ)\big|_i\Big] \neq 0 \implies i \in \Paplus(I^m)$ is already shown above. For the reverse direction, we give the proof for soft interventions on additive noise models, Case (ii). We will use interventional regularity since \Cref{lm:pq-parent-dependence-new} shows that it is satisfied for additive noise models. The proof for hard interventions, Case (i), follows from similar arguments since interventional regularity is also satisfied for hard interventions by \Cref{lm:pq-parent-dependence-new}.

\paragraph{Proof of $i \in \Paplus(I^m) \implies \E\Big[\big|\bss(\bZ)-\bss^{m}(\bZ)\big|_i\Big] \neq 0 $:} Note that the two score functions $s$ and $s^m$ are equal in their coordinate $i\in \Paplus(\ell)$ only if
\begin{align}
    0 &= \frac{\partial \log q_{\ell}(z_{\ell}  \mid \bz_{\Pa(\ell)})}{\partial z_i} -
    \frac{\partial \log p_{\ell}(z_{\ell}  \mid \bz_{\Pa(\ell)})}{\partial z_i} = \frac{\partial}{\partial z_i} \log \frac{q_{\ell}(z_{\ell} \mid \bz_{\Pa(i)})}{p_{\ell}(z_{\ell} \mid \bz_{\Pa(\ell)})} \ . \label{eq:derivate-pq-parent-step}
\end{align} 
However, \eqref{eq:derivate-pq-parent-step} contradicts with interventional regularity. Therefore, if $i \in \Paplus(\ell)$, $[\bss(\bz)_i]$ and $[\bss^{m}(\bz)]_i$ are not identical and by \Cref{prop:continuity-argument}, $\E\big[|\bss(\bZ)-\bss^{m}(\bZ)|_i \big] \neq 0$.

\paragraph{Case (iii) Coupled environments.} Suppose that $I^m=\tilde I^m= \ell$. Following \eqref{eq:pz_m_factorized_hard}, we have
\begin{align}
    \bss^{m}(\bz) &= \nabla_{\bz} \log q_{\ell}(z_{\ell}) +  \sum_{i \neq \ell} \nabla_{\bz} \log p_i(z_i \med \bz_{\Pa(i)}) \ , \label{eq:sz_m_decompose_hard_case_1} \\
    \mbox{and} \quad \tilde \bss^{m}(\bz) &= \nabla_{\bz} \log \tilde q_{\ell}(z_{\ell}) +  \sum_{i \neq \ell} \nabla_{\bz} \log p_i(z_i \med \bz_{\Pa(i)})\ . \label{eq:sz_m_decompose_hard_tilde_case_1}
\end{align}
Then, subtracting \eqref{eq:sz_m_decompose_hard_tilde_case_1} from \eqref{eq:sz_m_decompose_hard_case_1} and looking at $i$-th coordinate, we have
\begin{align}\label{eq:coupled-hard-sz-sz-tilde}
    \big[\bss^{m}(\bz) - \tilde \bss^{m}(\bz)\big]_i &= \dfrac{\partial \log q_{\ell}(z_{\ell})}{\partial z_i} -  \dfrac{\partial \log \tilde q_{\ell}(z_{\ell})}{\partial z_i} \ .
\end{align}
If $i\neq \ell$, the right-hand side is zero and we have $\big[\bss^{m}(\bz) - \tilde \bss^{m}(\bz)\big]_i=0$ for all $\bz$. On the other hand, if $i=\ell$, since $q_{\ell}(z_{\ell})$ and $\tilde q_{\ell}(z_{\ell})$ are distinct, there exists $\bz \in \R^n$ such that $q_{\ell}(z_{\ell})\neq \tilde q_{\ell}(z_{\ell})$. Subsequently, by \Cref{prop:continuity-argument}, we have $\E\big[|\bss^{m}(\bZ)-\tilde \bss^{m}(\bZ)|_i\big] \neq 0$.

\paragraph{Case (iv) Uncoupled environments.} 
Suppose that $I^m = \ell$ and $\tilde I^m= j$, and $\ell \neq j$. Following \eqref{eq:pz_m_factorized_hard}, we have 
\begin{align}
    \bss^{m}(\bz) &= \nabla_{\bz} \log q_{\ell}(z_{\ell}) +  \nabla_{\bz} \log p_j(z_j \med \bz_{\Pa(j)}) + \sum_{k \in [n] \setminus \{\ell,j\}} \nabla_{\bz} \log p_k(z_k \med \bz_{\Pa(k}) \ , \label{eq:sz_m_decompose_hard_proof} \\
    \mbox{and} \quad \tilde \bss^{m}(\bz) &= \nabla_{\bz} \log q_j(z_j) +  \nabla_{\bz} \log p_{\ell}(z_{\ell} \med \bz_{\Pa(\ell)}) + \sum_{k \in [n] \setminus \{\ell,j\}} \nabla_{\bz} \log p_k(z_k \med \bz_{\Pa(k)})\ . \label{eq:sz_m_decompose_hard_tilde_proof}
\end{align}
Then, subtracting \eqref{eq:sz_m_decompose_hard_tilde_proof} from \eqref{eq:sz_m_decompose_hard_proof} we have
\begin{align}
    \bss^{m}(\bz) - \tilde \bss^{m}(\bz) &= \;  \nabla_{\bz} \log q_{\ell}(z_{\ell}) +  \nabla_{\bz} \log p_j(z_j \med \bz_{\Pa(j)})  - \nabla_{\bz} \log q_j(z_j) -  \nabla_{\bz} \log p_{\ell}(z_{\ell} \med \bz_{\Pa(\ell)}) \ . \label{eq:score-difference-proof-1} 
\end{align}
Scrutinizing the $i$-th coordinate, we have
\begin{align}
    \big[\bss^{m}(\bz) - \tilde \bss^{m}(\bz)\big]_i &= \dfrac{\partial \log q_{\ell}(z_{\ell})}{\partial z_i} + \dfrac{\partial \log p_j(z_j \med \bz_{\Pa(j)}) }{\partial z_i} - \dfrac{\partial \log q_j(z_j)}{\partial z_i} - \dfrac{\partial \log p_{\ell}(z_{\ell} \med \bz_{\Pa(\ell)})}{\partial z_i}  \ . \label{eq:score-difference-proof-k}
\end{align}

\paragraph{Proof of $\E\Big[\big|\bss^{m}(\bZ) - \tilde \bss^{m}(\bZ)\big|_i\Big] \neq 0\ \implies i\in \Paplus(\ell,j)$:} Suppose that $i \notin \Paplus({\ell},j)$. Then, none of the terms in the RHS of \eqref{eq:score-difference-proof-k} is a function of $z_i$. Therefore, all the terms in the RHS of \eqref{eq:score-difference-proof-k} are zero, and we have $\big[\bss^{m}(\bz) - \tilde \bss^{m}(\bz)\big]_i = 0$ for all $\bz$. By \Cref{prop:continuity-argument}, $\E\big[|\bss^{m}(\bZ)-\tilde \bss^{m}(\bZ)|_i\big]=0$. This, equivalently, means that if $\E\big[|\bss^{m}(\bZ)-\tilde \bss^{m}(\bZ)|_i\big] \neq 0$, then $i \in \Paplus({\ell},j)$. 

\paragraph{Proof of $\E\Big[\big|\bss^{m}(\bZ)-\tilde \bss^{m}(\bZ)\big|_i\Big] \neq 0 \impliedby i\in \Paplus(\ell,j)$:} We prove it by contradiction. Assume that $\big[\bss^{m}(\bz) - \tilde \bss^{m}(\bz)\big]_i = 0$ for all $\bz$. Without loss of generality, let $\ell \notin \Paplus(j)$. 
\begin{enumerate}[leftmargin=*]
    \item \textbf{If $i=\ell$.} In this case, \eqref{eq:score-difference-proof-k} is simplified to
    \begin{align}
         0 = \big[\bss^{m}(\bz) - \tilde \bss^{m}(\bz)\big]_{\ell} = \dfrac{\partial \log q_{\ell}(z_{\ell})}{\partial z_{\ell}} - \dfrac{\partial \log p_{\ell}(z_{\ell} \med \bz_{\Pa(\ell)})}{\partial z_{\ell}}  \ .  \label{eq:proof-k-equals-i}
    \end{align}
    If $\ell$ is a root node, i.e., $\Pa(\ell)=\emptyset$, \eqref{eq:proof-k-equals-i} implies that $(\log q_{\ell})' (z_{\ell}) = (\log p_{\ell})'(z_{\ell})$ for all $z_{\ell}$. Integrating, we get $p_{\ell}(z_{\ell}) = \alpha q_{\ell}(z_{\ell})$ for some constant $\alpha$. Since both $p_{\ell}$ and $q_{\ell}$ are pdfs, they both integrate to one, implying $\alpha = 1$ and $p_{\ell}(z_{\ell}) = q_{\ell}(z_{\ell})$, which contradicts the premise that observational and interventional mechanisms are distinct. If $\ell$ is not a root node, consider some $k \in \Pa(\ell)$. Then, taking the derivative of \eqref{eq:proof-k-equals-i} with respect to $z_k$, we have
    \begin{align}\label{eq:proof-pi-by-zi-zl}
        0 = \dfrac{\partial^2 \log p_{\ell}(z_{\ell} \med \bz_{\Pa(\ell)})}{\partial z_{\ell} \partial z_k} \ . 
    \end{align}
    Recall the equation $Z_{\ell} = f_{\ell}(\bZ_{\Pa(\ell)}) + N_{\ell}$ for additive noise models specified in \eqref{eq:additive-SCM}. Denote the pdf of the noise term $N_{\ell}$ by $p_N$. Then, the conditional pdf $p_{\ell}(z_{\ell} \med \bz_{\Pa(\ell)})$ is given by $p_{\ell}(z_{\ell} \med \bz_{\Pa(\ell)}) = p_N(z_{\ell} - f_{\ell}(\bz_{\Pa(\ell)}))$. Denoting the score function of $p_N$ by $r_p$,
    \begin{align}\label{eq:score-of-additive-noise}
        r_p(u) \triangleq \dfrac{\rm d}{{\rm d} u} \log p_N(u) \ ,
    \end{align}
    we have
    \begin{align}\label{eq:additive-component-score}
        \dfrac{\partial \log p_{\ell}(z_{\ell} \med \bz_{\Pa(\ell)})}{\partial z_{\ell}} =  \dfrac{\partial \log p_{N}(z_{\ell} - f_{\ell}(\bz_{\Pa(\ell)}))}{\partial z_{\ell}} = r_p(z_{\ell} - f_{\ell}(\bz_{\Pa(\ell)})) \ .
    \end{align}
    Substituting this into \eqref{eq:proof-pi-by-zi-zl}, we obtain
    \begin{align}
        0 &= \dfrac{\partial r_p\big(z_{\ell} - f_{\ell}(\bz_{\Pa(\ell)})\big)}{\partial z_k} = - \dfrac{\partial f_{\ell}(\bz_{\Pa(\ell)})}{\partial z_k} \cdot r_p'\big(z_{\ell}-f_{\ell}(\bz_{\Pa(\ell)})\big) \ , \quad \forall \bz \in \R^n \ . \label{eq:proof-pi-by-zizl-additive}
    \end{align}
    Since $k$ is a parent of $\ell$, there exists a fixed $\bZ_{\Pa(\ell)} = \bz_{\Pa(\ell)}^*$ realization for which $\partial f_{\ell}(\bz_{\Pa(\ell)}^*)/\partial z_k$ is nonzero. Otherwise, $f_{\ell}(\bz_{\Pa(\ell)})$ would not be sensitive to $z_k$ which is contradictory to $k$ being a parent of $\ell$. Note that $Z_{\ell}$ can vary freely after fixing $\bZ_{\Pa(\ell)}$. Therefore, for \eqref{eq:proof-pi-by-zizl-additive} to hold, the derivative of $r_p$ must always be zero. However, the score function of a valid pdf with full support cannot be constant. Therefore, $\big[\bss^{m}(\bz)_i - \tilde \bss^{m}(\bz)\big]_i$ is not always zero, and we have $\E\big[|\bss^{m}(\bZ) - \tilde \bss^{m}(\bZ)|_i\big] \neq 0$.   

    \item \textbf{If $i\neq \ell$.} In this case, \eqref{eq:score-difference-proof-k} is simplified to
    \begin{align}
         0 = \big[\bss^{m}(\bz) - \tilde \bss^{m}(\bz)\big]_i = \dfrac{\partial \log p_j(z_j \med \bz_{\Pa(j)}) }{\partial z_i} - \dfrac{\partial \log q_j(z_j)}{\partial z_i} - \dfrac{\partial \log p_{\ell}(z_{\ell} \med \bz_{\Pa(\ell)})}{\partial z_i}  \ . \label{eq:proof-k-notequals-i-1}
    \end{align}
    We investigate case by case and reach a contradiction for each case. First, suppose that $i \notin \Pa({\ell})$. Then, we have $i \in \Paplus(j)$, and \eqref{eq:proof-k-notequals-i-1} becomes
    \begin{align}
         0 = \big[\bss^{m}(\bz) - \tilde \bss^{m}(\bz)\big]_i = \dfrac{\partial \log p_j(z_j \med \bz_{\Pa(j)}) }{\partial z_i} - \dfrac{\partial \log q_j(z_j)}{\partial z_i}  \ . \label{eq:proof-k-notequals-i-2} 
    \end{align}
    If $i = j$, the impossibility of \eqref{eq:proof-k-notequals-i-2} directly follows from the impossibility of \eqref{eq:proof-k-equals-i}. The remaining case is $i \in \Pa(j)$. In this case, taking the derivative of the right-hand side of \eqref{eq:proof-k-notequals-i-2} with respect to $z_j$, we obtain
    \begin{align}
           0 = \dfrac{\partial^2 \log p_j(z_j \med \bz_{\Pa(j)})}{\partial z_i \partial z_j} \ ,
    \end{align}    
    which is a realization of \eqref{eq:proof-pi-by-zi-zl} for $i \in \Pa(j)$ and $j$ in place of $k \in \Pa(\ell)$ and $\ell$, which we proved to be impossible in $i=\ell$ case. Therefore, $i \notin \Pa(\ell)$ is not viable. Finally, suppose that $i \in \Pa({\ell})$. Then, taking the derivative of the right-hand side of \eqref{eq:proof-k-notequals-i-1} with respect to $z_{\ell}$, we obtain 
    \begin{align}\label{eq:proof-pi-by-zi-zk}
           0 = \dfrac{\partial^2 \log p_{\ell}(z_{\ell} \med \bz_{\Pa(\ell)})}{\partial z_i \partial z_{\ell}} \ ,
    \end{align}
    which is again a realization of \eqref{eq:proof-pi-by-zi-zl} for $k = i$, which we proved to be impossible.
\end{enumerate}
Hence, we showed that $\big[\bss^{m}(\bz) - \tilde \bss^{m}(\bz)\big]_i$ cannot be zero for all $\bz$ values. Then, by \Cref{prop:continuity-argument} we have $\E\big[|\bss^{m}(\bZ) - \tilde \bss^{m}(\bZ)|_i\big] \neq 0$, and the proof is concluded. 

\subsection{Proof of Lemma~\ref{lm:score-difference-transform-general}}\label{sec:score-diff-lm-proof}

Let us recall the setting. Consider random vectors $\bY_1,\bY_2\in\R^r$ and $\bW_1$, $\bW_2 \in \R^s$ that are related through $\bY_1=f(\bW_1)$ and $\bY_2=f(\bW_2)$ such that $r \geq s$, probability measures of $\bW_1,\bW_2$ are absolutely continuous with respect to the $s$-dimensional Lebesgue measure and $f: \R^s \to \R^r$ is an injective and continuously differentiable function.

In this setting, the realizations of $\bW_1$ and $\bY_1$, and that of $\bW_2$ and $\bY_2$, are related through $\by = f(\bw)$. Since $f$ is injective and continuously differentiable, volume element ${\rm d}\bw$ in $\R^{s}$ gets mapped to $\left|\det([J_{f}(\bw)]^{\top} \cdot J_{f}(\bw))\right|^{1/2}\ {\rm d}\bw$ on $\image(f)$. Since $\bW_1$ has density $p_{\bW_1}$ absolutely continuous with respect to the $s$-dimensional Lebesgue measure, using the area formula~\citep{boothby2003introduction}, we can define a density for $\bY_1$, denoted by $p_{\bY_1}$, supported only on manifold $\mcM \triangleq \image(f)$ which is absolutely continuous with respect to the $s$-dimensional Hausdorff measure:
\begin{equation}
    p_{\bY_1}(\by) = p_{\bW_1}(\bw) \cdot \left|\det([J_{f}(\bw)]^{\top} \cdot J_{f}(\bw))\right|^{-1/2} \ , \quad \mbox{where} \quad \by = f(\bw) \ . \label{eq:px-py-via-g-jacobian}
\end{equation}
Densities $p_{\bY_2}$ and $p_{\bW_2}$ of $\bY_2$ and $\bW_2$ are related similarly. Subsequently, score functions of $\{\bW_1,\bW_2\}$ and $\{\bY_1,\bY_2\}$ are specified similarly to \eqref{eq:sz-def} and \eqref{eq:sx-def}, respectively. Denote the Jacobian matrix of $f$ at point $\bw \in \R^{s}$ by $J_{f}(\bw)$, which is an $r \times s$ matrix with entries given by
\begin{equation}
    \big[J_{f}(\bw)\big]_{i,j} = \pdv{\big[f(\bw)\big]_i}{w_j}(x) = \pdv{y_i}{w_j} \ , \quad \forall i \in [r] \, , j \in [s] \ . \label{eq:g-jacobian-defn}
\end{equation}
Next, consider a function $\phi \colon \mcM \to \R$. Since the domain of $\phi$ is a manifold, its differential, denoted by $D\phi$, is defined according to \eqref{eq:Df-defn}. By noting $\by = f(\bw)$, we can also differentiate $\phi$ with respect to $\bw \in \R^{s}$ as \citep[p.~57]{simon2014introduction}
\begin{equation}
    \nabla_{w} \phi(\by) = \nabla_{w} (\phi \circ f)(\bw) = \big[J_{f}(\bw)\big]^{\top} \cdot D \phi(\by) \ . \label{eq:nabla-x-fy}
\end{equation}
Next, given the identities in \eqref{eq:px-py-via-g-jacobian} and \eqref{eq:nabla-x-fy}, we find the relationship between score functions of $\bW_1$ and $\bY_1$ as follows.
\begin{align}
    \bss_{\bW_1}(\bw)
    &= \nabla_{\bw} \log p_{\bW_1}(\bw) \label{eq:score-transform-proof-tmp1} \\
    \overset{\eqref{eq:px-py-via-g-jacobian}}&{=} \nabla_{\bw} \log p_{\bY_1}(\by) + \nabla_{w} \log \left|\det([J_{f}(\bw)]^{\top} \cdot J_{f}(\bw))\right|^{1/2}  \\
    \overset{\eqref{eq:nabla-x-fy}}&{=} \big[J_{f}(\bw)\big]^{\top} \cdot D \log p_{\bY_1}(\by) + \nabla_{\bw} \log \left|\det([J_{f}(\bw)]^{\top} \cdot J_{f}(\bw))\right|^{1/2} \\
    &= \big[J_{f}(\bw)\big]^{\top} \cdot \bss_{\bY_1}(\by) + \nabla_{w} \log \left|\det([J_{f}(\bw)]^{\top} \cdot J_{f}(\bw))\right|^{1/2} \ . \label{eq:sx-sz-for-g}
\end{align}
Following the similar steps that led to \eqref{eq:sx-sz-for-g} for $\bW_2$ and $\bY_2$, we obtain
\begin{align}
    \bss_{\bW_2}(\bw) &= \big[J_{f}(\bw)\big]^{\top} \cdot \bss_{\bY_2}(\by) + \nabla_{\bw} \log \left|\det([J_{f}(\bw)]^{\top} \cdot J_{f}(\bw))\right|^{1/2} \ . \label{eq:sxm-szm-for-g} 
\end{align}
Subtracting \eqref{eq:sxm-szm-for-g} from \eqref{eq:sx-sz-for-g}, we obtain the desired result
\begin{align}
    \bss_{\bW_1}(\bw)-\bss_{\bW_2}(\bw) = \big[J_f(\bw)\big]^{\top} \cdot \big[\bss_{\bY_1}(\by) - \bss_{\bY_2}(\by)\big] \ . \label{eq:sw-from-sy-in-proof}
\end{align}

\paragraph{Proof of the reverse direction.}


Multiplying \eqref{eq:sw-from-sy-in-proof} from left with $\big[[J_f(\bw)]^{\dag}\big]^{\top}$, we obtain
\begin{equation}
    \Big[\big[J_f(\bw)\big]^{\dag}\Big]^{\top} \cdot \big[\bss_{\bW_1}(\bw)-\bss_{\bW_2}(\bw)\big] = \Big[\big[J_f(\bw)\big]^{\dag}\Big]^{\top} \cdot \big[J_f(\bw)\big]^{\top} \cdot \big[\bss_{\bY_1}(\by) - \bss_{\bY_2}(\by)\big] \ . \label{eq:jac-rel-left-multiplied}
\end{equation}
Note that
\begin{equation}
    \Big[\big[J_f(\bw)\big]^{\dag}\Big]^{\top} \cdot \big[J_f(\bw)\big]^{\top} = J_f(\bw) \cdot \big[J_f(\bw)\big]^{\dag} \ . \label{eq:jjw-col-proj-eq}
\end{equation}
By properties of the Moore-Penrose inverse, for any matrix $\bA$, we have $\bA \cdot \bA^{\dag} \cdot \bA = \bA$. This means that $\bA \cdot \bA^{\dag}$ acts as a left identity for vectors in the column space of $\bA$. By definition, $\bss_{\bY_1}$ and $\bss_{\bY_2}$ have values in $T_{\bw} \image(f)$, the tangent space of the image manifold $f$ at point $\bw$. This space is equal to the column space of the matrix $J_f(\bw)$. Therefore, $J_f(\bw) \cdot [J_f(\bw)]^{\dag}$ acts as a left identity for $\bss_{\bY_1}(\by)$ and $\bss_{\bY_2}(\by)$, and we have
\begin{equation}
    J_f(\bw) \cdot \big[J_f(\bw)\big]^{\dag} \cdot \big[\bss_{\bY_1}(\by) - \bss_{\bY_2}(\by)\big] = \bss_{\bY_1}(\by) - \bss_{\bY_2}(\by) \ . \label{eq:sy-in-col-jfw}
\end{equation}
Substituting \eqref{eq:jjw-col-proj-eq} and \eqref{eq:sy-in-col-jfw} into \eqref{eq:jac-rel-left-multiplied} completes the proof.

\paragraph{Proof of \Cref{corollary:score-difference-transform-linear}}
For a given linear transform $\bF$, we have $J_f(\bw)=\bF$, which is independent of $\bw$. Then, in the proof of \Cref{lm:score-difference-transform-general}, \eqref{eq:sx-sz-for-g} reduces to $\bss_{\bW}(\bw) = \bF^{\top} \cdot \bss_{\bY}(\by)$. Finally, we note that the score difference of $\bY$ can be similarly written in terms of the score difference of $\bW$.

\subsection{Proof of Lemma~\ref{lm:pq-parent-dependence-new}}\label{proof:pq-parent-dependence-new}

We will prove that causal mechanisms $p_i$ and $q_i$ satisfy interventional regularity for (i) hard interventions and (ii) additive noise models. To this end, we first define
\begin{align}
     \psi(z_i, \bz_{\Pa(i)}) \triangleq \frac{q_i(z_i\mid \bz_{\Pa(i)})}{p_i(z_i\mid \bz_{\Pa(i)})} \ . \label{eq:h-pq-parent-dependence}
\end{align}
We start by showing that $\psi(z_i, \bz_{\Pa(i)})$ varies with $z_i$. We prove it by contradiction. Assume the contrary, i.e., let $\psi(z_i, \bz_{\Pa(i)}) = \psi(\bz_{\Pa(i)})$. By rearranging \eqref{eq:h-pq-parent-dependence} we have
\begin{align}\label{eq:misc1}
    q_i(z_i \mid \bz_{\Pa(i)}) = \psi(\bz_{\Pa(i)}) \, p_i(z_i \mid \bz_{\Pa(i)}) \ .
\end{align}
Fix a realization of $\bz_{\Pa(i)}^*$ and integrate both sides of \eqref{eq:misc1} with respect to $z_i$. Since both $p_i$ and $q_i$ are pdfs, we have
\begin{align}
    1 &= \int_{\R}  q_i(z_i \mid \bz_{\Pa(i)}^*) \, \dd z_i
    = \int_{\R} \psi(\bz_{\Pa(i)}^*) \, p_i(z_i \mid \bz_{\Pa(i)}^*) \mathrm{d}{z_i} \dd z_i \\
    &= \psi(\bz_{\Pa(i)}^*) \int_{\R} p_i(z_i \mid \bz_{\Pa(i)}^*) \dd z_i \\
    &= \psi(\bz_{\Pa(i)}^*)\ .
\end{align}
This identity implies that $p_i(z_i \mid \bz_{\Pa(i)}^*) = q_i(z_i \mid \bz_{\Pa(i)}^*)$ for any arbitrary realization $\bz_{\Pa(i)}^*$. This contradicts the premise that observational and interventional distributions are distinct. As a result, to check if a model satisfies interventional regularity for node $i$, it suffices to investigate whether the function $\psi$ is not invariant with respect to $z_k$ for $k\in\Pa(i)$. To this end, from \eqref{eq:h-pq-parent-dependence} we know that $\psi(z_i, \bz_{\Pa(i)})$ varies with $z_k$ if and only if
\begin{align}\label{eq:int-regularity-suff}
    \frac{\partial}{\partial z_i} \log \psi(z_i, \bz_{\Pa(i)}) = 
   \frac{\partial q_i(z_i \mid \bz_{\Pa(i)})}{\partial z_i} \cdot  \frac{1}{q_i(z_i \mid \bz_{\Pa(i)})}  -
    \frac{\partial p_i(z_i \mid \bz_{\Pa(i)})}{\partial z_i} \cdot \frac{1}{p_i(z_i \mid \bz_{\Pa(i)})}  \neq 0 \ .
\end{align}
Next, we investigate the sufficient conditions listed in \Cref{lm:pq-parent-dependence-new}.

\subsubsection{Hard Interventions}\label{proof:pq-parent-dependence-hard}
Under hard interventions, note that for any $k\in\Pa(i)$,
\begin{align}
    \frac{\partial}{\partial z_k} p_i(z_i \mid \bz_{\Pa(i)}) \neq 0 \ , 
    \quad \mbox{and} \quad
    \frac{\partial}{\partial z_k} q_i(z_i) = 0\ .
\end{align}
Then, it follows directly from \eqref{eq:int-regularity-suff} that
\begin{align}
    \frac{\partial}{\partial z_k} \log \psi(z_i, \bz_{\Pa(i)})
    &= \underset{= \, 0}{\underbrace{\frac{\partial q_i(z_i)}{\partial z_k}}} \cdot \frac{1}{q_i(z_i)} - \underset{\neq \, 0}{\underbrace{\frac{\partial p_i(z_i \mid \bz_{\Pa(i)})}{\partial z_k}}} \cdot \frac{1}{p_i(z_i \mid \bz_{\Pa(i)})}
    \neq 0\ . 
\end{align}
Thus, hard interventions on any latent causal model satisfy interventional regularity.

\subsubsection{Additive Noise Models}\label{proof:pq-parent-dependence-additive}
The additive noise model for node $i$ is given by
\begin{equation}
    Z_i = f_i(\bZ_{\Pa(i)}) + N_i\ ,  \label{eq:additive-model-intreg-proof-obs}
\end{equation}
as specified in \eqref{eq:additive-SCM}. When node $i$ is soft intervened, $Z_i$ is generated according to
\begin{equation}
    Z_i = \bar f_i(\bZ_{\Pa(i)}) + \bar N_i\ ,  \label{eq:additive-model-intreg-proof-int}
\end{equation}
in which $\bar f_i$ and $\bar N_i$ specify the interventional mechanism for node $i$. Then, denoting the pdfs of $N_i$ and $\bar N_i$ by $p_N$ and $q_N$, respectively, \eqref{eq:additive-model-intreg-proof-obs} and \eqref{eq:additive-model-intreg-proof-int} imply that
\begin{align}
    p_i(z_i \mid \bz_{\Pa(i)}) = p_N\big(z_i - f_i(\bz_{\Pa(i)})\big)\ , \quad \mbox{and} \quad 
    q_i(z_i \mid \bz_{\Pa(i)}) = q_N\big(z_i - \bar f_i(\bz_{\Pa(i)})\big)\ . \label{eq:additive-model-pq}
\end{align}
Denote the score functions associated with $p_N$ and $q_N$ by
\begin{align}\label{eq:def-r-additive}
    r_p(u) \triangleq \frac{{\rm d}}{{\rm d} u} \log p_N(u) = \frac{p_N'(u)}{p_N(u)}\ , \quad \mbox{and} \quad
    r_q(u) \triangleq \frac{{\rm d}}{{\rm d} u} \log q_N(u) = \frac{q_N'(u)}{q_N(u)}\ .
\end{align}
We will prove that \eqref{eq:int-regularity-suff} holds by contradiction. Assume the contrary and let
\begin{align}\label{eq:int-regularity-additive-contrary}
    \frac{\partial q_i(z_i \mid \bz_{\Pa(i)})}{\partial z_k} \cdot \frac{1}{q_i(z_i \mid \bz_{\Pa(i)})}
    =
    \frac{\partial p_i(z_i \mid \bz_{\Pa(i)})}{\partial z_k} \cdot \frac{1}{p_i(z_i \mid \bz_{\Pa(i)})}\ .
\end{align}
From \eqref{eq:additive-model-pq} and \eqref{eq:def-r-additive}, for the numerators in \eqref{eq:int-regularity-additive-contrary} we have,
\begin{align}
    \frac{\partial p_i(z_i \mid \bz_{\Pa(i)})}{\partial z_k}
    &= - \frac{\partial f_i(\bz_{\Pa(i)})}{\partial z_k} \cdot p_N'\big(z_i - f_i(\bz_{\Pa(i)})\big) \ , \\
    \mbox{and} \qquad   \frac{\partial q_i(z_i \mid \bz_{\Pa(i)})}{\partial z_k}
    &= - \frac{\partial \bar f_i(\bz_{\Pa(i)})}{\partial z_k} \cdot q_N'\big(z_i - \bar f_i(\bz_{\Pa(i)})\big) \ .
\end{align}
Hence, the identity in \eqref{eq:int-regularity-additive-contrary} can be written as
\begin{equation}
    \frac{\partial f_i(\bz_{\Pa(i)})}{\partial z_k} \cdot r_p\big(z_i - f_i(\bz_{\Pa(i)})\big) =
    \frac{\partial \bar f_i(\bz_{\Pa(i)})}{\partial z_k} \cdot r_q\big(z_i - \bar f_i(\bz_{\Pa(i)})\big) \ . \label{eq:pre-rearrange-additive-proof}
\end{equation}
Define $n_i$ and $\bar n_i$ as the realizations of $N_i$ and $\bar N_i$ when $Z_i=z_i$ and $\bZ_{\Pa(i)}=\bz_{\Pa(i)}$. By defining $\delta(\bz_{\Pa(i)}) \triangleq f_i(\bz_{\Pa(i)}) - \bar f_i(\bz_{\Pa(i)})$, we have $\bar n_i = n_i + \delta(\bz_{\Pa(i)})$. Then, \eqref{eq:pre-rearrange-additive-proof} can rewritten as
\begin{equation}\label{eq:pre-rearrange-additive-proof2}
    \frac{\partial f_i(\bz_{\Pa(i)})}{\partial z_k} \cdot r_p(n_i) =
    \frac{\partial \bar f_i(\bz_{\Pa(i)})}{\partial z_k} \cdot r_q(n_i + \delta(\bz_{\Pa(i)})) \ .     
\end{equation}
Note that $\frac{\partial f_i(\bz_{\Pa(i)})}{\partial z_k}$ is a nonzero continuous function. Hence, there exists an interval $\Phi\subseteq\R^{|\Pa(i)|}$ over which $\frac{\partial f_i(\bz_{\Pa(i)})}{\partial z_k} \neq 0$. Likewise, $r_p(n_i)$ cannot be constantly zero over all possible intervals $\Omega \subseteq \R$. This is because otherwise, it would have to necessarily be a constant zero function (since it is analytic), which is an invalid score function. Hence, there exists an open interval $\Omega \subseteq \mathbb{R}$ over which $r_p(n_i)$ is nonzero for all $n_i \in \Omega$. Then, we can rearrange \eqref{eq:pre-rearrange-additive-proof2} as
\begin{align}\label{eq:rp-rq-equals-d-fq-d-fp}
    \frac{r_p(n_i)}{r_q(n_i+\delta(\bz_{\Pa(i)}))} &=
    \frac{\frac{\partial \bar f_i(\bz_{\Pa(i)})}{\partial z_k}}{\frac{\partial f_i(\bz_{\Pa(i)})}{\partial z_k}} \ , \quad \forall (n_i,\bz_{\Pa(i)})\in \Omega \times \Phi \ .
\end{align}
Note that the RHS of \eqref{eq:rp-rq-equals-d-fq-d-fp} is not a function of $n_i$. Then, taking the derivative of both sides with respect to $n_i$, we get
\begin{equation}\label{eq:d-log-rp-d-log-rq-with-delta}
    \frac{r_p'(n_i)}{r_p(n_i)} =
    \frac{r_q'(n_i + \delta(\bz_{\Pa(i)}))}{r_q(n_i + \delta(\bz_{\Pa(i)}))} \ .
\end{equation}
In the next step, we show that $\delta$ is not a constant function. We prove this by contradiction. Suppose that $\delta(\bz_{\Pa(i)}) = \delta^*$ is a constant function. Then, the gradients of $f_i$ and $\bar f_i$ are equal. From \eqref{eq:rp-rq-equals-d-fq-d-fp}, this implies that
\begin{equation}
    r_p(n_i) = r_q(n_i + \delta^*) \ , \quad \forall n_i \in \Omega \ .
\end{equation}
Since $r_p(n_i)$ and $r_q(n_i + \delta^*)$ are analytic functions that agree on an open interval of $\R$, they are equal for all $n_i \in \R$. This implies that $p_N(n_i) = \eta \cdot q_N(n_i+\delta^*)$ for some constant $\eta \in \R$. Since $p_N$ and $q_N$ are pdfs, $\eta=1$ is the only choice that maintains $p_N$ and $q_N$ are pdfs. Therefore, $p_N(n_i)=q_N(n_i+\delta^*)$. However, using \eqref{eq:additive-model-pq}, this implies that $p_i(z_i \mid \bz_{\Pa(i)}) = q_i(z_i \mid \bz_{\Pa(i)})$, which contradicts the premise that an intervention changes the causal mechanism of target node $i$. Therefore, $\delta$ is a continuous, non-constant function, and its image over $\bz_{\Pa(i)}\in\Phi$ includes an open interval $\Theta\subseteq\R$. With this result in mind, we return to \eqref{eq:d-log-rp-d-log-rq-with-delta}. Consider a fixed realization $n_i = n_i^*$ and denote the value of the left-hand side (LHS) for $n_i^*$ by $\alpha$. By defining $u \triangleq \delta(\bz_{\Pa(i)})$, we get
\begin{align}
    \alpha &= \frac{r_q'(n_i^* + u)}{r_q(n_i^* + u)} \ , \qquad \forall u\in\Theta \ .
\end{align}
This is only possible if $r_q$ is an exponential function, i.e., $r_q(u) = k_1 \exp(\alpha u)$ over interval $u \in \Theta$. Since $r_q$ is an analytic function, it is, therefore, exponential over entire $\R$. Then, the associated pdf must have the form $q_N(u) = k_2 \exp((k_1/\alpha) \exp(\alpha u))$. However, the integral of this function over the entire domain $\R$ diverges. Hence, it is not a valid pdf, rendering a contradiction. Hence, the additive noise model satisfies interventional regularity.

\section{Proofs of the Results for Linear Transformations}\label{sec:proofs-identifiability-linear}

\subsection{Auxiliary Results}\label{sec:auxiliary-lemmas}
First, we provide a linear algebraic property, which will be used in the proofs.

\begin{lemma}\label{lm:binary-matrices} Consider the latent causal graph $\mcG$, and a matrix $\bL \in \R^{n \times n}$.
    \begin{enumerate}
        \item Let $\bL_{\Pa}$ and $\bL_{\An}$ be binary matrices that denote the parental and ancestral relationships in $\mcG$, respectively, i.e., 
        \begin{align}
            [\bL_{\Pa}]_{i,j} \triangleq \begin{cases}
                1 \ , &\textnormal{if} \quad j \in \Paplus(i)  \\ 
                0 \ , &\textnormal{otherwise}
            \end{cases} \ , \;\; \mbox{and} \quad
            [\bL_{\An}]_{i,j} \triangleq \begin{cases}
                1 \ , &\textnormal{if} \quad j \in \Anplus(i)  \\ 
                0 \ , &\textnormal{otherwise}
            \end{cases} \ . \label{eq:binary-parent}
        \end{align}
        Then, if $\bI_{n \times n} \preccurlyeq \mathds{1}\{\bL\} \preccurlyeq \bL_{\Pa}$, we also have $\bI_{n \times n} \preccurlyeq \mathds{1}\{\bL^{-1}\} \preccurlyeq \bL_{\An}$. 

        \item Let $\bL_{\sur}$ be a binary matrix that denotes the surrounding relationships in $\mcG$, i.e., 
        \begin{align}
            [\bL_{\sur}]_{i,j} &\triangleq \begin{cases}
                1 \ , &\textnormal{if} \quad \Chplus(i) \subseteq \Chplus(j) \ , \\
                0 \ , &\textnormal{otherwise}
            \end{cases} \ . \label{eq:binary-surrounding}
        \end{align}
        Then, if $\bI_{n \times n} \preccurlyeq \mathds{1}\{\bL\} \preccurlyeq \bL_{\sur}$, we also have $\bI_{n \times n} \preccurlyeq \mathds{1}\{\bL^{-1}\} \preccurlyeq \bL_{\sur}$. 
    \end{enumerate}
\end{lemma}
\proof 
First, note that $\bL_{\sur} \preccurlyeq \bL_{\Pa} \preccurlyeq \bL_{\An}$. Hence, we start with considering a generic lower triangular matrix $\bL$ such that $\bI_{n \times n} \preccurlyeq \mathds{1}\{\bL\} \preccurlyeq \bL_{\Pa}$. Matrix $\bL$ can be decomposed as
\begin{equation}
    \bL = \bE \cdot (\bI_{n \times n} + \bLambda) \ ,
\end{equation}
in which $\bE$ is a diagonal matrix and $\bLambda$ is a strictly lower triangular matrix that satisfies $\mathds{1}\{\bLambda\}\preccurlyeq \bL_{\Pa}$. Subsequently, since $\bE$ is a diagonal matrix, we have
\begin{equation}
    \bL^{-1} = (\bI_{n \times n} + \bLambda)^{-1} \cdot \bE^{-1} \ , \quad \mbox{and} \quad \mathds{1}\{\bL^{-1}\} = \mathds{1}\big\{(\bI_{n \times n}+\bLambda)^{-1}\big\} \ .
\end{equation}
Note that $\bLambda$ is a strictly lower triangular $n \times n$ matrix, which implies that $\bLambda^{n}$ is a zero matrix. Therefore, the inverse of $(\bI_{n \times n} + \bLambda)$ can be expanded as
\begin{equation}
    (\bI_{n \times n}+\bLambda)^{-1} = \bI_{n \times n} - \bLambda + \bLambda^{2} - \dots + (-1)^{n-1}\bLambda^{n-1} \ . \label{eq:inverse-expansion}
\end{equation}
If $\big[\bI_{n \times n}+\bLambda\big]^{-1}_{i,j} \neq 0$ for some $i, j \in [n]$ with $j < i$, by \eqref{eq:inverse-expansion} we have $[\bLambda^{k}]_{i,j} \neq 0$ for some $k \in [n-1]$. Expanding matrix $\bLambda^{k}$ yields that $[\bLambda^{k}]_{i,j}$ is equal to the sum of the products with $k$ terms, i.e.,
\begin{equation}
   [\bLambda^{k}]_{i,j} = \sum_{i > a_1 > \dots > a_{k-1} > j} [\bLambda]_{i,a_1} [\bLambda]_{a_1,a_2} \dots [\bLambda]_{a_{k-1},j} \ .
\end{equation}
Therefore, if $[\bLambda^{k}]_{i,j} \neq 0$, there exists a sequence of entries $\big([\bLambda]_{i,a_1},[\bLambda]_{a_1,a_2},\dots,[\bLambda]_{a_{k-1},j}\big)$ in which all terms are nonzero. Next, we prove the two cases as follows.

\paragraph{Case 1:} 
Since $\mathds{1}\{\bLambda\}\preccurlyeq \bL_{\Pa}$, we have $[\bL_{\Pa}]_{i,a_1}=[\bL_{\Pa}]_{a_1,a_2}=\dots=[\bL_{\Pa}]_{a_{k-1},j}=1$. By the definition of $\bL_{\Pa}$, this means that there is a path $j \rightarrow a_{k-1} \rightarrow \dots \rightarrow a_1 \rightarrow i$ in $\mcG$, which implies that $j \in \An(i)$ and $[\bL_{\An}]_{i,j}=1$. 
Therefore, we conclude that
\begin{equation}
    \big[(\bI_{n \times n}+\bLambda)^{-1}\big]_{i,j}\neq 0 \quad \implies \quad [\bL_{\An}]_{i,j}=1 \ .
\end{equation}
Hence, $\bI_{n \times n} \preccurlyeq \mathds{1}\{\bL^{-1}\} \preccurlyeq \bL_{\An}$. 

\paragraph{Case 2:}
Since $\mathds{1}\{\bLambda\}\preccurlyeq \bL_{\sur}$, we have $[\bL_{\sur}]_{i,a_1}=[\bL_{\sur}]_{a_1,a_2}=\dots=[\bL_{\sur}]_{a_{k-1},j}=1$. By the definition of $\bL_{\sur}$, this means that 
\begin{equation}
    \Chplus(i) \subseteq \Ch(a_1) \subset \dots \subset \Ch(a_{k-1}) \subset \Ch(j) \ ,
\end{equation} 
which implies $[\bL_{\sur}]_{i,j}=1$ since $\Chplus(i) \subseteq \Chplus(j)$. Therefore, we conclude that if 
\begin{equation}
    \big[(\bI_{n \times n}+\bLambda)^{-1}\big]_{i,j}\neq 0 \quad \implies \quad [\bL_{\sur}]_{i,j}=1 \ .
\end{equation}
Hence, $\bI_{n \times n} \preccurlyeq \mathds{1}\{\bL^{-1}\} \preccurlyeq \bL_{\sur}$.

\subsection{Proof of Theorem~\ref{th:linear-hard}}\label{proof:linear-hard}

\paragraph{Proof of the scaling consistency.}
First, by \Cref{th:linear-soft}, encoder estimate $\hat \bH$ at the end of Stage~L1 of \Cref{alg:linear} satisfies 
\begin{align}\label{eq:before-unmixing}
    \hat \bH \cdot \bG = \bP_{\mcI} \cdot \bL \ ,
\end{align}
where $\bP_{\mcI}$ is the permutation matrix of the intervention order $(I^1,\dots,I^n)$, $\bL$ has nonzero diagonal entries and satisfies $\bL_{i,j} = 0$ for all $j \notin \Paplus(i)$. We will show that the output of Stage~L3 of \Cref{alg:linear} for hard interventions satisfies that
\begin{align}\label{eq:unmixing_goal_node}
    \mathds{1}\big([\hat \bH \cdot \bG]_m \big) = \be_{I^m}^{\top} \ , \quad \forall m \in [n] \ ,
\end{align}
which will imply scaling consistency as
\begin{equation}
    \hat \bZ(\bX; \hat \bH) = \hat \bH \cdot \bX = \hat \bH \cdot \bG \cdot \bZ  = \bP_{\mcI} \cdot \bC_{\rm s} \cdot \bZ \ ,
\end{equation}
where $\bC_{\rm s}$ is a constant diagonal matrix with nonzero diagonal entries. We prove this as follows. 

First, consider the topological order $\pi$ of $\hat \mcG$. By \Cref{th:linear-soft}, transitive closures of $\hat \mcG$ and $\mcG$ are the same under relabeling of the nodes with permutation $(1,\dots,n) \rightarrow (I^1,\dots,I^n)$. Then, since $(1,\dots,n)$ is assumed to be topological order of $\mcG$, $\rho \triangleq (I^{\pi_1},\dots,I^{\pi_n})$ is also a topological order of~$\mcG$. According to this notation, in the $\pi_t$-th environment, node $\rho_t$ is intervened.

Next, since $\rho_1$ has no parent in $\mcG$, by \Cref{th:linear-soft} we already know that $\hat Z_{\pi_1} = b_1 \cdot Z_{\rho_1}$ for some constant $b_1 \in \R_+$. Consider $m = \pi_k$ step of the algorithm, and assume that for all $t \in [k-1]$, the updated encoder rows satisfy
\begin{equation}\label{eq:unmixing-induction-step}
    \hat Z_{\pi_t} = \hat \bH_{\pi_t} \cdot \bG \cdot \bZ = b_t \cdot Z_{\rho_t} \ , 
\end{equation}
for some constant $b_t \in \R_+$. We will show that the update in step $m = \pi_k$ will ensure that $\hat Z_{\pi_k}$ also satisfies \eqref{eq:unmixing-induction-step}. Before the update, by \Cref{th:linear-soft} we know that $\hat Z_{\pi_k}$ is a linear function of $Z_{\rho_k}$ and $\bZ_{\Pa(\rho_k)}$. Then, using \eqref{eq:unmixing-induction-step}, we have
\begin{equation}
    \hat Z_{\pi_k} = c_0 \cdot Z_{\rho_k} + \bc \cdot \hat \bZ_{\hat \Pa(\pi_k)} \ ,
\end{equation}
for $|\hat \Pa(\pi_k)|$-dimensional vector $\bc$ and nonzero constant $c_0 \in \R_+$. We will use \Cref{fact:hard-indep} in environment $\mcE^{\pi_k}$. For brevity, using $m = \pi_k$ and letting $\bV \triangleq \hat \bZ^{m}_{\hat \Pa(\pi_k)}$, 
\begin{equation}
    \hat Z^{m}_{m} = c_0 \cdot Z^{m}_{\rho_k} + \bc \cdot \bV \ .
\end{equation}
Since node $\rho_k$ is intervened in environment $\mcE^m$, by \Cref{fact:hard-indep}, we know that $Z^{m}_{\rho_k} \ci \bV$, which implies that
\begin{equation}
    \Cov(Z^{m}_{\rho_k}, \bV) = \boldsymbol{0}_{1 \times |\bV|} \ .
\end{equation}
Consider a $|\bV|$-dimensional row vector $\bu$ and let 
\begin{equation}
    Y = \big(\hat \bH_{m} - \bu \cdot \hat \bH_{\hat \Pa(m)}\big) \cdot \bX^{m} = \hat Z^{m}_{m} - \bu \cdot \bV \ .
\end{equation}
Note that $\bu = \bc$ would yield that $Y = c_0 \cdot Z^{m}_{\rho_k}$, and subsequently $Y \ci \bV$. On the other hand, $Y \ci \bV$ implies that 
\begin{equation}
    \boldsymbol{0}_{1 \times |\bV|} = \Cov(Y, \bV) = \Cov(\hat Z^{m}_{m} - \bu \cdot \bV, \bV) \ ,
\end{equation}
which has a unique solution
\begin{equation}
    \bu = \Cov(\hat Z^{m}_{m}, \bV) \cdot \big(\Cov(\bV)\big)^{-1} \ .
\end{equation}
i.e., linear minimum mean square error (LMMSE) estimator. Finally, $\Cov(\bV)$ is invertible since the causal relationships among the entries in $\hat Z^{m}_{\hat \Pa(\pi_k)}$ are not deterministic. Using this row vector $\bu$, we update
\begin{equation}
    \hat \bH_m \gets \bH_m - \bu \cdot \bH_{\hat \Pa(\pi_k)} \ ,
\end{equation}
and achieve $\hat Z_m = c_0 \cdot Z_{\rho_k}$. Therefore, \eqref{eq:unmixing-induction-step} holds for $t= k$ as well, and by induction we obtain
\begin{equation}\label{eq:z-hat-z-unmixed}
    \hat \bZ = \hat \bH \cdot \bG \cdot \bZ = \bP_{\mcI} \cdot \bC_{\rm s} \cdot \bZ \ ,
\end{equation}
for a diagonal matrix $\bC_{\rm s}$ with nonzero diagonal entries. \endproof

\paragraph{Proof of the perfect DAG recovery.} 
We show that the graph construction process in \eqref{eq:linear-hard-parent-set} achieves perfect DAG recovery as follows. First, using \Cref{lm:score-difference-transform-general} and \eqref{eq:z-hat-z-unmixed}, we have
\begin{align}
    \bsd_{\hat \bZ}^{m}(\hat \bZ; \hat \bH) = (\bP_{\mcI} \cdot \bC_{\rm s})^{-\top} \cdot \bsd_{\bZ}^{m}(\bZ) = \bP_{\mcI} \cdot \bC_{\rm s}^{-\top} \cdot \bsd_{\bZ}^{m}(\bZ) \ .
\end{align}
Subsequently, using the fact that $\bC_{\rm s}$ is a diagonal matrix, for all $i, m \in [n]$ we obtain
\begin{align}
    \big[\bsd_{\hat \bZ}^{m}(\hat \bZ; \hat \bH)\big]_i \neq 0 \; \; \iff \;\; \big[\bsd_{\bZ}^{m}(\bZ)\big]_{I^i} \neq 0 \ .
\end{align}
Then, \Cref{lm:parent_change_comprehensive} and \eqref{eq:linear-hard-parent-set}, we have
\begin{align}
    I^i \in \Paplus(I^m) \; &\iff \; \E\Big[\big|\bsd_{\bZ}^{m}(\bZ)\big|\Big]_{I^i} \neq 0 \; \iff \; \E\Big[\big|\bsd_{\hat \bZ}^{m}(\hat \bZ; \hat \bH)\big|_i \Big] \; \iff \; i \in \hat \Paplus(m) \ .
\end{align}
This concludes the proof that $\hat \mcG$ and $\mcG$ are related through a graph isomorphism by permutation $\mcI$, which denotes the intervention order. \endproof

\subsection{Proof of Theorem~\ref{th:linear-soft-full-rank}}\label{proof:linear-soft-full-rank}

\paragraph{Latent graph estimation.}
We continue from the intuitions provided in \Cref{sec:linear-full-rank} and prove the perfect recovery of the latent graph as follows. First, recall that output $\hat \mcG$ of \Cref{alg:linear} for soft interventions satisfies that transitive closures of $\hat \mcG$ and $\mcG$ are the same under permutation $(I^1,\dots,I^n)$. Since $(1,\dots,n)$ is assumed to be a topological order of $\mcG$,  $\rho \triangleq (I^{\pi_1},\dots,I^{\pi_n})$ is also a topological order of $\mcG$. Then, for any $k = \pi_u$ and $t = \pi_v \in \{\pi_1,\dots,\pi_{u-1}\}$, \eqref{eq:dx-col-intersection-dim} becomes
\begin{equation}\label{eq:full-rank-parent-dim}
    \dim\big( \colspace(\bR_{\bX}^{t}) \cap \colspace(\bR_{\bX}^{k})\big) = |\Pa(\rho_u) \cap \Pa(\rho_v)| + \mathds{1}\big\{\rho_v \in \Pa(\rho_u) \big\}
    \ . 
\end{equation}
By induction, we will prove that the parent sets generated by \Cref{alg:linear-full-rank} satisfy
\begin{equation}\label{eq:full-rank-parent-condition}
    \pi_u \in \hat \Pa(\pi_v) \; \iff \; \rho_u \in \Pa(\rho_v) \ .
\end{equation}
\textbf{At the base case}, consider $k=\pi_2$. The only possible parent is $t=\pi_1$. Since $\rho_1$ is a root node in $\mcG$, \eqref{eq:full-rank-parent-dim} implies
\begin{equation}
     \dim\big( \colspace(\bR_{\bX}^{\pi_1}) \cap \colspace(\bR_{\bX}^{\pi_2})\big) = \mathds{1}\{ \rho_1 \in \Pa(\rho_2) \} \ .
\end{equation}
Since $\hat \Pa(\pi_1) = \hat \Pa(\pi_2) = \emptyset$ at this step, the algorithm adds $\pi_1$ to $\hat \Pa(\pi_2)$ if and only if $\rho_1 \in \Pa(\rho_2)$, and \eqref{eq:full-rank-parent-condition} is satisfied for $k=\pi_2$. Next, \textbf{as the induction step}, assume that the algorithm output $\hat \Pa(k)$ satisfies \eqref{eq:full-rank-parent-condition} for $k \in \{\pi_1,\dots,\pi_{u-1}\}$. We prove that $k = \pi_u$ also satisfies the condition by induction again. At the base case, consider $t = \pi_1$. Similar to the previous base case, \eqref{eq:full-rank-parent-dim} implies
\begin{equation}
     \dim\big( \colspace(\bR_{\bX}^{\pi_1}) \cap \colspace(\bR_{\bX}^{\pi_u})\big) = \mathds{1}\{ \rho_1 \in \Pa(\rho_u) \} \ ,
\end{equation}
and since $\hat \Pa(\pi_u)$ does not contain $\pi_1$ yet, the algorithm adds $\pi_1$ to $\hat \Pa(\pi_u)$ if and only if $\rho_1 \in \Pa(\rho_u)$. Next, as the induction step, assume that for $\pi_j \in \{\pi_1,\dots,\pi_{v-1}\}$ where $v < u$, the algorithm correctly identified the parents, i.e., $\pi_j \in \hat \Pa(\pi_u)$ if and only if $\rho_j \in \Pa(\rho_u)$. Consider $k = \pi_u$ and $t = \pi_v$, for which \eqref{eq:full-rank-parent-dim} implies
\begin{equation}
     \dim\big( \colspace(\bR_{\bX}^{\pi_v}) \cap \colspace(\bR_{\bX}^{\pi_u})\big) = |\Pa(\rho_u) \cap \Pa(\rho_v)| + \mathds{1}\{ \rho_v \in \Pa(\rho_u)\} \ .
\end{equation}
By the induction hypothesis, we know that $|\hat \Pa(\pi_u) \cap \hat \Pa(\pi_v)| = |\Pa(\rho_u) \cap \Pa(\rho_v)|$. Therefore, the algorithm adds $\pi_v$ to $\hat \Pa(\pi_v)$ if and only if $\rho_v \in \Pa(\rho_u)$. Therefore, by the inner induction, $\hat \Pa(k)$ satisfies \eqref{eq:full-rank-parent-condition}. Then, by the outer induction, \eqref{eq:full-rank-parent-condition} is satisfied for all parent set estimates, which concludes the proof of perfect graph recovery. \endproof

\paragraph{Markov property of consistency up to mixing with surrounding parents.}

Let $\hat \bZ$ denote $\hat \bZ(\bX; \hat \bH)$. We further investigate the properties of consistency up to mixing with surrounding parents to prove that $\hat \bZ$ is Markov with respect to $\hat \mcG$. Using the results of the encoder estimation step, we have
\begin{equation}\label{eq:hat-z-surrounding}
    \hat \bZ = \bP_{\mcI} \cdot \bC_{\rm sur} \cdot \bZ \ ,
\end{equation}
where $\bC_{\rm sur}$ has nonzero diagonal entries and satisfies $[\bC_{\rm sur}]_{i,j} = 0$ for all $j \notin \overline{\sur}(i)$. Recall the following SCM specified in \eqref{eq:general-SCM}
\begin{equation}
    Z_i = f_i(\bZ_{\Pa(i)},N_i) \ .
\end{equation}
Let $\tau$ be the permutation that maps $\{1,\dots,n\}$ to $\mcI$, i.e., $I^{\tau_i}=i$ for all $i\in[n]$. Since $\sur(i) \subseteq \Pa(i)$, \eqref{eq:hat-z-surrounding} implies
\begin{align}
    \hat Z_{\tau_i} &= [\bC_{\sur}]_{i,i} \cdot Z_i + \sum_{j \in \sur(i)}  [\bC_{\sur}]_{i,j} \cdot Z_j \label{eq:rho-next-3} \\
    &= [\bC_{\sur}]_{i,i} \cdot f_i(\bZ_{\Pa(i)},N_i) + \sum_{j \in \sur(i)} [\bC_{\sur}]_{i,j} \cdot Z_j \ , \label{eq:rho-next-33}
\end{align}
which is a function of $\bZ_{\Pa(i)}$ and $N_i$. To prove Markov property, we need to specify $Z_{\tau_i}$ in terms of $\hat \bZ_{\hat \Pa(\tau_i)}$. To this end, it suffices to show that for any $k \in \Pa(i)$, $Z_k$ is a function of $\hat \bZ_{\hat \Pa}(\tau_i)$. Using \eqref{eq:hat-z-surrounding}, we have
\begin{equation}
    \bZ = \bC_{\sur}^{-1} \cdot \bP_{\tau} \cdot \hat \bZ \ . \label{eq:z-from-hat-z-pre}
\end{equation}
By \Cref{lm:binary-matrices}, we know that $[\bC_{\sur}^{-1}]_{i,j}=0$ if $j \notin \sur(i)$ for distinct $i$ and $j$. Subsequently, as counterpart of \eqref{eq:rho-next-3}, we have
\begin{equation}\label{eq:rho-next-4}
    Z_k = [\bC_{\sur}^{-1}]_{k,k} \cdot \hat Z_{\tau_k} + \sum_{j \in \sur(k)} [\bC_{\sur}^{-1}]_{k,j} \cdot \hat Z_{\tau_j} \ .
\end{equation}
Since $k \in \Pa(i)$, we have $\tau_k \in \Pa(\tau_i)$ and $\hat Z_{\tau_k}$ is in $\hat \bZ_{\hat \Pa(\tau_i)}$. Note that if $j \in \sur(k)$, $j$ is also in $\Pa(i)$. Therefore, every term in the RHS of \eqref{eq:rho-next-4} belongs to $\hat \bZ_{\hat \Pa(\tau_i)}$ and $Z_k$ is a function of $\hat \bZ_{\hat \Pa(\tau_i)}$. Then, using \eqref{eq:rho-next-33}, we know that $\hat Z_{\tau_i}$ is a function of only $\hat \bZ_{\hat \Pa(\tau_i)}$ and $N_i$, which concludes the proof that $\hat \bZ$ is Markov with respect to $\hat \mcG$.

\section{Proofs of the Results for General Transformations}\label{sec:proofs-identifiability-general}

The proof of \Cref{th:general-coupled} is already given in \Cref{sec:general-coupled}. In this section, we first prove \Cref{th:general-coupled-no-observational}, that is, identifiability for coupled interventions without observational environment. Then, we prove the results in \Cref{sec:general-uncoupled} related to the uncoupled interventions.

\subsection{Proof of Theorem~\ref{th:general-coupled-no-observational}}\label{proof:general-coupled-no-observational}

We will show that if $p(\bZ)$ is adjacency-faithful to $\mcG$ and the latent causal model is an additive noise model, then we can recover $\mcG$ without having access to observational environment $\mcE^0$. 
Let $\tau$ be the permutation that maps $\{1,\dots,n\}$ to $\mcI$, i.e., $I^{\tau_i}=i$ for all $i\in[n]$ and $\bP_{\tau}$ to denote the permutation matrix that corresponds to $\tau$, i.e.,
\begin{align}\label{eq:intervention-order}
    [\bP_{\tau}]_{i,m} = \begin{cases}
        1  \ , &m = \rho_i \ , \\
        0 \ , & \text{else} \ .
    \end{cases}
\end{align}
By \Cref{lm:parent_change_comprehensive}(iv), true latent score changes across $\{\mcE^{\tau_i},\tilde \mcE^{\tau_j}\}$ gives us $\Paplus(i,j)$ for $i\neq j$. First, we use the perfect latent recovery result to show that \Cref{lm:parent_change_comprehensive}(iv) also applies to estimated latent score changes. Denote $f = \hat h \circ g$. Using score transform between $\bZ$ and $\hat \bZ$ in $\eqref{eq:score-difference-z-zhat}$, and recalling $\mathds{1}\{J_{f}^{-1}\}=\bP_{\tau}$, we have
\begin{align}
    \big[\bss_{\hat \bZ}^{\tau_i}(\hat \bz;\hat h) - \tilde \bss_{\hat \bZ}^{\tau_j}(\hat \bz; \hat h)\big]_{\tau_k}
    &= \big[J_{f}^{-\top}(\bz)\big]_{\tau_k} \cdot \big[\bss^{\tau_i}(\bz) - \tilde \bss^{\tau_j}(\bz)\big] \\
    &=  \big[J_{f}^{-\top}(\bz)\big]_{\tau_k,k} \cdot \big[\bss^{\tau_i}(\bz) - \tilde \bss^{\tau_j}(\bz)\big]_k \ . 
\end{align}
Using $\big[J_{f}^{-\top}(\bz)\big]_{\tau_k,k} \neq 0$ for all $\bz \in \R^n$, we have
\begin{align}\label{eq:graph-iso-step-1}
    \E\Big[\big|\bss_{\hat \bZ}^{\tau_i}(\hat \bZ;\hat h) - \tilde \bss_{\hat \bZ}^{\tau_j}(\hat \bZ; \hat h)\big|_{\tau_k}\Big] \neq 0 \quad \iff \quad \E\Big[\big|\bss^{\tau_i}(\bZ) - \tilde \bss^{\tau_j}(\bZ)\big|_{\tau_k}\Big] \neq 0 \ .
\end{align}
Hence, by \Cref{lm:parent_change_comprehensive}(iv),
\begin{align}\label{eq:graph-iso-step-2}
     \E\Big[\big|\bss_{\hat \bZ}^{\tau_i}(\hat \bZ;\hat h) - \tilde \bss_{\hat \bZ}^{\tau_j}(\hat \bZ; \hat h)\big|_{\tau_k}\Big] \neq 0 \quad \iff \quad k \in \Paplus(i,j)\ .
\end{align}
Let us denote the graph $\mcG_{\tau}$ that is related to $\mcG$ by permutation $\tau$, i.e., $i\in\Pa(j)$ if and only if $\tau_i \in \Pa_{\tau}(\tau_j)$ for which $\Pa_{\tau}(\tau_j)$ denotes the parents of node $\tau_j$ in $\mcG_{\tau}$. Using \eqref{eq:graph-iso-step-2}, we have
\begin{align}
   \E\Big[\big|\bss_{\hat \bZ}^{\tau_i}(\hat \bZ;\hat h) - \tilde \bss_{\hat \bZ}^{\tau_j}(\hat \bZ; \hat h)\big|_{\tau_k}\Big] \neq 0 \quad \iff \quad \tau_k \in \Paplus_{\tau}(\tau_i,\tau_j)\ .
\end{align}
In the rest of the proof, we will show how to obtain $\{\Pa_{\tau}(i) : i \in [n]\}$ using $\{\Paplus_{\tau}(i,j) : i,j \in [n], \; i \neq j\}$. Since $\mcG_{\tau}$ is a graph isomorphism of $\mcG$, this problem is equivalent to obtaining $\{\Pa(i) : i \in [n]\}$ using $\{\Paplus(i,j) : i,j \in [n], \; i \neq j\}$. Note that $\hat Z_i$ (which corresponds to node $i$ in $\mcG_{\tau}$) is intervened in environments $\mcE^i$ and $\tilde \mcE^i$. We denote the set of root nodes by
\begin{equation}
    \mcK \triangleq \{i\in[n] : \Pa(i) = \emptyset \} \ ,
\end{equation}
and also define
\begin{equation}
    \mcB_i \triangleq \, \bigcap_{j \neq i} \Paplus(i,j) \ , \quad \forall i \in [n] \ , \quad \mbox{and} \quad \mcB \triangleq \{i : |\mcB_i|=1\} \ .
\end{equation}
Note that $\Paplus(i) \subseteq \mcB_i$. Hence, $|\mcB_i|=1$ implies that $i$ is a root node. We investigate the graph recovery in three cases.
\begin{enumerate}[leftmargin=*]
    \item $|\mcB|\geq 3$: For any node $i\in[n]$, we have
    \begin{align}
    \Paplus(i) \subseteq \mcB_i \subseteq \bigcap_{j \in \mcK \setminus \{i\}} \Paplus(i,j) = \Paplus(i) \cup \Big\{\bigcap_{j \in \mcK \setminus \{i\}} \{ j \} \Big\} = \Paplus(i) \ .
    \end{align}
    Note that the last equality is due to $\textstyle\bigcap_{j \in \mcK \setminus \{i\}} \{ j \} = \emptyset$ since there are at least two root nodes excluding $i$. Then, $\mcB_i = \Paplus(i)$ for all $i \in [n]$ and we are done.
    
    \item $|\mcB|=2$: The two nodes in $\mcB$ are root nodes. If there were at least three root nodes, we would have at least three nodes in $\mcB$. Hence, the two nodes in $\mcB$ are the only root nodes. Subsequently, every $i \notin \mcB$ is also not in $\mcK$ and we have 
    \begin{align}
        \Paplus(i) \subseteq \mcB_i \subseteq \bigcap_{j \in \mcK} \Paplus(i,j) = \Paplus(i) \cup \Big\{\bigcap_{j \in \mcK}\{j\}\Big\} = \Paplus(i) \ .
    \end{align}
    Hence, $\mcB_i = \Paplus(i)$ for every non-root node $i$ and we already have the two root nodes in $\mcB$, which completes the graph recovery.
    
    \item $|\mcB|\leq 1$: First, consider all $(i,j)$ pairs such that $|\Paplus(i,j)|=2$. For such an $(i,j)$ pair, at least one of the nodes is a root node; otherwise $\Paplus(i,j)$ would contain a third node. Using these pairs, we identify all root nodes as follows. Note that a hard intervention on node $i$ makes $Z_i$ independent of all of its non-descendants, and all conditional independence relations are preserved under componentwise diffeomorphisms such as $f$. Then, using the adjacency-faithfulness assumption, we infer that
    \begin{itemize}
        \item if $\hat Z_i \ci \hat Z_j$ in $\mcE^i$ and $\hat Z_i \ci \hat Z_j$ in $\tilde \mcE^j$, then both $i$ and $j$ are root nodes.
        \item if $\hat Z_i \notci \hat Z_j$ in $\mcE^i$, then $i \rightarrow j$ and $i$ is a root node.
        \item if $\hat Z_i \notci \hat Z_j$ in $\tilde \mcE^j$, then $j \rightarrow i$ and $j$ is a root node.
    \end{itemize}
    This implies that we can determine whether $i$ and $j$ nodes are root nodes by using at most two independence tests. Hence, we identify all root nodes by using at most $n$ independence tests. We also know that there are at most two root nodes. If we have two root nodes, then $\mcB_i = \Paplus(i)$ for all non-root nodes, and the graph is recovered. If we have only one root node $i$, then for any $j \neq i$ we have
    \begin{align}
        \Paplus(j) \subseteq \mcB_j \subseteq \Paplus(i,j) = \Paplus(j) \cup \{i\} \ .
    \end{align}
    Finally, if $\hat Z_j \ci \hat Z_i \mid \{\hat Z_\ell : \ell \in \mcB_j \setminus \{i\}\}$ in $\tilde \mcE^j$, we have $i \notin \Paplus(j)$ due to adjacency-faithfulness. Otherwise, we conclude that $i \in \Paplus(j)$. Hence, an additional $(n-1)$ conditional independence tests ensure the recovery of all $\Paplus(j)$ sets, and the graph recovery is complete. 
\end{enumerate}

\subsection{Proof of Lemma~\ref{lm:uncoupled-existence}}\label{proof:uncoupled-existence}

Let us start by scrutinizing the constraints. For the true encoder $g^{-1}$, \Cref{lm:parent_change_comprehensive} gives us
\begin{align}
    [\bD(g^{-1})]_{i,m} \neq 0 \;\; &\iff \;\; i \in \Paplus(I^m) \ , \\
    \mbox{and} \quad [\tilde \bD(g^{-1})]_{i,m} \neq 0 \;\; &\iff \;\; i \in \Paplus(\tilde I^m) \ .
\end{align}
Therefore, using $\tilde \mcI = \sigma \circ \mcI$, we have $\mathds{1}\{\tilde \bD(g^{-1})\} = \mathds{1}\{\bD(g^{-1})\} \cdot \bP_{\sigma}$. Note that acyclicity of graph $\mcG$ also implies that $\mathds{1}\{\bD(g^{-1})\} \odot \mathds{1}\{\bD(g^{-1})^{\top}\} = \bI_{n \times n}$.

Next, note that $\bD_{\rm t}(g^{-1};\sigma)$ is equal to the true score change matrix $\bD_{\rm t}(g^{-1})$ defined for the coupled interventions, which satisfies $\mathds{1}\{\bD_{\rm t}(g^{-1};\sigma)\} = \bP_{\mcI}^{\top}$. 
Let $h^* = \big[\bD_{\rm t}(g^{-1};\sigma)\big]^{\top} \cdot g^{-1}$. Since $\bD_{\rm t}(g^{-1};\sigma)$ is a scaled permutation matrix, using \eqref{eq:score-multiply-encoder} we obtain
\begin{align}
    \bD_{\rm t}(h^*;\sigma) &= [\bD_{\rm t}(g^{-1};\sigma)]^{-1} \cdot \bD_{\rm t}(g^{-1};\sigma) = \bI_{n \times n} \ , \\
    \bD(h^*) &= [\bD_{\rm t}(g^{-1};\sigma)]^{-1} \cdot \bD(g^{-1}) \\
    \tilde \bD(h^*) &= [\bD_{\rm t}(g^{-1};\sigma)]^{-1} \cdot \tilde \bD(g^{-1}) \ .    
\end{align}
Thus, using $\mathds{1}\{\tilde \bD(g^{-1})\} =  \mathds{1}\{\bD(g^{-1})\} \cdot \bP_{\sigma}$, we have
\begin{align}
    \mathds{1}\{\tilde \bD(h^*)\} = \mathds{1}\{\bD(h^*)\} \cdot \bP_{\sigma} \ ,     
\end{align}
which satisfies the first constraint. Also, using $\mathds{1}\{\bD(h)\}= \bP_{\sigma} \cdot \mathds{1}\{\tilde \bD(h)\}$,
\begin{equation}
    \mathds{1}\{\bD(h^*)\} = \bP_{\mcI} \cdot \mathds{1}\{\bD(g^{-1})\} \ .
\end{equation}
Therefore, we have $\mathds{1}\{\bD(h^*)\} \odot \mathds{1}\{\bD(h^*)^{\top}\} = \bI_{n \times n}$. Finally, $(h^*)^{-1} \triangleq g \cdot \big[\bD_{\rm t}(g^{-1};\sigma)\big]^{-\top}$ satisfies $(h^*)^{-1}\circ h(\bX)=\bX$. Hence, $(h^*, \sigma)$ minimizes the objective in~\eqref{eq:general-OPT2} at value zero.
\endproof

\subsection{Proof of Lemma~\ref{lm:uncoupled-feasibility}}\label{proof:uncoupled-feasibility}

We will prove the desired result by contradiction. Suppose that $(h^*, \pi)$ is a solution to the optimization problem specified in \eqref{eq:general-OPT2} and let $\pi \neq \sigma$. Denote $f = h^* \circ g$, so we have $\hat \bZ = f(\bZ)$. If $\pi_m = \sigma_m$ for some $m \in [n]$, this means $I^m = \tilde I^{\pi_m} = \ell$ for some node $\ell \in [n]$. We follow the same approach in the proof of \Cref{thm:general-partial-identifiability}. Specifically, using score difference transformation property in \eqref{eq:score-difference-z-zhat} and one-sparse property of $[\bss^{m}(\bZ) - \tilde \bss^{\pi_m}(\bZ)]$ via \Cref{lm:parent_change_comprehensive}(iii), we have
\begin{align}
    [\bD_{\rm t}(h^*)]_{i,m} &= \E \Big[\big| \bss_{\hat \bZ}^{m}(\hat \bZ) - \tilde \bss_{\hat \bZ}^{\pi_m}(\hat \bZ) \big|_i \Big] \\
    &= \E \bigg[ \Big| \big[J_{f}^{-\top}(\bZ)\big]_i \cdot \big[\bss^{m}(\bZ) - \tilde \bss^{\pi_m}(\bZ)\big] \Big| \bigg] \\ 
    &= \E \bigg[ \Big[J_{f}^{-1}(\bZ)\big]_{\ell,i} \cdot \big[\bss^{m}(\bZ) - \tilde \bss^{\pi_m}(\bZ)\big]_{\ell} \Big] \ .
    \end{align}
By definition of $h^*$, $[\bD_{\rm t}(h^*)]_{i,m} = 0$ for all $i \neq m$. Also, interventional discrepancy between $q_\ell(z_\ell)$ and $\tilde q_\ell(z_\ell)$ implies that $\big[\bss^{m}(\bz) - \tilde \bss^{\pi_m}(\bz)\big]_{\ell} \neq 0$ except for a null set. Then, $[\bD_{\rm t}(h^*)]_{i,m} = 0$ implies that $\big[J_{f}^{-1}(\bz)\big]_{\ell,i} = 0$ except for a null set. Since $J_{f}^{-1}$ is a continuous function, this implies that $\big[J_{f}^{-1}(\bz)\big]_{\ell,i} = 0$ for all $\bz \in \R^n$. Furthermore, since $J_{f}^{-1}$ is invertible for all $\bz$, none of its columns can be a zero vector. Hence, for all $\bz \in \R^n$, $\big[J_{f}^{-1}(\bz)\big]_{\ell,m}  \neq 0$. To summarize, if $\pi_m = \sigma_m$, then
\begin{align}\label{eq:matching-node-column}
    \forall \bz\in \R^n \;\; \big[J_f^{-1}(\bz)\big]_{I^m,i} \neq 0 \quad \iff \quad i = m   \ .
\end{align}
Next, consider the set of \emph{mismatched nodes}
\begin{align}
    \mcA \triangleq \{I^m  : I^m \neq \tilde I^{\pi_m}\} \ .
\end{align}
Let $I^a \in \mcA$ be a non-descendant of all the other nodes in $\mcA$. There exist nodes $I^b, I^c \in \mcA$, not necessarily distinct, such that 
\begin{equation}\label{eq:rho-abc}
    I^a = \tilde I^{\pi_b} \ , \quad \mbox{and} \quad I^c = \tilde I^{\pi_a} \ .
\end{equation} 
In four steps, we will show that $[\bD(h^*)]_{a,b}\neq 0$ and $[\bD(h^*)]_{b,a}\neq 0$, which violates the constraint $\mathds{1}\{\bD(h^*)\} \odot \mathds{1}\{\bD^{\top}(h^*)\} = \bI_{n \times n}$ and will conclude the proof by contradiction. Before giving the steps, we provide the following argument which we repeatedly use in the rest of the proof. For any continuous function $f:\R^n \to \R$, we have
\begin{align}
    \E\Big[\big|f(\bZ)\big|\Big] \neq 0 \;\; &\iff \;\; \E\Big[\big|f(\bZ) \cdot \big[\bss(\bZ) - \bss^{a}(\bZ)\big]_{I^a} \big|\Big] \neq 0  \ , \label{eq:continuous-argument-2} \\
    \mbox{and} \quad \E\Big[\big|f(\bZ)\big|\Big] \neq 0  \;\; &\iff \;\; \E\Big[\big|f(\bZ) \cdot \big[\bss(\bZ) - \tilde \bss^{\pi_b}(\bZ)\big]_{I^a} \big|\Big] \neq 0  \ . \label{eq:continuous-argument-3}
\end{align}
First, suppose that $\E\big[|f(\bZ)|\big] \neq 0$. Then, there exists an open set $\Psi \subseteq \R^n$ for which $f(\bz)\neq 0$ for all $z\in \Psi$. Due to the interventional discrepancy condition, there exists an open set within $\Psi$ for which $[\bss(\bZ) - \bss^{a}(\bZ)]_{I^a} \neq 0$. This implies that
\begin{equation}
    \E\Big[\big|f(\bZ) \cdot \big[\bss(\bZ) - \bss^{a}(\bZ)\big]_{I^a} \big|\Big] \neq 0 
\end{equation}
For the other direction, suppose that $\E\Big[\big|f(\bZ) \cdot \big[\bss(\bZ) - \bss^{a}(\bZ)\big]_{I^a} \big|\Big] \neq 0 $, which implies that there exists an open set $\Psi$ for which both $f(\bz)$ and $[\bss(\bz) - \bss^{a}(\bz)]_{I^a}$ are nonzero. Then, $\E\big[|f(\bZ)|\big] \neq 0$, and we have \eqref{eq:continuous-argument-2}. Similarly, due to $\tilde I^{\pi_b} = I^a$ and interventional discrepancy, we obtain \eqref{eq:continuous-argument-3}. 
\paragraph{Step 1: Show that $\E\Big[\big|[J_f^{-1}(\bZ)]_{I^a,a}\big|\Big] \neq 0$.} First, using \eqref{eq:score-difference-z-zhat} and \Cref{lm:parent_change_comprehensive}(i), we have
\begin{align}
    \big[\bD(h^*)\big]_{a,a} &= \E\bigg[\Big|\big[J_f^{-\top}(\bZ)\big]_{a} \cdot \big[\bss(\bZ)-\bss^{a}(\bZ)\big]\Big|\bigg] \\ 
    &= \E\bigg[\Big|\sum_{I^j \in \Paplus(I^a)} \big[J_f^{-1}(\bZ)\big]_{I^j,a} \cdot \big[\bss(\bZ)-\bss^{a}(\bZ)\big]_{I^j}\Big|\bigg] \ . \label{eq:no-solution-step1}
\end{align}
Note that $\Paplus(I^a) \cap \mcA = \{I^a\}$ since $I^a$ is non-descendant of the other nodes in $\mcA$. Consider $I^j\in\Pa(I^a)$, which implies that $I^j \notin \mcA$ and $I^j=\tilde I^{\pi_j}$. By \eqref{eq:matching-node-column}, we have $[J_f^{-1}(\bZ)]_{I^j,a} = 0$. Then, \eqref{eq:no-solution-step1} becomes
\begin{equation}
    \big[\bD(h^*)\big]_{a,a} = \E\bigg[\Big|\big[J_f^{-1}(\bZ)\big]_{I^a,a} \cdot \big[\bss(\bZ)-\bss^{a}(\bZ)\big]_{I^a}\Big|\bigg] \neq 0 \ ,
\end{equation} 
since diagonal entries of $\bD(h^*)$ are nonzero due to the last constraint in \eqref{eq:general-OPT2}. Then, \eqref{eq:continuous-argument-2} implies that $\E\Big[\big|[J_f^{-1}(\bZ)]_{I^a,a}\big|\Big] \neq 0$.
\paragraph{Step 2: Show that $\big[\tilde \bD(h^*)\big]_{a,b}\neq 0$.} Next, we use $I^a=\tilde I^{\pi_b}$ and \Cref{lm:parent_change_comprehensive}(i) to obtain
\begin{align}
    \big[\tilde \bD(h^*)\big]_{a,b} &=  \E\bigg[\Big|\big[J_f^{-\top}(\bZ)\big]_{a} \cdot \big[\bss(\bZ)-\tilde \bss^{\pi_b}(\bZ)\big]\Big|\bigg]\\
    &= \E\bigg[\Big|\sum_{I^j \in \Paplus(I^a)} \big[J_f^{-1}(\bZ)\big]_{I^j,a} \cdot \big[\bss(\bZ) - \tilde \bss^{\pi_b}(\bZ)\big]_{I^j} \Big|\bigg] \\
    &= \E\bigg[\Big|\big[J_f^{-1}(\bZ)\big]_{I^a,a} \cdot \big[\bss(\bZ)-\tilde \bss^{\pi_b}(\bZ)\big]_{I^a} \Big|\bigg] \ .
\end{align}
Using \eqref{eq:continuous-argument-3} and Step~1 result, we have $\big[\tilde \bD(h^*)\big]_{a,b}\neq 0$.
\paragraph{Step 3: Show that $\E\Big[\big|[J_f^{-1}(\bZ)]_{I^a,b}\big|\Big] \neq 0$.} Using \eqref{eq:score-difference-z-zhat} and \Cref{lm:parent_change_comprehensive}(i), we have
\begin{align}
    \big[\tilde\bD(h^*)\big]_{b,b} &= \E\bigg[\Big|\big[J_f^{-1}(\bZ)\big]_{b} \cdot \big[\bss(\bZ)-\tilde \bss^{\pi_b}(\bZ)\big]\Big|\bigg] \\ 
    &= \E\bigg[\Big|\sum_{I^j \in \Paplus(I^a)} \big[J_f^{-1}(\bZ)\big]_{I^j,b} \cdot \big[\bss(\bZ)-\tilde s^{\pi_b}(\bZ)\big]_{I^j}\Big|\bigg] \\ 
    &= \E\bigg[\Big|\big[J_f^{-1}(\bZ)]_{I^a,b} \cdot \big[\bss(\bZ)-\tilde s^{\rho_c}(\bZ)\big]_{I^a} \Big|\bigg] \ .
\end{align}
Since $\mathds{1}\{\bD(h^*)\}=\mathds{1}\{\tilde \bD(h^*)\}$, the diagonal entry $\big[\tilde \bD(h^*)\big]_{b,b}$ is nonzero. Then, using \eqref{eq:continuous-argument-3} we have $\E\Big[\big|[J_f^{-1}(\bZ)]_{I^a,b}\big|\Big] \neq 0$.
\paragraph{Step 4: Show that $\big[\bD(h^*)\big]_{b,a}\neq 0$.} Next, we use $I^a=\tilde I^{\pi_b}$ and \Cref{lm:parent_change_comprehensive}(i) to obtain
\begin{align}
    \big[\bD(h^*)\big]_{b,a} &=  \E\bigg[\Big|\big[J_f^{-\top}(\bZ)\big]_{b} \cdot \big[\bss(\bZ)-\bss^{a}(\bZ)\big]\Big|\bigg]\\
    &= \E\bigg[\Big|\sum_{I^j \in \Paplus(I^a)} \big[J_f^{-1}(\bZ)\big]_{I^j,b} \cdot \big[\bss(\bZ) - \bss^{a}(\bZ)\big]_{I^j} \Big|\bigg] \\
    &= \E\bigg[\Big|\big[J_f^{-1}(\bZ)\big]_{I^a,b} \cdot \big[\bss(\bZ)- \bss^{a}(\bZ)\big]_{I^a} \Big|\bigg] \ .
\end{align}
Using \eqref{eq:continuous-argument-2} and Step~3 result, we have $[\bD(h^*)]_{b,a}\neq 0$.

Finally, using the constraint $\mathds{1}\{\bD(h^*)\}=\mathds{1}\{\tilde \bD(h^*)\}$, Step~2 implies that  $[\bD(h^*)]_{a,b}\neq 0$. Combining with Step~4 result, we have $[\bD(h^*) \odot \bD^{\top}(h^*)]_{a,b} \neq 0$, which violates the last constraint in \eqref{eq:general-OPT2}. Therefore, if $(h^*,\pi)$ is a minimizer of \eqref{eq:general-OPT2}, then $\pi$ must be the correct coupling $\sigma$.

\section{Analysis of the Assumptions}\label{appendix:assumptions}

\subsection{Analysis of Assumption~\ref{assumption:rank-two}}\label{sec:discuss-assumption-rank-two}

In this subsection, we prove \Cref{lm:rank-two-additive-hard} statement, i.e., \Cref{assumption:rank-two} is satisfied for additive noise models under hard interventions, as follows.
Consider a hard interventional environment $\mcE^m$ and let $I^m = i$ be the intervened node. Recall \eqref{eq:def-score-diff-z-int-obs-single}, which becomes
\begin{equation}
   [\bsd_{\bZ}^{m}(\bz)]_k =  \big[\bss(\bz) - \bss^{m}(\bz)\big]_k = \frac{\partial}{\partial z_k} \log p_{i}(z_i \mid \bz_{\Pa(i)})  \ ,
\end{equation}
for the hard intervention on node $i$ and parent node $k \in \Pa(i)$. Then, the ratio in \Cref{assumption:rank-two} is given by 
\begin{equation}\label{eq:ratio-SN-2}
     \frac{[\nabla_{\bz} \log p(z_i \mid \bz_{\Pa(i)})]_k }{[\nabla_{\bz} \log p(z_i \mid \bz_{\Pa(i)})]_i - [\nabla_{\bz} \log q(z_i)]_i } = \frac{[\bsd_{\bZ}^{m}(\bz)]_k}{[\bsd_{\bZ}^{m}(\bz)]_i} \ .
\end{equation}
We will prove the desired result by contradiction. Assume that there exists a nonzero constant $c \in \R$ such that $[\bsd_{\bZ}^{m}(\bz)]_k = c \cdot [\bsd_{\bZ}^{m}(\bz)]_i$ for all $\bz \in \R^n$.

\paragraph{Additive noise model.} The additive noise model for node $i$ is given by $Z_i = f_i(\bZ_{\Pa(i)}) + N_i$ as specified in \eqref{eq:additive-SCM}. When node $i$ is hard intervened, $Z_i$ is generated according to $Z_i = \bar N_i$ in which $\bar N_i$ denotes the exogenous noise term for the interventional model. Then, denoting the pdfs of $N_i$ and $\bar N_i$ by $p_N$ and $q_N$, respectively, we have 
\begin{align}\label{eq:additive-pdf-connection}
    p_i(z_i \mid \bz_{\Pa(i)}) = p_N(z_i - f_i(\bz_{\Pa(i)}))\ , \quad \mbox{and} \quad 
    q_i(z_i) = q_N(z_i)\ .
\end{align}
Denote the score functions associated with $p_N$ and $q_N$ by
\begin{align}\label{eq:def-noise-scores}
    r_p(u) \triangleq \frac{{\rm d}}{{\rm d} u} \log p_N(u) = \frac{p_N'(u)}{p_N(u)}\ , \quad \mbox{and} \quad
    r_q(u) \triangleq \frac{{\rm d}}{{\rm d} u} \log q_N(u) = \frac{q_N'(u)}{q_N(u)}\ .
\end{align}
Define $n_i$ and $\bar n_i$ as the realizations of $N_i$ and $\bar N_i$ when $Z_i=z_i$ and $\bZ_{\Pa(i)}=\bz_{\Pa(i)}$. Then, we have $\bar n_i = n_i + f_i(\bz_{\Pa(i)})$. Using \eqref{eq:additive-pdf-connection} and \eqref{eq:def-noise-scores}, we can express \eqref{eq:ratio-SN-2} as 
\begin{align}\label{eq:rank2-proof-additive-score-diff}
\frac{[\bsd_{\bZ}^{m}(\bz)]_k}{[\bsd_{\bZ}^{m}(\bz)]_i}  = \frac{-\pdv{f_i(\bz_{\Pa(i)})}{z_k} \cdot r_p(n_i) }{r_p(n_i) - r_q(n_i + f_i(\bz_{\Pa(i)}))} \ . 
\end{align} 
Then, $[\bsd_{\bZ}^{m}(\bz)]_k = c \cdot [\bsd_{\bZ}^{m}(\bz)]_i$ for all $\bz \in \R^n$ becomes
\begin{equation}\label{eq:additive-rank2-step2}
    c \cdot \frac{\partial f_i(\bz_{\Pa(i)})}{\partial z_k} + 1 = \frac{r_q(n_i + f_i(\bz_{\Pa(i)}))}{r_p(n_i)} \ .
\end{equation}
We scrutinize \eqref{eq:additive-rank2-step2} in two cases.
\paragraph{Case 1: $\frac{\partial f_i(\bz_{\Pa(i)})}{\partial z_k}$ is constant.} In this case, let $\eta = \frac{\partial f_i(\bz_{\Pa(i)})}{\partial z_k}$ for some $\eta \in \R$ for all $\bz_{\Pa(i)} \in \R^{|\Pa(i)|}$. Then, the RHS of \eqref{eq:additive-rank2-step2} is also constant for all $\bz_{\Pa(i)} \in \R^{|\Pa(i)|}$ and $n_i \in \R$. Fix a realization $n_i = n_i^*$ and note that $f_i(\bz_{\Pa(i)})$ needs to be constant, denoted by $\delta^*$, for all $\bz_{\Pa(i)} \in \R^{|\Pa(i)|}$. However, this implies that $\frac{\partial f_i(\bz_{\Pa(i)})}{\partial z_k} = 0$, and we have $r_q(n_i + \delta^*) = r_p(n_i)$ for all $n_i \in \R$. This implies that $p_N(n_i) = \eta^* q_N(n_i+\delta^*)$ for some constant $\eta^*$. Since $p_N$ and $q_N$ are pdfs, the only choice is $\eta^*=1$ and $p_N(n_i)=q_N(n_i+\delta^*)$, then $p_i(z_i \mid \bz_{\Pa(i)}) =  q_i(z_i)$, which contradicts the premise that an intervention changes the causal mechanism of the target node $i$.

\paragraph{Case 2: $\frac{\partial f_i(\bz_{\Pa(i)})}{\partial z_k}$ is not constant.} In this case, note that LHS of \eqref{eq:additive-rank2-step2} is not a function of $n_i$. Then, taking the derivative of both sides with respect to $n_i$ and rearranging, we obtain
\begin{equation}\label{eq:additive-rank2-ratio}
    \frac{r_p'(n_i)}{r_p(n_i)} =
    \frac{r_q'(n_i + f_i(\bz_{\Pa(i)}))}{r_q(n_i + f_i(\bz_{\Pa(i)}))} \ , \quad \forall (n_i,\bz_{\Pa(i)})\in \R \times \R^{|\Pa(i)|}  \ .
\end{equation}
Next, consider a fixed realization $n_i=n_i^*$, and denote the value of LHS by $\alpha$. Since $f_i(\bz_{\Pa(i)})$ is continuous and not constant, its image contains an open interval $\Theta \subseteq \mathbb{R}$. Denoting $u\triangleq f_i(\bz_{\Pa(i)})$, we have
\begin{equation}
    \alpha = \frac{r_q'(n_i^* + u)}{r_q(n_i^* + u)} \ , \quad \forall u \in \Theta \ .
\end{equation}
The only solution to this equality is that $r_q(n_i^* + u)$ is an exponential function, $r_q(u) = k_1\exp(\alpha u)$ over interval $u \in \Theta$. Since $r_q$ is an analytic function that equals to an exponential function over an interval, it is exponential over entire $\mathbb{R}$. This implies that the pdf $q_N(u)$ is of the form $q_N(u) = k_2 \exp((k_1/\alpha) \exp(\alpha u))$. However, this cannot be a valid pdf since its integral over $\R$ diverges. Then, $[\bsd_{\bZ}^{m}(\bz)]_k = c \cdot [\bsd_{\bZ}^{m}(\bz)]_i$ cannot be true, which concludes the proof.

\subsection{Analysis of Assumption~\ref{assumption:full-rank}}\label{appendix:analyze-assumption-full-rank}
In this section, we establish the necessary and sufficient conditions under which \Cref{assumption:full-rank} holds for additive models. Furthermore, we show that a large class of nonlinear models in the latent space satisfy these conditions, including the two-layer neural networks. We have focused on such NNs since they effectively approximate continuous functions \citep{cybenko1989approximation}. Readily, the necessary and sufficient conditions can be investigated for other choices of nonlinear functions.

\subsubsection{Interpreting Assumption~\ref{assumption:full-rank}}\label{appendix:interpret-assumption-full-rank}
The implication of \Cref{assumption:full-rank} is that the effect of an intervention is not lost in any linear combination of the varying coordinates of the scores. To formalize this, for each node $i\in[n]$, we define 
\begin{align}
    \mcC_i \triangleq \{\bc\in\R^n\;:\; \exists j \in \Paplus(i) \;\; \mbox{such that }\; c_j\neq 0 \}\ , \label{eq:C}
\end{align}
Let $I^m = i$. Note that $\bR_{\bZ}^{m}$ is positive semi-definite and $[\bR_{\bZ}^{m}]_{j,k} = 0$ for all $(j,k) \notin \Paplus(i) \times \Paplus(i)$. Therefore, for any $\bc \in \mcC_{i}$, 
\begin{equation}
    \bc^{\top} \cdot \bR_{\bZ}^{m} \cdot \bc = \bE\big[ (\bc^{\top} \cdot \bsd_{\bZ}^m) \cdot (\bc^{\top} \cdot \bsd_{\bZ}^m)^{\top}  \big] \neq 0 \ ,
\end{equation}
since $\bR_{\bZ}^{m}$ has rank $|\Paplus(i)|$. The reverse direction also holds true, i.e., if $\bR_{\bZ}^{m}$ has rank less than $|\Paplus(i)|$, then there exists $\bc \mcC_{i}$ which makes $\bc^{\top} \cdot \bR_{\bZ}^{m} \cdot \bc = 0$. Therefore, \Cref{assumption:full-rank} is equivalent to the following statement:
\begin{equation}\label{eq:assumption-full-rank-original}
    \bE\big[|\bc^{\top} \cdot \bsd_{\bZ}^{m}|\big] \neq 0 \ , \quad \forall \bc \in \mcC_{i} \ . 
\end{equation}

\paragraph{Necessary and sufficient conditions.} Consider the additive noise model for node $i$
\begin{align}
    Z_i &= f_i(\bZ_{\Pa(i)}) + N_i\ ,  \label{eq:additive-model-discussion-proof-obs}
\end{align}
as specified in \eqref{eq:additive-SCM}. When node $i$ is soft intervened, $Z_i$ is generated according to
\begin{align}
    Z_i &= \bar f_i(\bZ_{\Pa(i)}) + \bar N_i\ ,  \label{eq:additive-model-discussion-proof-int}
\end{align}
in which $\bar f_i$ and $\bar N_i$ specify the interventional mechanism for node $i$. The following lemma characterizes the necessary and sufficient conditions under which \eqref{eq:assumption-full-rank-original} (equivalently, \Cref{assumption:full-rank}) is satisfied. In this subsection, we use $\varphi$ as the shorthand for $\bz_{\Pa(i)}$. 
\begin{lemma}\label{lm:condition-assumption}
For each node $i\in[n]$ consider the following two set of equations for $c\in\R^n$:
\begin{align}\label{eq:NS_equations}    
   \left\{
   \begin{array}{l}
         c_i - \bc^{\top} \cdot \nabla_{\bz} f_i(\varphi)  = 0  \\
         \\ 
         c_i - \bc^{\top} \cdot \nabla_{\bz} \bar f_i(\varphi)  = 0
   \end{array}
   \right.\ , \qquad \qquad \forall \varphi \in\R^{|\Pa(i)|}\ .
\end{align}
\Cref{assumption:full-rank} holds if and only if the only all solutions $\bc$ to \eqref{eq:NS_equations} satisfy $\bc \notin \mcC_i$, or based on~\eqref{eq:C}, equivalently $c_j =0$ for all $j \in \Paplus(i)$.
\end{lemma}
\proof See \Cref{proof:condition-assumption}. \\

To provide some intuition about the conditions in \Cref{lm:condition-assumption}, we consider a node $i \in [n]$ and discuss the conditions in the context of a few examples. Note that by sweeping $\varphi\in \R^{|\Pa(i)|}$ we generate a continuum of linear equations of the form: 
\begin{align}\label{eq:NS_equations2}
c_i - \bc^{\top} \cdot \nabla_{\bz} f_i(\varphi)  = 0\ , \qquad \mbox{and} \qquad c_i - \bc^{\top} \cdot \nabla_{\bz} \bar f_i(\varphi)  = 0  \ .
\end{align}
Note that for all $j\not\in \Pa(i)$ we have $[\nabla_{\bz} f_i(\varphi)]_j=[\nabla_{\bz} \bar f_i(\varphi)]_j=0$. Hence, in finding the solutions to~\eqref{eq:NS_equations2} only the coordinates $\{j\in \Paplus(i)\}$ of $\bc$ are relevant. Let us define 
\begin{align} \label{eq:varphi}
  u(\varphi)\triangleq \nabla_{\varphi} f_i(\varphi) \ , \quad \mbox{and} \quad 
  \bar u(\varphi)\triangleq \nabla_{\varphi} \bar f_i(\varphi) \ ,
\end{align}
which are the gradients of $f_i$ and $f_i$ by considering only the coordinates of $\bz$ in $\{j\in {\Pa}(i)\}$. Accordingly, we also define $\bb$ by concatenating only the coordinates of $\bc$ with their indices in $\{j\in {\Pa}(i)\}$. Next, consider $w$ distinct choices of $\varphi$ and denote them by $\{\varphi^t\in\R^{|\Pa(i)|}:t\in[w]\}$. By concatenating the two equations in~\eqref{eq:varphi} specialized to these realizations, we get the following linear system with $2w$ equations and $|\Pa(i)|+1$ unknown variables.
\begin{align}\label{eq:discussion-full-rank}
    \underset{\triangleq \, \bV \, \in \, \R^{2w\times (|\Pa(i)|+1)}}{\underbrace{
    \begin{bmatrix}
        [u(\varphi^1)]^{\top} & -1\\
        [\bar u(\varphi^1)]^{\top} & -1\\
        \vdots & \vdots\\
        [u(\varphi^w)]^{\top} & -1\\
        [\bar u(\varphi^w)]^{\top} & -1\\
    \end{bmatrix}}}
    \begin{bmatrix}
        \bb \\ c_i
    \end{bmatrix}
    =\boldsymbol{0}_{2w}\ .
\end{align}
When $\bV$ is full-rank, i.e., $\rank(\bV)=|\Pa(i)|+1$, the system has only the trivial solutions $c_i=0$ and $\bb=\boldsymbol{0}$. Then, we make the following observations.

\begin{enumerate}[leftmargin=*]
    \item If $f_i$ and $\bar f_i$ are linear functions, the vector spaces generated by $u$ and $\bar u$ have dimensions 1. Subsequently, we always have $\rank(\bV)\leq 2$, rendering an underdetermined system when $|\Pa(i)|\geq 2$. Hence, when the maximum degree of $\mcG$ is at least 2, a linear causal model does not satisfy \Cref{assumption:full-rank}. 
    
    \item If $f_i$ and $\bar f_i$ are quadratic with full-rank matrices, i.e., $f_i(\varphi) = \varphi^{\top} \bA \varphi$ and $\bar f_i(\varphi) = \varphi^{\top} \bar \bA \varphi$ where $\rank(\bA)=\rank(\bar \bA)=|\Pa(i)|$, there is a choice of $w \geq |\Pa(i)| + 1$ and realizations $\{\varphi^t \in \R^{|\Pa(i)|} : t \in [w] \}$ for which $\rank(\bV)=|\Pa(i)|+1$ and the system in~\eqref{eq:NS_equations2} admits only the trivial solutions $a_i=0$ and $\bb=\boldsymbol{0}$. Hence, quadratic causal models satisfy \Cref{assumption:full-rank}.
    
    \item If $f_i$ and $\bar f_i$ are two-layer NNs with a sufficiently large number of hidden neurons, they also render a fully determined system, and as a result, they satisfy \Cref{assumption:full-rank}. 
\end{enumerate}
We investigate the last example in detail as follows. Assume that $f_i$ and $\bar f_i$ are two-layer NNs with $|\Pa(i)|$ inputs, $w_i$ and $\bar w_i$ hidden nodes, respectively, and with sigmoid activation functions. Denote the weight matrices between input and hidden layers in $f_i$ and $\bar f_i$ by $\bW^i \in \R^{w_i \times |\Pa(i)|}$ and $\bar \bW^i \in \R^{\bar w_i \times |\Pa(i)|}$, respectively. Furthermore, define $\bnu \in \R^{w_i}$ and $\bar \bnu \in \R^{w_i}$ as the weights between the hidden layer and output in $f_i$ and $\bar f_i$, respectively. Finally, define $\nu_0$ and $\bar \nu_0$ as the bias terms. Hence, we have
\begin{align}
    f_i(\varphi) & = \bnu^{\top} \cdot \sigma(\bW^i \cdot \varphi) + \nu_0  = \sum_{j=1}^{w_i} \nu_j \cdot \sigma(\bW^i_j \cdot \varphi) + \nu_0  \ , \label{eq:nn-expression} \\
    \mbox{and} \quad \bar f_i(\varphi) & = \bar \bnu^{\top} \cdot \sigma(\bar \bW^i \cdot \varphi) + \nu_0  = \sum_{j=1}^{\bar w_i} \bar \nu_j \cdot \sigma(\bar \bW^i_j \cdot \varphi) + \bar \nu_0  \ , \label{eq:nn-expression2}
\end{align}
in which activation function $\sigma$ is applied element-wise.

\paragraph{\Cref{prop:nn}}
Consider NNs $f_i$ and $\bar f_i$ specified in~\eqref{eq:nn-expression} and \eqref{eq:nn-expression2}. If for all $i \in[n]$ we have
   \begin{align}
       \max\{\rank(\bW^i)\; , \; \rank(\bar \bW^i)\} = |\Pa(i)|\ ,
   \end{align}
then \Cref{assumption:full-rank} holds. \\

\proof See \Cref{proof:nn}.

\subsubsection{Proof of Lemma~\ref{lm:condition-assumption}}\label{proof:condition-assumption}
We show that for node $I^m = i$, the condition in \eqref{eq:assumption-full-rank-original}
\begin{align}
    \E\Big[\big| \bc^{\top} \cdot \bsd_{\bZ}^{m} \big|\Big] \neq 0  \ , \;\; \forall c \in \mcC_i\ , 
    \label{eq:assumption_sc:node_level}
\end{align}
holds if and only if the following continuum of equations admit their solutions $\bc$ in $\R^n \setminus \mcC_i$.
\begin{align}
  \left\{
   \begin{array}{l}
         c_i - \bc^{\top} \cdot \nabla_{\bz} f_i(\varphi)  = 0  \\
         \\ 
         c_i - \bc^{\top} \cdot \nabla_{\bz} \bar f_i(\varphi)  = 0
   \end{array}
   \right.\ , \qquad \qquad \forall \varphi \in\R^{|\Pa(i)|}\ , \label{eq:assumption_sc:condition_node_level}
\end{align}
in which shorthand $\varphi$ is used for $\bz_{\Pa(i)}$. We first note that the condition in \eqref{eq:assumption_sc:node_level} can be equivalently stated using \eqref{eq:s_z_decompose_obs} and \eqref{eq:s_z_decompose_int} as 
\begin{align}\label{eq:equivalent_statement2-pre}
    \E\bigg[\Big| \bc^{\top}\cdot \big[\nabla_{\bz} \log p_i(Z_i \mid \bZ_{\Pa(i)}) - \nabla_{\bz} \log q_i(Z_i \mid \bZ_{\Pa(i)}) \big]\Big|\bigg] \neq 0 \ , \quad \forall \bc \in \mcC_{i} \ .
\end{align}
Also, using \Cref{prop:continuity-argument}, \eqref{eq:equivalent_statement2-pre} is equivalent to 
\begin{align}\label{eq:equivalent_statement2}
    \forall \bc \in \mcC_i, \;\; \exists \bz : \quad \bc^{\top}\cdot \nabla_{\bz} \log p_i(z_i \mid \bz_{\Pa(i)}) \neq \bc^{\top}\cdot \nabla_{\bz} \log q_i(z_i \mid \bz_{\Pa(i)}) \ . 
\end{align}
The additive noise model for node $i$ is given by
\begin{align}
    Z_i &= f_i(\bZ_{\Pa(i)}) + N_i\ ,  \label{eq:additive-model-aux-proof-obs}
\end{align}
as specified in \eqref{eq:additive-SCM}. When node $i$ is soft intervened, $Z_i$ is generated according to
\begin{align}
    Z_i &= \bar f_i(\bZ_{\Pa(i)}) + \bar N_i\ ,  \label{eq:additive-model-aux-proof-int}
\end{align}
in which $\bar f_i$ and $\bar N_i$ specify the interventional mechanism for node $i$. Then, denoting the pdfs of $N_i$ and $\bar N_i$ by $p_N$ and $q_N$, respectively, \eqref{eq:additive-model-aux-proof-obs} and \eqref{eq:additive-model-aux-proof-int} imply that
\begin{align}\label{eq:p-h-connection}
    p_i(z_i \mid \bz_{\Pa(i)}) = p_N\big(z_i - f_i(\bz_{\Pa(i)})\big)\ , \quad \mbox{and} \quad 
    q_i(z_i \mid \bz_{\Pa(i)}) = q_N\big(z_i - \bar f_i(\bz_{\Pa(i)})\big)\ . 
\end{align}
Denote the score functions associated with $p_N$ and $q_N$ by
\begin{align}\label{eq:def-r}
    r_p(u) \triangleq \frac{{\rm d}}{{\rm d} u} \log p_N(u) = \frac{p_N'(u)}{p_N(u)}\ , \quad \mbox{and} \quad
    r_q(u) \triangleq \frac{{\rm d}}{{\rm d} u} \log q_N(u) = \frac{q_N'(u)}{q_N(u)}\ .
\end{align}
Define $n_i$ and $\bar n_i$ as the realizations of $N_i$ and $\bar N_i$ when $Z_i=z_i$ and $\bZ_{\Pa(i)}=\bz_{\Pa(i)}$. By defining $\delta(\bz_{\Pa(i)}) \triangleq f_i(\bz_{\Pa(i)}) - \bar f_i(\bz_{\Pa(i)})$, we have $\bar n_i = n_i + \delta(\bz_{\Pa(i)})$. Using \eqref{eq:def-r} and \eqref{eq:p-h-connection}, we can express the relevant entries of $\nabla_{\bz} \log p_i(z_i \mid \bz_{\Pa(i)})$ and $\nabla_{\bz} \log q_i(z_i \mid \bz_{\Pa(i)})$ as
\begin{align}
    \big[\nabla_{\bz} \log p(z_i \mid \bz_{\Pa(i)})\big]_j
    &= \begin{cases} 
        r_p(n_i)  \ , &j = i \ , \\
        -\pdv{f_i(\bz_{\Pa(i)})}{z_j} r_p(n_i)  \ , &j \in \Pa(i) \ , \\
        0 \ , &j \notin  \Paplus(i)  \ ,
    \end{cases} \label{eq:a2_score_obs}  \\
    \mbox{and} \quad 
    \big[\nabla_{\bz} \log q_i(z_i \mid \bz_{\Pa(i)})\big]_j
    &= \begin{cases} 
        r_q(\bar n_i)  \ , &j = i \ , \\
        -\pdv{\bar f_i(\bz_{\Pa(i)})}{z_j}  r_q(\bar n_i)  \ , &j \in \Pa(i) \  , \\
        0 \ , &j \notin  \Paplus(i) \ .
    \end{cases} \label{eq:a2_score_int}
\end{align}
By substituting \eqref{eq:a2_score_obs}--\eqref{eq:a2_score_int} in \eqref{eq:equivalent_statement2} and rearranging the terms, the statement in~\eqref{eq:equivalent_statement2} becomes equivalent to following statement. For all $\bc \in \mcC_i$, there exist $n_i \in \R$ and $\bz_{\Pa(i)}\in\R^{|\Pa(i)|}$ such that 
\begin{align} \label{eq:equivalent_statement3} 
    r_p(n_i) \cdot \bigg(c_i- \sum_{j \in {\Pa}(i)} c_j \cdot \pdv{f_i(\bz_{\Pa(i)})}{z_j}\bigg) \neq r_q(n_i+\delta(\bz_{\Pa(i)})) \cdot \bigg(c_i- \sum_{j \in {\Pa}(i)} c_j\cdot \pdv{\bar f_i(\bz_{\Pa(i)})}{z_j}\bigg)\ ,   
\end{align}
which by using the shorthand $\varphi$ for $\bz_{\Pa(i)}$ can be compactly presented as follows. For all $\bc\in \mcC_i$, there exists $n_i \in \R$ and $\varphi \in \R^{|\Pa(i)|}$ such that
\begin{align}\label{eq:equivalent_statement4}
    r_p(n_i) \cdot \left[c_i - \bc^{\top} \cdot \nabla_{\bz} f_i(\varphi)\right] \neq r_q\big(n_i+\delta(\varphi)\big) \cdot \left[c_i - \bc^{\top} \cdot \nabla_{\bz}  \bar f_i(\varphi)\right] \ .
\end{align}
Hence, \Cref{assumption:full-rank} is equivalent to the statement in~\eqref{eq:equivalent_statement4}, which we use for the rest of the proof.

\paragraph{Sufficient condition.} We show that if~\eqref{eq:assumption_sc:condition_node_level} admit solutions $\bc$ only in $\R^n\setminus \mcC_i$, then the statement in~\eqref{eq:equivalent_statement4} holds. By contradiction, assume that there exists $\bc^*\in\mcC_i$ such that for all $n_i \in\R$ and $\varphi \in \R^{|\Pa(i)|}$,
\begin{align}\label{eq:equivalent_statement5}
    r_p(n_i) \cdot \left[c^*_i-(\bc^*)^{\top} \cdot \nabla_{\bz} f_i(\varphi)\right] = r_q\big(n_i+\delta(\varphi)\big)  \cdot \left[c^*_i-(\bc^*)^{\top} \cdot \nabla_{\bz} \bar f_i(\varphi)\right] \  .
\end{align}
We show that $\bc^*\in\mcC_i$ is also a solution to~\eqref{eq:assumption_sc:condition_node_level}, contradicting the premise. In order to show that $\bc^*\in\mcC_i$ is also a solution to~\eqref{eq:assumption_sc:condition_node_level}, suppose, by contradiction, that \eqref{eq:equivalent_statement5} holds, and there exists $\varphi^*\in\R^{|\Pa(i)|}$ corresponding to which
\begin{align}\label{eq:contradiction_premise}
    (\bc^*)^{\top}\cdot \nabla_{\bz} \bar f_i(\varphi^*)- c^*_i \neq 0\ .
\end{align}
Note that $f_i$ is a continuously differentiable function and also a function of $z_j$ for all $j \in \Pa(i)$. Hence, there exists an open set $\Phi \subseteq \mathbb{R}^{|\Pa(i)|}$ for $\varphi$ for which $(\bc^*)^{\top}\cdot \nabla_{\bz} f_i(\varphi^*)- c^*_i$ is nonzero everywhere in $\Phi$. Likewise, $r_p$ cannot constantly be zero over all possible intervals $\Omega$. This is because otherwise, it would have to necessarily be a constant zero function (since it is analytic), which is an invalid score function. Hence, there exists an open interval $\Omega \subseteq \mathbb{R}$ over which $r_p(n_i)$ is nonzero for all $n_i \in\Omega$. Subsequently, the left-hand side of \eqref{eq:equivalent_statement5} is nonzero over the Cartesian product $(n_i, \varphi) \in \Omega \times \Phi$. This means that if \eqref{eq:equivalent_statement5} is true, then both functions on its right-hand side must also be nonzero over $\Omega \times \Phi$. Hence, by rearranging the terms in~\eqref{eq:equivalent_statement5} we have
\begin{align}
   \frac{r_q\big(n_i+\delta(\varphi)\big)}{r_p(n_i)} = \frac{c^*_i - (\bc^*)^{\top}\cdot \nabla_{\bz}  f_i(\varphi)}{ c^*_i - (\bc^*)^{\top}\cdot \nabla_{\bz} \bar f_i(\varphi)}\ , \qquad \forall (n_i, \varphi) \in \Omega \times \Phi \ . \label{eq:ratio}
\end{align}
In two steps, we show that \eqref{eq:ratio} cannot be valid.

\noindent \textbf{Step 1.} First, we show that function $\delta$ cannot be a constant over $\Phi$ (interval specified above). Suppose the contrary and assume that $\delta(\varphi)= \delta^*$ for all $\varphi \in \Phi$. Hence, the gradient of $\delta$ is zero. Using the definition of $\delta$, this implies that
\begin{align}
    \nabla_{\bz} f_i(\varphi) = \nabla_{\bz} \bar f_i(\varphi) \ , \quad \forall \varphi \in \Phi \ .
\end{align}
Then, by leveraging \eqref{eq:ratio}, we conclude that $r_p(n_i) = r_q(n_i + \delta^*)$ for all $n_i \in \Omega$. Since $r_p(n_i)$ and $r_q(n_i + \delta^*)$ are analytic functions that agree on an open interval of $\mathbb{R}$, they are equal for all $n_i \in \mathbb{R}$ as well. This implies that $p_N(n_i) = \eta \cdot q_N(n_i+\delta^*)$ for some constant $\eta$. Since $p_N$ and $q_N$ are pdfs, the only choice is $\eta=1$ and $p_N(n_i)=q_N(n_i+\delta^*)$, then $p_i(z_i \mid \bz_{\Pa(i)}) =  q_i(z_i \mid \bz_{\Pa(i)})$, which contradicts the premise that observational and interventional mechanisms are distinct.

\noindent \textbf{Step 2.} Finally, we will show that \eqref{eq:ratio} cannot be true when $\delta$ is not a constant function over $\Phi$. Note that the right-hand side of \eqref{eq:ratio} is not a function of $n_i$. Then, taking the derivative of both sides with respect to $n_i$ and rearranging, we obtain
\begin{align}\label{eq:ratio-derivatives}
    \frac{r_p'(n_i)}{r_p(n_i)} &=
    \frac{r_q'(n_i + \delta(\varphi))}{r_q(n_i + \delta(\varphi))} \ , \quad \forall (n_i,\varphi)\in \Omega \times \Phi  \ .
\end{align}
Next, consider a fixed realization $n_i=n_i^*$, and denote the value of LHS by $\alpha$. Since $\delta$ is continuous and not constant over $\Phi$, its image contains an open interval $\Theta \subseteq \mathbb{R}$. Denoting $u\triangleq \delta(\varphi)$, we get
\begin{align}
    \alpha = \frac{r_q'(n_i^* + u)}{r_q(n_i^* + u)} \ , \quad \forall u \in \Theta \ .
\end{align}
The only solution to this equality is that $r_q(n_i^* + u)$ is an exponential function, $r_q(u) = k_1\exp(\alpha u)$ over interval $u \in \Theta$. Since $r_q$ is an analytic function that equals to an exponential function over an interval, it is exponential over entire $\mathbb{R}$. This implies that the pdf $q_N(u)$ is of the form $q_N(u) = k_2 \exp((k_1/\alpha) \exp(\alpha u))$. However, this cannot be a valid pdf since its integral over $\R$ diverges. Hence, \eqref{eq:ratio} is not true, and the premise that there exists $\bc^* \in \mcC_i$ and $\varphi\in\R^{|\Pa(i)|}$ corresponding to which \eqref{eq:contradiction_premise} holds is invalid, concluding that for all $\varphi\in\R^{|\Pa(i)|}$ we have $\bc^{\top}\cdot \nabla_{\bz} \bar f_i(\varphi) = c_i$. Proving the counterpart identity $\bc^{\top}\cdot  \nabla_{\bz}  f_i(\varphi) = c_i$ follows similarly. Therefore, \eqref{eq:equivalent_statement5} implies
\begin{align}
  \left\{
   \begin{array}{l}
         c^*_i - (\bc^*)^{\top} \cdot \nabla_{\bz} f_i(\varphi)  = 0  \\
         \\ 
         c^*_i - (\bc^*)^{\top} \cdot \nabla_{\bz} \bar f_i(\varphi)  = 0
   \end{array}
   \right.\ , \qquad \qquad \forall \varphi \in\R^{|\Pa(i)|}\ ,
\end{align}
which means that we have found a solution to~\eqref{eq:assumption_sc:condition_node_level} that is not in $\R^n\setminus \mcC_i$, contradicting the premise. 

\paragraph{Necessary condition.} Assume that \Cref{assumption:full-rank}, and equivalently, the statement in~\eqref{eq:equivalent_statement4} holds but \eqref{eq:assumption_sc:condition_node_level} has a solution $\bc^*$ in $\mcC_i$. Then,
\begin{align}
     c^*_i - (\bc^*)^{\top} \cdot \nabla_{\bz} f_i(\varphi)  =  c^*_i - (\bc^*)^{\top} \cdot \nabla_{\bz} \bar f_i(\varphi)= 0 \ , \quad \forall \varphi \in \R^{|\Pa(i)|} \ .
\end{align}
Multiplying the left side by $r_p(n_i)$ and the right side by $r_q(n_i + \delta(\varphi))$, we obtain that for all $n_i \in \R$ and $ \varphi\in\R^{|\Pa(i)|}$
\begin{align}
      r_p(n_i) \cdot \left[c^*_i-(\bc^*)^{\top} \cdot \nabla_{\bz} f_i(\varphi)\right] = r_q(n_i+\delta(\varphi))  \cdot \left[c^*_i-(\bc^*)^{\top} \cdot \nabla_{\bz}  \bar f_i(\varphi)\right] =  0 \ , 
\end{align}
which implies that for all $\bc \in \mcC_i$, $n_i \in \R$ and $ \varphi\in\R^{|\Pa(i)|}$
\begin{align}
   r_p(n_i) \cdot \left[c_i-\bc^{\top} \cdot \nabla_{\bz} f_i(\varphi)\right] = r_q(n_i+\delta(\varphi))  \cdot \left[c_i-\bc^{\top} \cdot \nabla_{\bz} \bar f_i(\varphi)\right] \ .
\end{align}
This contradicts \eqref{eq:equivalent_statement4}, and equivalently \Cref{assumption:full-rank}. Hence, the proof is complete.  
\subsubsection{Proof of Lemma~\ref{prop:nn}}\label{proof:nn}

\paragraph{Approach.}
We will use the same argument as at the beginning of the proof of \Cref{lm:condition-assumption}. Specifically, we will show that for any node $i$ and two-layer NNs $f_i$ and $\bar f_i$ with weight matrices $\bW^i $ and $\bar \bW^i$, the following continuum of equations admit their solutions $\bc$ in $\R^n\setminus \mcC_i$.
\begin{align}
    \left\{
   \begin{array}{l}
         c_i -\bc^{\top} \cdot \nabla_{\bz} \bar f_i(\varphi)  = 0  \\
         \\ 
         c_i -\bc^{\top} \cdot \nabla_{\bz} f_i(\varphi)  = 0
   \end{array}
   \right.\ , \qquad \qquad \forall \varphi \in\R^{|\Pa(i)|}\ ,
\end{align}
in which shorthand $\varphi$ is used for $\bz_{\Pa(i)}$. Hence, by invoking \Cref{lm:condition-assumption}, \Cref{assumption:full-rank} holds. 

\paragraph{Definitions.} Define $d_i \triangleq |\Pa(i)|$. Since $\max\{\rank(\bW^i), \rank(\bar \bW^i)\}=d_i$, without loss of generality, suppose that $\bW^i \in \R^{w_i \times d_i}$ has rank $d_i$. The rest of the proof follows similarly for the case of $\rank(\bar \bW^i)=d_i$. We use shorthand $\{\bW, w\}$ for $\{\bW^i, w_i\}$ when it is obvious from context.

\paragraph{Parameterization.}
Note that $f_i$ can be represented by different parameterizations, some containing more hidden nodes than others. Without loss of generality, let $w$ be the fewest number of nodes that can represent $f_i$. This implies that the entries of $\bnu$ are nonzero. Otherwise, if $\nu_i=0$, we can remove $i$-th hidden node and still have the same $f_i$. Similarly, the rows of $\bW$ are distinct. Otherwise, if there exist rows $\bW_i = \bW_j$ for distinct $i,j \in [w]$, removing $j$-th hidden node and using $(\nu_i + \nu_j)$ in place of $\nu_i$ results the same function as $f_i$ with $(w-1)$ hidden nodes. Similarly, we have $\bW_i \neq \boldsymbol{0}$ for all $i \in [w]$. Otherwise, we have 
\begin{equation}
    \nu_i \cdot \sigma(\bW_i \cdot \varphi) + \nu_0 = (\nu_0 + \frac{\nu_i}{2}) \ ,
\end{equation}
and by removing $i$-th hidden node and using $(\nu_0 + \frac{\nu_i}{2})$ instead of $\nu_0$, we reach the same function as $f_i$ with $(w-1)$ hidden nodes. Finally, we have $\bW_i + \bW_j \neq \boldsymbol{0}$. Otherwise, we have
\begin{equation}
    \nu_i \cdot \sigma(\bW_i \varphi) + \nu_j \sigma(\bW_j \cdot \varphi) + \nu_0 = (\nu_i - \nu_j) \cdot \sigma(\bW_i \cdot \varphi) + (\nu_0+\nu_j) \ ,
\end{equation}
and by removing $j$-th hidden node and using $(\nu_i-\nu_j)$ instead of $\nu_i$ and $(\nu_0+\nu_j)$ instead of $\nu_0$, we reach the same function as $f_i$ with $(w-1)$ hidden nodes. In summary, we have
\begin{align}
    \bW_i \neq \boldsymbol{0} \ , \;\; \nu_i \neq 0 \ ,  \quad  \forall i \in [w] \ , \quad \mbox{and}  \quad    \bW_i \pm \bW_j \neq \boldsymbol{0} \ , \quad \forall i, j \in [w] : i \neq j \ . \label{eq:W_construction}
\end{align}
We will show that there does not exist $\bc \in \mcC_i$ such that $\bc^{\top} \cdot \nabla_{\bz} f_i(\varphi) = c_i$ for all $\varphi \in \R^{d_i}$. 
Assume the contrary, and assume there exist $\bc^* \in \mcC_i$ such that
\begin{align}
    (\bc^*)^{\top} \cdot \nabla_{\bz} f_i(\varphi) = c^*_i \ , \quad \forall \varphi \in \R^{d_i} \ .
    \label{eq:nn-contrary}
\end{align}
This is equivalent to showing that there exists nonzero $\bb^* \in \R^{d_i}$ such that 
\begin{align}
    (\bb^*)^{\top} \cdot \nabla_{\varphi} f_i(\phi) = c_i^{*} \ , \quad \forall \varphi \in \R^{d_i} \ .
\end{align}
Based on \eqref{eq:nn-expression}, the gradient of $f_i(\varphi)$ is
\begin{align}
    \nabla_{\varphi} f_i(\varphi) = \bW^{\top} \cdot \diag(\bnu) \cdot \sdot(W \cdot \varphi) \ , \label{eq:u-expression}
\end{align}
where $\diag(\bnu)$ is the diagonal matrix with $\bnu$ as its diagonal elements, and $\sdot$ is the derivative of the sigmoid function, applied element-wise to its argument. Hence, based on \eqref{eq:nn-contrary}, the contradiction premise is equivalent to having a nonzero $\bb^* \in \R^{d_i}$ such that
\begin{align}
    \big[\sdot(\bW \cdot \varphi)\big]^{\top} \cdot \diag(\bnu) \cdot \bW \cdot \bb^* = c_i \ , \quad \forall \varphi \in \R^{d_i} \ .  \label{eq:nn-constant-requirement}
\end{align}
We note that since $\bW$ is full-rank and $\nu_i$ is nonzero for all $i \in [w]$, the matrix $\diag(\bnu) \cdot \bW$ is full-rank as well and it has a trivial null space. Subsequently, $\bb^* \in \R^{d_i}$ is nonzero if and only if $(\diag(\bnu) \cdot \bW \cdot \bb^*) \in \R^{w}$ is nonzero. We will use the following lemma to show that \eqref{eq:nn-constant-requirement} cannot be true. This establishes that the contradiction premise is not correct, and completes the proof.
\begin{lemma}\label{lm:aux-nn}
    Let $\bu \in \R^p$ have nonzero entries with distinct absolute values, i.e., $u_i \neq 0$ and $|u_i| \neq |u_j|$ for all $i \neq j$, and $\alpha \in \R$ be a constant. Then, for every nonzero vector $\bc \in \R^p$, there exists $\alpha \in \R$ such that $[\sdot(\alpha \bu)]^{\top} \cdot \bc \neq a$.
\end{lemma}
\proof See \Cref{sec:proof:lm:aux-nn}.

Let us define $\xi \triangleq \bW \cdot \varphi$. We will show that there exists $\varphi^* \in \R^{d_i}$ such that $\xi = \bW \cdot \varphi^*$ satisfies the conditions in \Cref{lm:aux-nn}. Then, using \Cref{lm:aux-nn} with the choice of $\bu=\bW \cdot \varphi^*$, $d=\diag(\bnu) \cdot \bW \cdot \bb^*$, and $a=c_i$, we find that there exists $\alpha \in \R$ such that 
\begin{align}\label{eq:nn-constant-requirement-contrary}
     \big[\sdot(\alpha \cdot \bW  \cdot \varphi^*)\big]^{\top} \cdot \diag(\bnu) \cdot \bW \cdot \bb^* \neq c_i \ .
\end{align}
Hence, \eqref{eq:nn-constant-requirement} is false since it is violated for $\varphi=\alpha \varphi^*$ and the proof is completed. We show the existence of such $\varphi^*$ as follows. We first construct the set of $\varphi$ values for which conditions of \Cref{lm:aux-nn} on $w$ are \emph{not} satisfied. The set in question is the union of the following cases: (i) $\xi_i = \bW_i \cdot \varphi = 0$ for some $i \in [w]$, (ii) $|\xi_i| = |\xi_j|$ for some distinct $i,j \in [w]$, or equivalently, $(\bW_i \pm \bW_j) \varphi = 0$. Note that $\bW_i \neq \boldsymbol{0}$ and $\bW_i \pm \bW_j \neq \boldsymbol{0}$ by \eqref{eq:W_construction}. For a nonzero vector $\by \in \R^{d_i}$, the set $\{\varphi \in \R^{d_i} \ : \ \by^{\top} \varphi =  0\}$ is a $(d_i-1)$-dimensional subspace of $\R^{d_i}$. Then, there is $w$ number of $(d_i-1)$-dimensional subspaces that fall under case (i), and $w(w-1)$ number of $(d_i-1)$-dimensional subspaces that fall under case (ii). Therefore, there are $w^2$ lower-dimensional subspaces for which the conditions of \Cref{lm:aux-nn} do not hold. However, $\R^{d_i}$ cannot be covered by a finite number of lower-dimensional subspaces of itself. Therefore, there exists $\varphi^* \in \R^{d_i}$ such that $\xi = \bW \cdot \varphi^*$ satisfies the conditions of \Cref{lm:aux-nn}, and the proof is completed. 

\subsubsection{Proof of Lemma~\ref{lm:aux-nn}}\label{sec:proof:lm:aux-nn}
Assume the contrary and suppose that there exists a nonzero $\bc$ and $\alpha$ for a given $\bu$. Define the function $g_u(\alpha) \triangleq\bc^{\top} \cdot \sdot(\alpha \bu)$. Note that 
\begin{align}
    \sdot(x) = \frac{1}{1+e^{-x}+e^{x}}
\end{align}
is an even analytic function, and its Taylor series expansion at $0$ has the domain of convergence $\{x \in \R : |x| < \frac{\pi}{2}\}$. Thus, for all $\alpha \in (-\frac{\pi}{2 \max_{i} |u_{i}|}, \frac{\pi}{2 \max_{i} |u_{i}|})$,
\begin{align}
    a = g_u(\alpha) 
    &= \sum_{j=1}^{p} c_j \cdot \sdot(\alpha u_j) = \sum_{i=0}^{\infty} \gamma_i \sum_{j=1}^{p} c_j (\alpha u_j)^{2i} = \sum_{i=0}^{\infty} (\gamma_i \sum_{j=1}^{p} c_j u_j^{2i}) \alpha^{2i} \ .
\end{align}
Note that $g_u(\alpha)$ is a constant function of $\alpha \in (-\frac{\pi}{2 \max_{i} |u_{i}|}, \frac{\pi}{2 \max_{i} |u_{i}|})$. Thus, its Taylor coefficients, i.e., $(\gamma_i \sum_{j=1}^{p} c_j u_j^{2i})$, are zero for all $i \in \N^+$. However, the coefficients of even powers in Taylor expansion of $\sdot$, i.e., $\gamma_i$'s, are nonzero \citep{weisstein2002sigmoid}. Therefore, we have $\sum_{i=1}^{p}(c_j u_j^{2i}) = 0$ for all $i \in \N^+$. Next, construct the following system of linear equations,
\begin{align}
    \begin{bmatrix}
    u_1^2 & u_1^4 & \dots & u_1^{2w} \\
    \vdots & \vdots & \dots & \vdots \\
    u_p^2 & u_p^4 & \dots & u_p^{2w}
    \end{bmatrix} 
    \begin{bmatrix}
        c_1 \\ \vdots \\  c_p 
    \end{bmatrix}
    = \boldsymbol{0} \ .
\end{align}
This is equivalent to
\begin{align}
    \diag\big([u_1^2,\dots,u_p^2]^{\top}\big) \cdot  
    \underset{{\rm Vandermonde}}{\underbrace{
    \begin{bmatrix}
    1 & u_1^2 & \dots & u_1^{2(p-1)} \\
    \vdots & \vdots & \dots & \vdots \\
    1 & u_p^2 & \dots & u_p^{2(p-1)}
    \end{bmatrix}}}
    \cdot
    \begin{bmatrix}
        c_1 \\ \vdots \\  c_p 
    \end{bmatrix}
    =  \boldsymbol{0} \ . \label{eq:vandermonde}
\end{align}
Note that the Vandermonde matrix in \eqref{eq:vandermonde} has determinant $\prod_{1 \leq i \leq j < p}(u_i^2 - u_j^2)$, which is nonzero since $|u_i|\neq |u_j|$ for $i\neq j$. Multiplying an invertible matrix with a diagonal invertible matrix generates another invertible matrix, and $\bc$ must be a zero vector which is a contradiction. Hence, there does not exist such $\bc$ for which $\bc^{\top} \cdot \sdot(\alpha \bu)$ is constant for every $\alpha \in \R$.

\subsection{Assumptions of \texorpdfstring{\citep{zhang2023identifiability}}{Zhang et al. (2023)}}
\label{appendix:zhang2023-nonlinearity}
We clarify that Assumption 2 of \citet{zhang2023identifiability}, which is referred to as ``linear interventional faithfulness'', \emph{requires} nonlinear SCMs. Subsequently, both of their main results (Theorem 1 and Theorem 2) require this assumption. First, note that \citet[p.6]{zhang2023identifiability} implicitly explain that their results are for nonlinear SCMs: \textsl{``In general, we show in Appendix B that a large class of nonlinear SCMs and soft interventions satisfy this assumption.''} Next, we show that this assumption is violated for linear additive noise models as follows. First, we quote their assumption:
\paragraph{Linear interventional faithfulness}~\citep[Assumption 2]{zhang2023identifiability}: Intervention $I$ on node $i$ satisfies linear interventional faithfulness if for every $j \in \Ch(i) \cup \{i\}$ such that $\Pa(j) \cap \De(i) = \emptyset$, it holds that $\mathbb{P}(Z_j + \bc^{\top} \bZ_S) \neq \mathbb{P}^I(Z_j + \bc^{\top} \bZ_S)$ for all constant vectors $\bc \in \R^{|S|}$ where $S = [n] \setminus \De(i)$.
\paragraph{Linear SCMs violate linear interventional faithfulness.} Next, we show that linear SCMs violate the above linear interventional faithfulness assumption. Consider an intervention on node $i$, and let $j \in \Ch(i)$ such that $\Pa(i) \cap \De(i) = \emptyset$. Let 
\begin{equation}
    Z_j = \bw^\top \bZ_{\Pa(j)} + N_j \ ,
\end{equation}
where $\bw \in \R^{|\Pa(j)|}$ denotes the weight vector. Note that we have $\Pa(j) \subseteq S = [n] \setminus \De(i)$ since $\Pa(i) \cap \De(i) = \emptyset$. Consider vector $\bc \in \R^{|S|}$ such that $\bc_{\Pa(j)} = - \bW$ and $\bc_{S \setminus \Pa(j)} = \boldsymbol{0}$ where $\bc_{\Pa(j)}$ denotes the entries corresponding to nodes in $\Pa(j)$. Then, we have 
\begin{equation}
    Z_j + \bc^{\top}\bZ_S = N_j \ ,
\end{equation} 
which remains invariant after an intervention on node $i$ and yields $\mathbb{P}(Z_j + \bc^{\top} \bZ_S) = \mathbb{P}^I(Z_j + \bc^{\top} \bZ_S)$. Therefore, linear SCMs violate linear interventional faithfulness.
\section{Empirical Evaluations: Details and Additional Results}\label{appendix:experiments}

\subsection{Implementation Details of LSCALE-I Algorithm}\label{appendix:experiments-linear}

\paragraph{Preprocessing: Dimensionality reduction.}

Since $\bX = \bG \cdot \bZ$, we can compute $\image(\bG)$ using $n$ random samples of $\bX$ almost surely. Specifically, $\image(\bG)$ equals the column space of the sample covariance matrix of $n$ samples of $\bX$. In LSCALE-I, as a preprocessing step, we compute this subspace and express samples of $\bX$ and $\bs_\bX$ in this basis. This procedure effectively reduces the dimension of $\bX$ and $\bs_\bX$ from $d$ to $n$. Then, we perform steps of LSCALE-I using this $n$ dimensional observed data. 

\paragraph{Implementation details of \Cref{alg:linear}.} In Stage~L1 of \Cref{alg:linear}, choosing any $\bh \in \colspace(\bR_\bX^m)$ suffices. To make our algorithm deterministic, we choose the top eigenvector of $\colspace(\bR_\bX^m)$ that corresponds to the largest eigenvalue. In Stage~L2, we use a nonzero threshold $\lambda_{\mcG}$, i.e.,
\begin{equation}\label{eq:linear-hard-parent-set-empirical}
   \hat \Pa(m) \triangleq \Big\{ i \neq m :  \E\Big[\big| \bsd_{\hat \bZ}^{m}(\hat \bZ; \hat \bH)\big|_i \Big] \geq \lambda_{\mcG} \Big\} \ .
\end{equation}
In \Cref{tab:edge-param}, we list all $\lambda_{\mcG}$ values used in the experiments.

\begin{table}[t]
    \centering
    \caption{Choice of thresholding parameter $\lambda_{\mcG}$ for latent graph estimation.}
    \label{tab:edge-param}
    \begin{tabular}{ccc}
        \toprule
         Experiment & perfect scores & noisy scores \\
         \midrule
         \Cref{tab:lscalei-linear-hard} & 0.001 &  0.1 \\
         \Cref{tab:lscalei-mlp-hard-n5} & 0.001 & 0.05 \\
         \Cref{tab:lscalei-linear-soft} & 0.0001 & 0.001 \\
         \Cref{tab:lscalei-quadratic-soft-n5} & 0.001 & 0.1 \\
         \Cref{tab:gscalei-quadratic} & 0.01 & 0.5 \\
         \bottomrule
    \end{tabular}
\end{table}

\paragraph{Implementation details of \Cref{alg:linear-full-rank}.}

The latent graph estimation part of \Cref{alg:linear-full-rank} involves rank tests, i.e., determining the dimension of the intersection of column spaces. Therefore, we also need a threshold, denoted by $\lambda_{\rm eigv}$ to determine the numerical ranks in practice. In experiments reported in \Cref{tab:lscalei-quadratic-soft-n5}, we set $\lambda_{\rm eigv}=0.01$.

\paragraph{Using the algorithms of related work.} For the comparisons in \Cref{sec:experiments:comparisons}, we used the shared sources codes of \citet{squires2023linear}\footnote{ \url{https://github.com/csquires/linear-causal-disentanglement-via-interventions}} and \citet{zhang2023identifiability}\footnote{\url{https://github.com/uhlerlab/discrepancy_vae}}.

\subsection{Implementation Details of GSCALE-I Algorithm}\label{appendix:experiments-general}

We perform experiments for the coupled interventions setting, and solve the following optimization problem in Step~G2 of GSCALE-I,
\begin{equation}\label{eq:gscale-i-l1}
    \min_{h}  \big\|\bD_{\rm t}(h) - \bI_{n \times n}\big\|_{1,1} + \lambda \E\Big[ \big\| h^{-1}(h(\bX)) - \bX \big\|_2^2\Big] \ .
\end{equation}
where $\lambda > 0 $ is a regularization parameter to ensure injectivity. In the following, first, we show why this problem is equivalent to solving \eqref{eq:general-OPT1-procedure}. Then, we describe the computation of the ground truth score differences for $\bX$ and discuss other implementation details.

\paragraph{Implementation steps.}
We use $n_{\rm s}$ samples from the observational environment to compute empirical expectations. Since encoder $h$ is parameterized by $\bH$, we use gradient descent to learn this matrix. To do so, we minimize the loss described above in \eqref{eq:gscale-i-l1}. We denote the final parameter estimate by $\bH^*$ and the encoder parameterized by $\bH^*$ by $h^*$.

In the simulation results reported in \Cref{sec:experiments-general}, we set $\lambda = 1$ and solve \eqref{eq:gscale-i-l1} using RMSprop optimizer with learning rate $10^{-3}$ for $3 \times 10^{4}$ steps for $n=5$ and $4 \times 10^{4}$ steps for $n=8$. We also use early stopping when the training converges before the maximum number of steps. Recall that the latent graph estimate $\hat \mcG$ is constructed using $\mathds{1}\{\bD(h^*)\}$. Similarly to \eqref{eq:linear-hard-parent-set-empirical}, We use a threshold $\lambda_{\mcG}$ to obtain the graph from the upper triangular part of $\bD(h^*)$.

\paragraph{Architecture details and hyperparameter selection for image experiments.}

In \Cref{tab:autoencoder-architecture} and \Cref{tab:ldr-architecture}, we list the architecture details and training hyperparameters used in the image experiments in \Cref{sec:experiments-image}. Specifically, \Cref{tab:autoencoder-architecture} specifies the autoencoder architecture we used to estimate the latent variables, and \Cref{tab:ldr-architecture} details the classification-based score difference network we trained. Finally, \Cref{tab:autoencoder-architecture-cnn} lists the first level CNN-based autoencoder we used for image experiments with 2 balls.

\begin{table}[tbp]
    \centering
    \small
    \caption{2-step autoencoder architecture for learning latent representations of images.}
    \begin{tabular}{ll}
        \toprule
        \textbf{Step} & \textbf{Details} \\
        \midrule
        Step 1 Encoder & Input: Image $\bX \in \mathbb{R}^{64 \times 64 \times 3}$ \\
                       & Flatten \\ 
                       & FC(256), ReLU, LayerNorm \\
                       & FC(64) \Comment{Intermediate representation $\bY$}
        \\
        \midrule
        Step 1 Decoder & FC(256), ReLU, LayerNorm \\
                       & FC($64 \times 64 \times 3$), Sigmoid \Comment{Reconstructed image $\hat \bX_{\rm c}$}
        \\
        \midrule
        Autoencoder-1 &  Training: 100 epochs, Adam optimizer, batch size:16 \\
                       & Learning rate: $10^{-3}$, weight decay: 0.01 \\ 
                       & Minimize $\E\big[\| \bX-\hat \bX_{\rm c} \|^2\big]$ \Comment{recons. loss}
        \\
        \midrule
        Step 2 Encoder & Input: $\bY \in \R^{64}$ \\
                       & FC(256), ReLU, LayerNorm \\
                       & FC(6) \Comment{Latent causal variables $\hat \bZ$}
        \\
        \midrule
        Step 2 Decoder & FC(256), ReLU, LayerNorm \\
                       & FC($64$) \Comment{Reconstructed $\hat \bY$}
        \\
        \midrule
        Autoencoder-2 &  Training: 100 epochs, Adam optimizer, batch size: 16, weight decay: 0.01 \\
                       & Learning rate: 0.95 decay/epoch for 75 epochs, reset to $10^{-3}$ for 25 epochs 
                       \\ 
                       & Minimize $\norm{\bD_{\rm t}(h)-\bI_{n \times n}}_{1,1} + \lambda \E\big[\| \bY -\hat \bY \|^2\big]$ \Comment{score loss + recons. loss}
        \\
        \bottomrule
    \end{tabular}
    \label{tab:autoencoder-architecture}
\end{table}
\begin{table}[t]
    \centering
    \small
    \caption{Architecture for LDR (log-density-ratio) model for score-difference estimation.}
    \begin{tabular}{ll}
        \toprule
        Layers:
        & Input: Image $\bX \in \mathbb{R}^{64 \times 64 \times 3}$, class (intervention) label $y \in \{0,1\}$ \\
        & Conv(3, 32, kernel size=3, stride=1, padding=1), ReLU, BatchNorm, Dropout(0.1) \\
        & MaxPool(kernel size=2, stride=2) \\
        & Conv(32, 64, kernel size=3, stride=1, padding=1), ReLU, BatchNorm, Dropout(0.1) \\
        & MaxPool(kernel size=2, stride=2) \\
        & Flatten, FC(128), ReLU \\
        & FC(1) \\
        \midrule
        Training:
        & 10 epochs, Adam optimizer, weight decay: 0.01, batch size: 16 \\
        & Learning rate: $10^{-5}$ \\
        & Minimize binary cross entropy between sigmoid of negation of output and labels \\
        \bottomrule
    \end{tabular}
    \label{tab:ldr-architecture}
\end{table}
\begin{table}[tbp]
    \centering
    \small
    \caption{Autoencoder architecture for learning latent representations of images - Step 1, CNN variant.}
    \scalebox{0.8}{
    \begin{tabular}{ll}
        \toprule
        \textbf{Step} & \textbf{Details} \\
        \midrule
        Encoder:
        & Input: Image $\bX \in \mathbb{R}^{64 \times 64 \times 3}$ \\
        & Conv(3, 64, kernel size=3, stride=1, padding=1), BatchNorm, ReLU, MaxPool(kernel size=2, stride=2) \\
        & Conv(64, 128, kernel size=3, stride=1, padding=1), BatchNorm, ReLU, MaxPool(kernel size=2, stride=2) \\
        & Conv(128, 256, kernel size=3, stride=1, padding=1), BatchNorm, ReLU, MaxPool(kernel size=2, stride=2) \\
        & Conv(256, 512, kernel size=3, stride=1, padding=1), BatchNorm, ReLU, MaxPool(kernel size=2, stride=2) \\
        & Flatten, FC($n$) \\
        \midrule
        Decoder:
        & FC($n, 512 \cdot 4 \cdot 4$), Unflatten to (512, 4, 4) \\
        & Upsample(scale=2), ConvTranspose(512, 256, kernel size=3, stride=1, padding=1), BatchNorm, ReLU \\
        & Upsample(scale=2), ConvTranspose(256, 128, kernel size=3, stride=1, padding=1), BatchNorm, ReLU \\
        & Upsample(scale=2), ConvTranspose(128, 64, kernel size=3, stride=1, padding=1), BatchNorm, ReLU \\
        & Upsample(scale=2), ConvTranspose(64, 3, kernel size=3, stride=1, padding=1), BatchNorm, ReLU \\
        & Sigmoid \\
        \midrule
        Training:
        & 100 epochs, Adam optimizer, Cosine annealing with warm restart at 1/5th of epochs, batch size: 256 \\
        & Learning rate: $10^{-3}$, weight decay: 0.01 \\ 
        & Minimize $\E\big[\| \bX-\hat \bX_{\rm c} \|^2\big]$ \Comment{recons. loss} \\
        \bottomrule
    \end{tabular}}
    \label{tab:autoencoder-architecture-cnn}
\end{table}

\subsection{Experimental Methodology Comparison with DiscrepancyVAE}
\label{sec:discrvae-interv-enc-problem}

One major difference between the LSCALE-I algorithm and the publicly available DiscrepancyVAE implementation is that DiscrepancyVAE tries to learn the intervention effects as an encoding of the intervention labels. Specifically, for each intervention target, the VAE model learns a vector of latent intervention targets (a neural network with softmax activation) to denote which latent node is intervened on.  In practice, learning an intervention encoding runs the risk of learning encodings where different interventions affect the same latent node, consequently, having nodes that are not affected by any intervention encoding. We demonstrate this phenomenon for DiscrepancyVAE in \Cref{fig:discrepancy-vae-interv-enc-n5}, plotting the learned intervention embeddings for each single-node intervention. We observe that despite having exhaustive single-node interventions, the intervention embedding network fails to learn exhaustive intervention targets/effects. Therefore, the learned graph is not immediately interpretable as a DAG on $n$ causal latent variables since the learned intervention targets (i.e., node labels) are not surjective.

In contrast, in LSCALE-I, the effect of an intervention is estimated using score function differences. While the distinction may seem small, this approach ensures that the estimated graph's nodes are always interpretable since score differences between observational and interventional environments directly correspond to the intervention targets.
\begin{figure}
    \centering
    \includegraphics[width=0.3\linewidth]{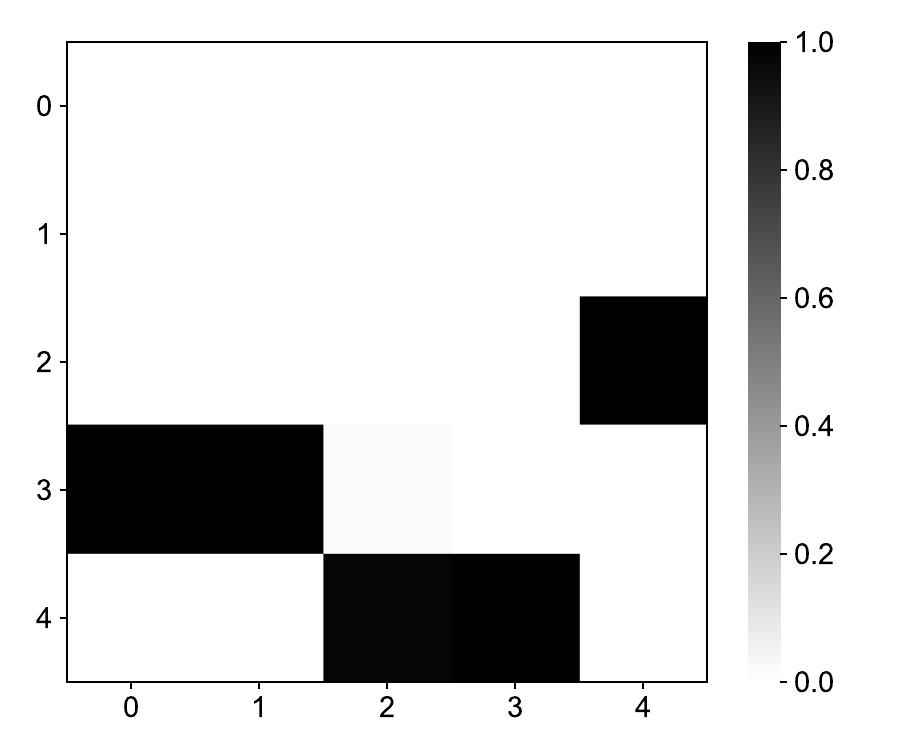}
    \caption{Intervention encoding learned in the sample run in \Cref{fig:discrepancy-vae-graph-n5}.}
    \label{fig:discrepancy-vae-interv-enc-n5}
\end{figure}

\subsection{Visualizations for Intervention Extrapolation on Biological Data} \label{sec:umap}

In Figure~\ref{fig:umaps}, we plot certain double-node intervention samples, both actual and sampled using our score-difference Langevin method in a UMAP~\citep{mcinnes2018umap} plot. We plot the double-node interventions that \citet{zhang2023identifiability} plots in their double-node intervention extrapolation results that have more than $1000$ samples for both single and double-node environments.{sec:discrvae-invert-enc-problem}
\begin{figure}
    \centering
    \includegraphics[width=0.18\linewidth]{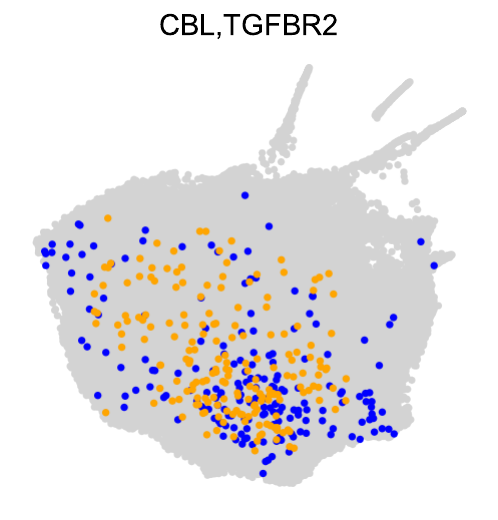} \ 
    \includegraphics[width=0.18\linewidth]{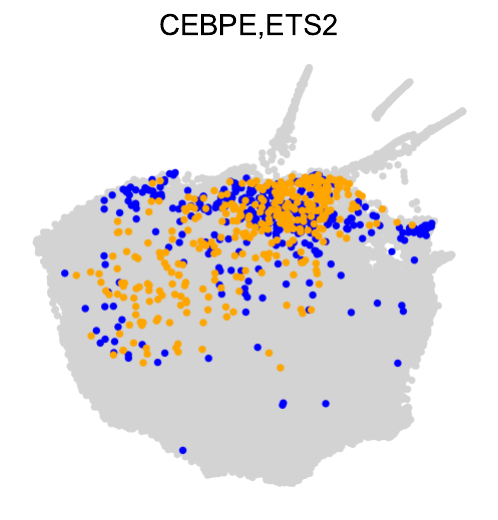} \ 
    \includegraphics[width=0.18\linewidth]{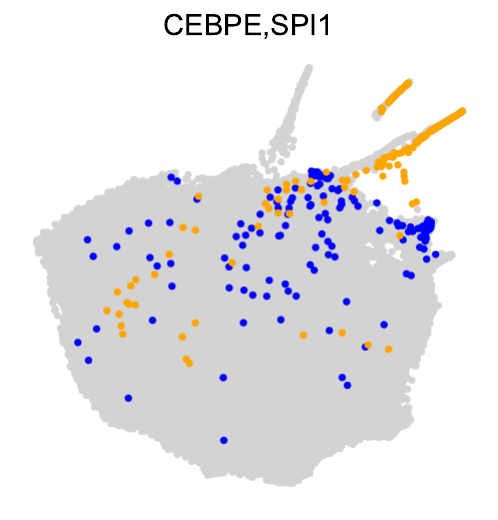} \ 
    \includegraphics[width=0.18\linewidth]{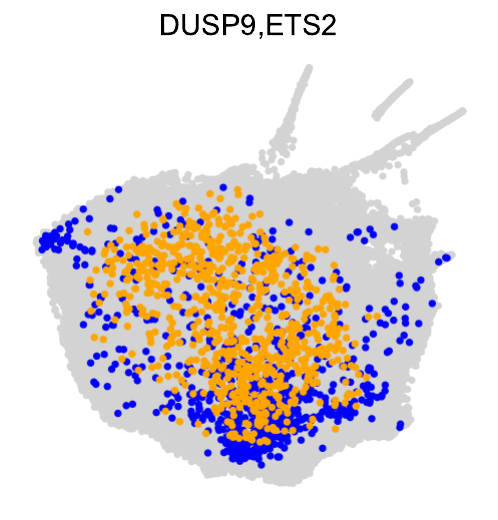} \\
    \includegraphics[width=0.18\linewidth]{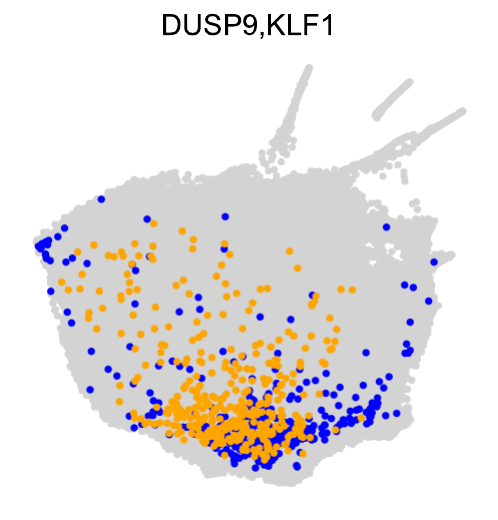} \ 
    \includegraphics[width=0.18\linewidth]{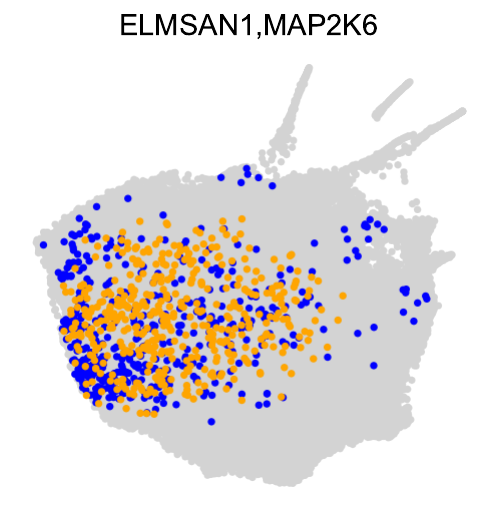} \ 
    \includegraphics[width=0.18\linewidth]{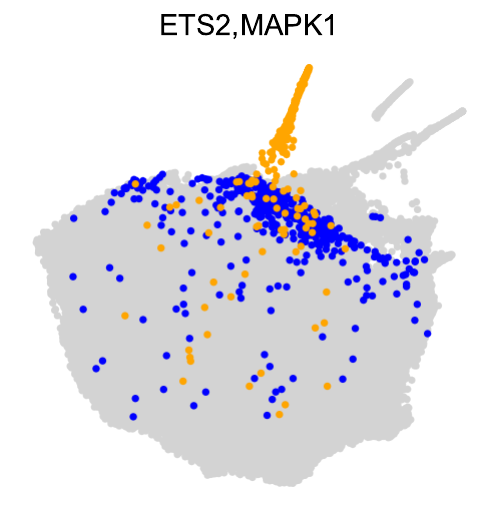} \ 
    \includegraphics[width=0.18\linewidth]{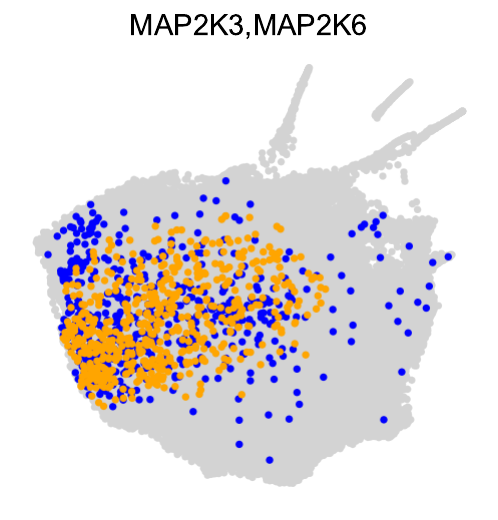} \
    \includegraphics[width=0.18\linewidth]{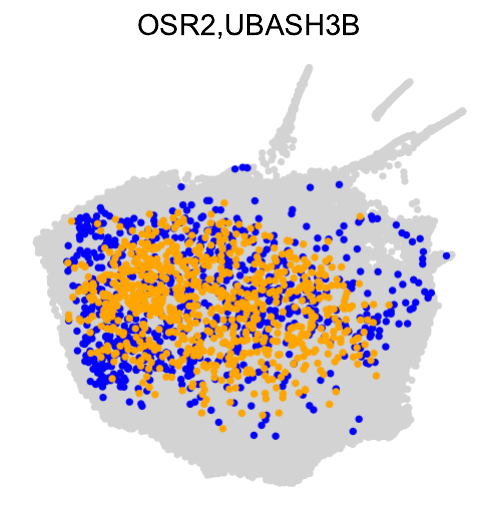}
    \caption{UMAP visualizations for double-node intervention extrapolation. Actual samples are in blue, generated samples are in orange.}
    \label{fig:umaps}
\end{figure}
We see across intervention pairs, score-based intervention extrapolation method leads to a sampling pattern that is relatively compliant with the actual double node interventional data. However, we observe that in some pairs, intervention effects are overestimated (particularly for CEBPE+SPI and ETS2+MAPK1).

\subsection{Details of Score Functions Used in the Experiments}\label{appendix:experiment-score-details}

\paragraph{Score function of the linear Gaussian model.}
For a Gaussian random vector $\bY \sim \mcN(\mu, \Sigma)$, the score function of $\bY$ is given by
\begin{equation}
    \bs_\bY(\by) = - \Sigma^{-1} \cdot (\by - \mu) \ .
\end{equation}
Using the linear model specified in \eqref{eq:linear-SEM--experiments}, in environment $\mcE^0$,
\begin{equation}
    \bZ = (\bI_{n \times n} - \bA)^{-1} \cdot N_i \ .
\end{equation}
Since $N_i$ is a zero-mean Gaussian random vector and $(\bI_{n \times n} - \bA)^{-1}$ is a full rank matrix, $\bZ$ is a zero-mean Gaussian random vector. Hence,
\begin{equation}
    \bs(\bz) = - \big[\Cov(\bZ)\big]^{-1} \cdot \bz \ .
\end{equation}
Note that the covariance of $\bZ$ has the following form
\begin{equation}
    \Cov(\bZ) = (\bI_{n \times n} - \bA)^{-1} \cdot \diag\big([\sigma_{1}^2,\dots,\sigma_{n}^2]^{\top}\big) \cdot \big((\bI_{n \times n} - \bA)^{-1}\big)^{\top} \ .
\end{equation}
Similarly, the score function of $\bZ$ in $\mcE^m$ is given by
\begin{equation}
    \bs^m(z) = - \big[\Cov(\bZ^m)\big]^{-1} \cdot \bz \ , \qquad \forall m \in [n] \ ,
\end{equation}
in which $\bZ^m$ denotes the latent variables $\bZ$ in environment $\mcE^m$. Finally, by setting $f=\bG$, \Cref{corollary:score-difference-transform-linear} specifies the score differences of $\bX$ in terms of $\bZ$ under different environment pairs as
\begin{equation}
    \bs_\bX(\bx) - \bs_\bX^m(\bx) = \big(\bG^{\dag}\big)^{\top} \cdot \big[\bs(\bz) - \bs^m(\bz)\big] \ .
\end{equation}

\paragraph{Score estimation for a linear Gaussian model.}
For all environments $\mcE^m \in \mcE$, given $n_{\rm s}$ i.i.d.\ samples of $\bX$, we first compute the sample covariance matrix denoted by $\hat \Sigma^m$. Then, we compute the sample precision matrix as
\begin{equation}
    \hat \Theta^m \triangleq (\hat \Sigma^m)^{\dag} \ ,
\end{equation}
which leads to the score function estimate given by
\begin{equation}
    \hat \bs_\bX^m(\bx) \triangleq - \hat \Theta^m \cdot \bx \ .
\end{equation}

\paragraph{Score function of the quadratic model.}
Following \eqref{eq:pz_m_factorized_hard}, score functions $s^m$ and $\tilde s^m$ are decomposed as follows.
\begin{align}
    \bs^{m}(\bz) &= \nabla_\bz \log q_{\ell}(z_{\ell}) +  \sum_{i \neq \ell} \nabla_\bz \log p_i(z_i \med \bz_{\Pa(i)}) \ , \label{eq:gscalei_sz_m_decompose_hard} \\
    \mbox{and} \quad \tilde \bs^{m}(\bz) &= \nabla_\bz \log \tilde q_{\ell}(z_{\ell}) +  \sum_{i \neq \ell} \nabla_\bz \log p_i(z_i \med \bz_{\Pa(i)})\ . \label{eq:gscalei_sz_m_decompose_hard_tilde}
\end{align}
For additive noise models, the terms in \eqref{eq:gscalei_sz_m_decompose_hard} and \eqref{eq:gscalei_sz_m_decompose_hard_tilde} have closed-form expressions. Specifically, using \eqref{eq:additive-component-score} and denoting the score functions of the noise terms $\{N_i : i \in [n]\}$ by $\{r_i : i \in [n]\}$, we have
\begin{equation}
    [\bs(\bz)]_i = r_i(n_i)-\sum_{j \in \Ch(i)} \frac{\partial f_i(\bz_{\Pa(j)})}{\partial z_i} \cdot r_j(n_j) \ .
\end{equation}
Recall that we consider a quadratic latent model with
\begin{equation}
    f_i(\bz_{\Pa(i)}) = \sqrt{\bz_{\Pa(i)}^{\top} \cdot \bQ_i \cdot \bz_{\Pa(i)}} \ ,
    \quad \mbox{and} \quad
    N_i \sim \mcN(0, \sigma_i^{2}) \ ,
\end{equation}
which implies 
\begin{equation}
  \frac{\partial f_j(\bz_{\Pa(j)})}{\partial z_i} = \frac{[\bQ_j]_i \cdot \bz_{\Pa(j)}}{\sqrt{\bz_{\Pa(j)}^{\top} \cdot \bQ_j \cdot \bz_{\Pa(j)}}} \ , \quad \mbox{and} \quad r_i(n_i) = -\frac{n_i}{\sigma_i^2} \ , \quad \forall i \in [n] \ .
\end{equation}
Components of the score functions $\bs^m$ and $\tilde \bs^m$ can be computed similarly. Subsequently, using \Cref{corollary:score-difference-transform-linear} of \Cref{lm:score-difference-transform-general}, we can compute the score differences of observed variables as follows.
\begin{align}
    \bs_\bX(\bx) - \bs_\bX^m(\bx) &= \Big[\big[J_g(\bz)\big]^{\dag}\Big]^{\top} \cdot \big[\bs(\bz) - \bs^m(\bz)\big] \ , \\
    \bs_\bX(\bx) - \tilde \bs_\bX^m(\bx) &= \Big[\big[J_g(\bz)\big]^{\dag}\Big]^{\top} \cdot \big[\bs(\bz) - \tilde \bs^m(\bz)\big] \ , \\   
    \bs_\bX^m(\bx) - \tilde \bs_\bX^m(\bx) &= \Big[\big[J_g(\bz)\big]^{\dag}\Big]^{\top} \cdot \big[\bs^m(\bz) - \tilde \bs^m(\bz)\big] \ .   
\end{align}

\bibliography{references}

\end{document}